\documentclass[11pt]{article}
\pdfoutput=1

\title{Some Insights into the Geometry and Training of Neural Networks}
\author{E. van den Berg\\[5pt]IBM Watson, Yorktown Heights, USA}

\usepackage[left=2cm,top=2cm,right=2cm,nohead]{geometry}
\usepackage{amsmath,amssymb,amsthm}
\usepackage{color}
\usepackage{graphicx}
\newtheorem{theorem}{Theorem}[section]
\newtheorem{lemma}[theorem]{Lemma}

\newcommand{\band}{{\sc{and}}}

\input{macros}

\begin{document}

\maketitle

\begin{abstract}
  Neural networks have been successfully used for classification tasks
  in a rapidly growing number of practical applications. Despite their
  popularity and widespread use, there are still many aspects of
  training and classification that are not well understood. In this
  paper we aim to provide some new insights into training and
  classification by analyzing neural networks from a feature-space
  perspective. We review and explain the formation of decision regions
  and study some of their combinatorial aspects. We place a particular
  emphasis on the connections between the neural network weight and
  bias terms and properties of decision boundaries and other regions
  that exhibit varying levels of classification confidence. We show
  how the error backpropagates in these regions and emphasize the
  important role they have in the formation of gradients. These
  findings expose the connections between scaling of the weight
  parameters and the density of the training samples. This sheds more
  light on the vanishing gradient problem, explains the need for
  regularization, and suggests an approach for subsampling training
  data to improve performance.
\end{abstract}

\section{Introduction}\label{Sec:Introduction}
Neural networks have been successfully used for classification tasks
in applications such as pattern recognition \cite{BIS1995a}, speech
recognition \cite{HIN2012DYDa}, and numerous others \cite{ZHA2000a}.
Despite their widespread use, the understanding of neural networks is
still incomplete, and they often remain treated as black boxes. In
this paper we provide new insights into training and classification by
analyzing neural networks from a feature-space perspective.  We
consider feedforward neural networks in which input vectors
$x_0 \in \mathbb{R}^d$ are propagated through $n$ successive layers,
each of the form
\begin{equation}\label{Eq:LayerDefn}
x_{k} = \nu_{k}(A_k x_{k-1} - b_k),
\end{equation}
where $\nu_k$ is a nonlinear activation function that acts on an
affine transformation of the output $x_{k-1}$ from the previous layer,
with weight matrix $A_k$ and bias vector $b_k$.  Neural networks are
often represented as graphs and the entries in vectors $x_k$ are
therefore often referred to as nodes or units. There are three main
design parameters in a feedforward neural network architecture: the
number of layers or depth the network, the number of nodes in each
layer, and the choice of activation function. Once these are fixed,
neural networks are training by adjusting only the weight and bias
terms.

Although most of the results and principles in this paper apply more
generally, we predominantly consider neural networks with sigmoidal
activation functions that are convex-concave and differentiable. To
keep the discussion concrete we focus on a symmetrized version of the
logistic function that acts  elementwise on its input as
\begin{equation}\label{Eq:Sigmoid}
\sigma_{\gamma}(x) = 2\ell_{\gamma}(x) - 1,
\qquad\mbox{with}\qquad
\ell_{\gamma}(x) = \frac{1}{1 + e^{-\gamma x}}.
\end{equation}
This function can be seen as a generalization of the hyperbolic
tangent, with $\sigma_{\gamma}(x) = \tanh(\gamma x/2)$.  We omit the
subscript $\gamma$ when $\gamma=1$, or when its exact value does not
matter.  For simplicity, and with some abuse of terminology we refer
to $\sigma_{\gamma}$ as the sigmoid function, irrespective of the
value of $\gamma$, and use the term logistic function for
$\ell_{\gamma}$.  Examples of several instances of $\sigma_{\gamma}$
and their first-order derivatives are plotted in
Figure~\ref{Fig:Sigmoid}.

\begin{figure}
\centering
\begin{tabular}{cc}
\includegraphics[width=0.475\textwidth]{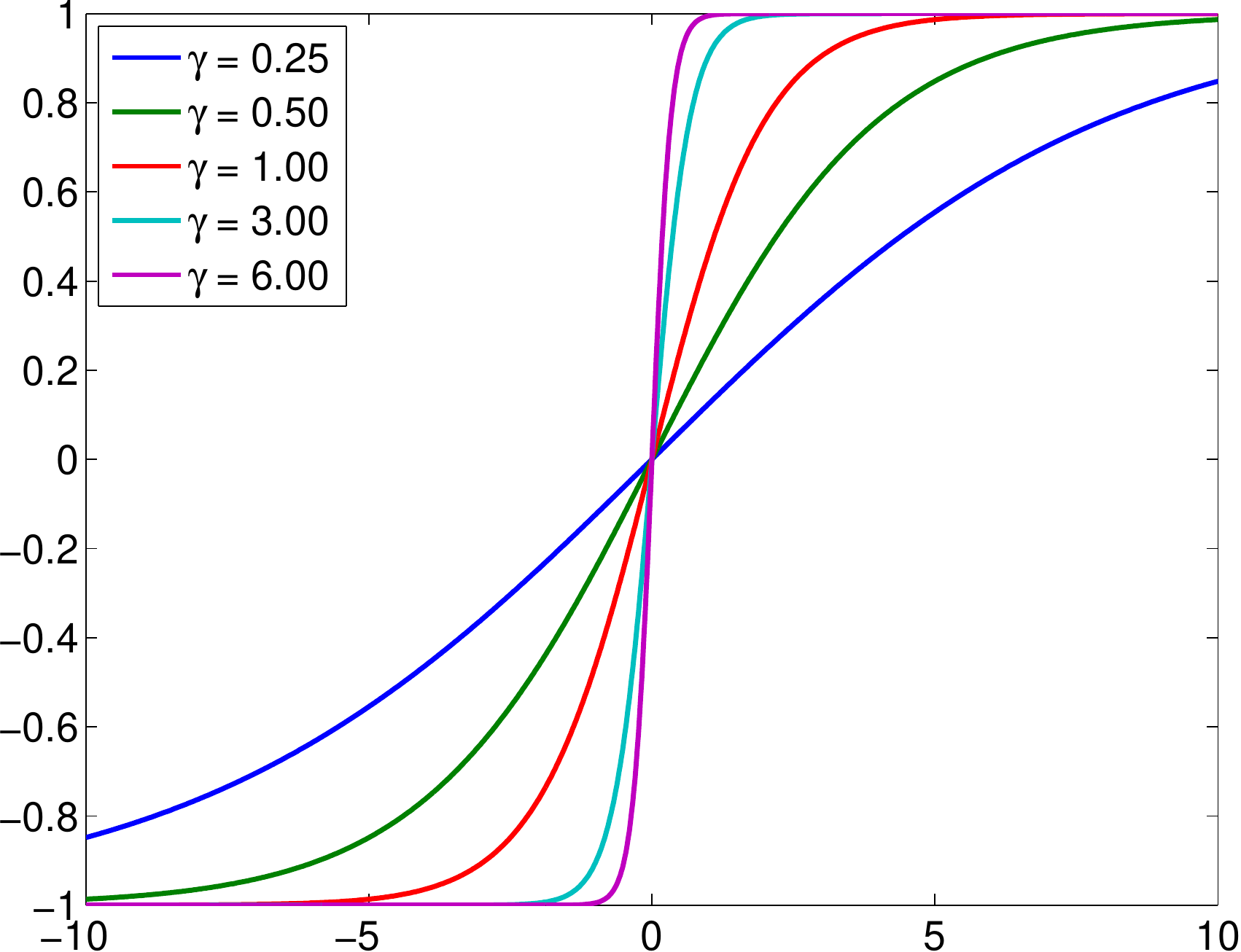}&
\includegraphics[width=0.475\textwidth]{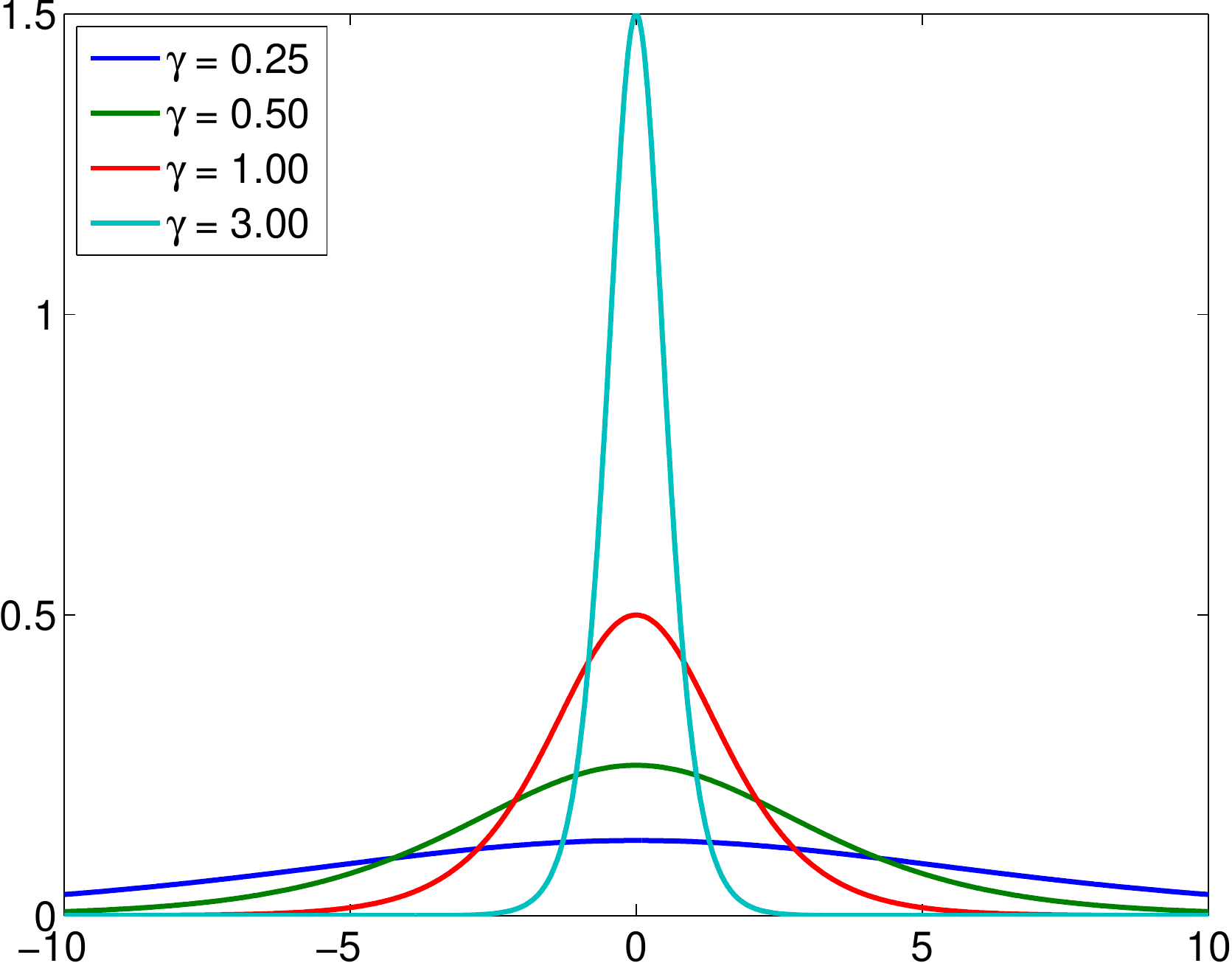}\\
({\bf{a}}) & ({\bf{b}})
\end{tabular}
\caption{Different instances of (a) the hyperbolic tangent function
  $\sigma_{\gamma}$ and (b) the derivatives.}\label{Fig:Sigmoid}
\end{figure}

The activation function in the last layer has the special purpose of
ensuring that the output of the neural network has a meaningful
interpretation. The softmax function is widely used and generates an
output vector whose entries are defined as
\begin{equation}\label{Eq:softmax}
[\mu(x)]_i= \frac{e^{x_{[i]}}}{\sum_{j=1}^{k}e^{x_{[j]}}},
\end{equation}
Exponentiation and normalization ensures that all output values are
nonnegative and sum up to one, and the output of node $i$ can therefore
be interpreted as an estimate of the posterior probability
$p(\textrm{class} = i \mid x)$. That is, we can define the estimated
probabilities as $\hat{p}_s(\textrm{class} = i \mid x) := [x_n(x)]_i$,
where $s$ is a vector containing of network weight and bias
parameters, and $x_n(x)$ is the output at layer $n$ corresponding to
input $x_0 = x$. The network parameters $s$ are typically learned from
domain-specific training data. In supervised training for multiclass
classification this training data comes in the form of a set of tuples
$\mathcal{T} = \{(x,c)\}$, each consisting of a sample feature vector
$x \in \mathbb{R}^d$ and its associated class label $c$. Training is
done by minimizing a suitably chosen loss function, such as
\begin{equation}\label{Eq:NNTraining}
\minimize{s}\quad \phi(s) := \frac{1}{\vert \mathcal{T}\vert}\sum_{(x,c) \in
  \mathcal{T}} f(s; x,c).
\end{equation}
with the cross-entropy function
\[
f(s; x,c) = -\log \hat{p}_s(c \mid x),
\]
which we shall use throughout the paper.  We denote the class $c$
corresponding to feature vector $x \in\mathbb{R}^d$ as $c(x)$, which,
in practice, is known only for all points in the training set.  For
notational convenience we also write $f(x)$ to mean $f(s; x, c(x))$.
The loss function $\phi(s)$ is highly nonconvex in $s$ making
\eqref{Eq:NNTraining} particularly challenging to solve. However, even
if it could be solved, care needs to be taken not to overfit the data
to ensure that the network generalizes to unseen data. This can be
achieved, for example, through regularization, early termination, or
by limiting the model capacity of the network.

The outline of the paper is as follows. In
Section~\ref{Sec:DecisionRegions} we review the definition of
halfspaces and the formation of decision regions.  In
Section~\ref{Sec:Regions} we look at combinatorial properties of the
decision regions, their ability to separate or approximate different
classes, and possible generalizations.  Section~\ref{Sec:Gradients}
analyzes the connection between the decision regions and the gradient
with respect to the different network parameters. Topics related to
the training of neural networks including backpropagation,
regularization, the contribution of individual training samples to the
gradient, and importance sampling are discussed in
Section~\ref{Sec:Training}. We conclude the paper with a discussion
and future work in Section~\ref{Sec:Discussion}.

Throughout the paper we use the following notational
conventions. Matrices are indicated by capitals, such as $A$ for the
weight matrices; vectors are denoted by lower-case roman letters. Sets
are denoted by calligraphic capitals. Subscripted square brackets
denote indexing, with $[x]_j$, $[A]_i$, and $[A]_{i,j}$ denoting
respectively the $j$-th entry of vector $x$, the $i$-th row of $A$ as
a column vector, and the $(i,j)$-th entry of matrix $A$. Square
brackets are also used to denote vector or matrix instances with
commas separating entries within one row, and semicolon separating
rows in in-line notation.  When not exponentiated $e$ denotes the
vector of all ones. The largest singular value of $A$ is denoted
$\sigma_{\max}(A)$, from the context it will be clear that this is not
a particular instance of the sigmoidal function $\sigma_{\gamma}$. The
vector $\ell_1$, and $\ell_2$ norms refer to the one- and two norms;
that is, the sum of absolute values and the Euclidean norm,
respectively. There should be no confusion between this and the
$\ell_{\gamma}$ logistic function.

\section{Formation of decision regions}\label{Sec:DecisionRegions}
Decision regions can be described as those regions or sets of points
in the feature space that are classified as a certain class.
Classification in neural networks is soft in the sense that it comes
as a vector of posterior probabilities and, in case that is desirable,
it is therefore not immediately obvious how to assign points to one
class or another. Two possible definitions of decision regions for
class $c$ are the set of points where the posterior probability is
highest among the classes:
\[
\mathcal{C}_c := \{ x \in \mathbb{R}^d \mid \hat{p}_s(c \mid x) = \max_j
\hat{p}_s(j \mid x)\},
\]
or exceeds a given threshold:
\begin{equation}\label{Eq:DefnDecisionRegion}
\mathcal{C}_c := \{x \in\mathbb{R}^d \mid \hat{p}_s(c \mid x) \geq \tau\}.
\end{equation}
Although this section discusses the formation and role of decision
regions and its boundaries, we will not use any formal definition of
decision regions. However, the intuitive notion used closely follows
definition~\eqref{Eq:DefnDecisionRegion}.

As we will see in this section, decision regions are formed as input
is propagated through the network. Even though the form
\eqref{Eq:LayerDefn} of all the layers is identical, we can
nevertheless identify two distinct stages in region formation. The
first stage defines a collection of halfspaces and takes place in the
first layer of the network. The second stage takes place over the
remaining layers in which intermediate regions are successively
combined to form the final decision regions, starting with the initial
set of halfspaces. The generation of halfspaces or hyperplanes in the
first layer of the neural network and their combination in subsequent
layers is well known (see for example~\cite{BIS1995a,MAK1989SEa}). The
formation of soft decision boundaries and some of their properties
does not appear to have been studied widely.  Some of the notions
discussed next form the basis for subsequent sections, and we
therefore review the two separate stages mentioned above in some
detail, with a particular emphasis on the role of the sigmoidal
activation function.

\subsection{Definition of halfspaces}\label{Sec:Halfspaces}

The output of the first layer in the network can be written as $y =
\sigma(Ax - b)$ with $x \in \mathbb{R}^d$. For an individual unit $j$
this reduces to $\sigma(\langle a,x\rangle - \beta)$, where $a \in
\mathbb{R}^d$ corresponds to $[A]_j$, the $j$-th row of $A$, and
$\beta = [b]_j$. When applied over all points of the feature space,
the affine mapping $\langle a,x\rangle -\beta$ generates a linear
gradient, as shown in Figure~\ref{Fig:Example1}(a).  The output of
a unit is then obtained by applying the sigmoid function to
these intermediate values. When doing so, assuming throughout that
$a\neq 0$, two prominent regions form: one with values close to $-1$
and one with values close to $+1$. In between the two regions there is
a smooth transition region, as illustrated in
Figure~\ref{Fig:Example1}(b). The center of the transition region
consists of all feature points whose output value equal zero. It can
be verified that this set is given by all points $x = \beta
a/\norm{a}_2^2 + v$ such that $\langle a,v\rangle = 0$, and therefore
describes a hyperplane. The normal direction of the hyperplane is
given by $a$, and the exact location of the hyperplane is determined
by a shift along this normal, controlled by both $\beta$ and
$\norm{a}_2$.  The region of all points that map to nonnegative values
forms a halfspace, and because the linear functions can be chosen
independently for each unit, we can define as many halfspaces are
there are units in the first layer. As the transition between the
regions on either side of the hyperplane is gradual it is convenient
to work with soft boundaries and interpret the output values as a
confidence level of set membership with values close to $+1$
indicating strong membership, those close to $-1$ indicating strong
non-membership, and with decreasing confidence levels in between. For
simplicity we use the term halfspace for both the soft and sharp
versions of the region.

\begin{figure}
\centering
\begin{tabular}{ccc}
\includegraphics[width=0.3\textwidth]{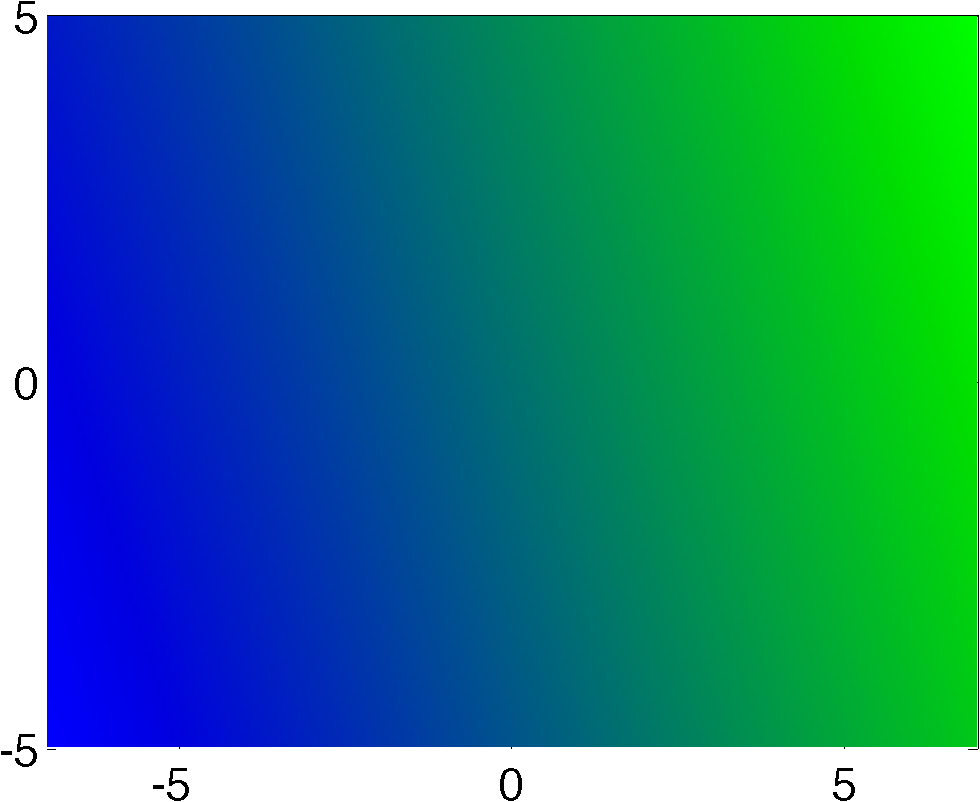}&
\includegraphics[width=0.3\textwidth]{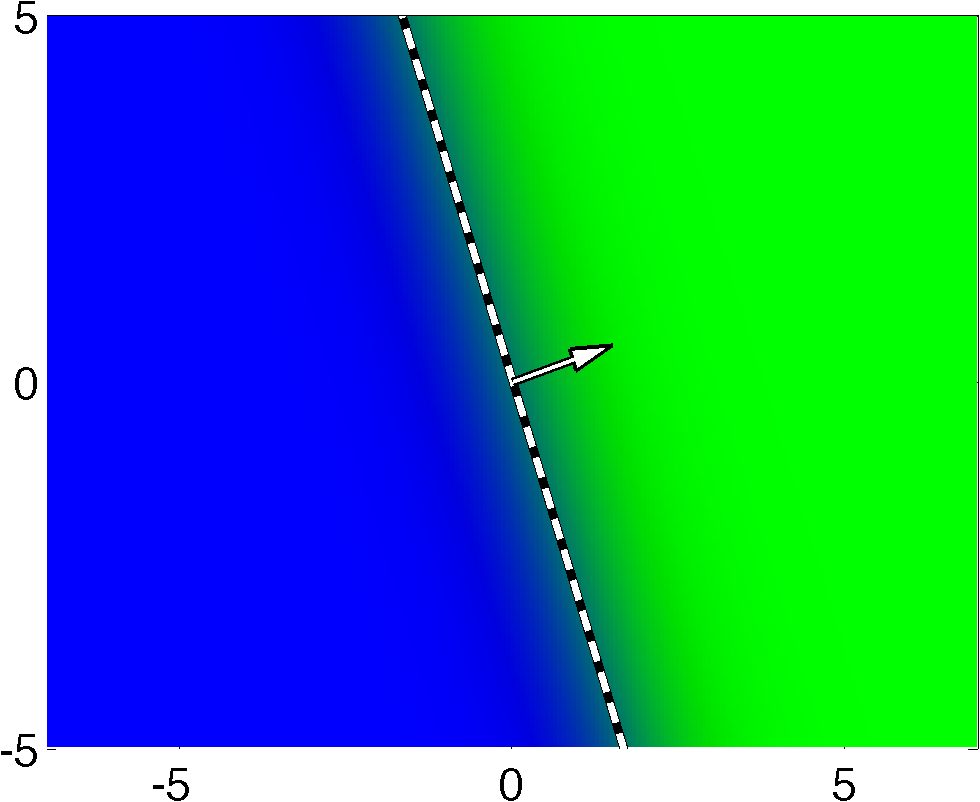}%
\begin{picture}(0,0)(0,0)%
\put(-61,75){\color{black}$a$}
\put(-34,60){\color{black}$\mathcal{R}$}
\put(-115,60){\color{white}$\mathcal{R}^c$}
\end{picture}&
\includegraphics[width=0.3\textwidth]{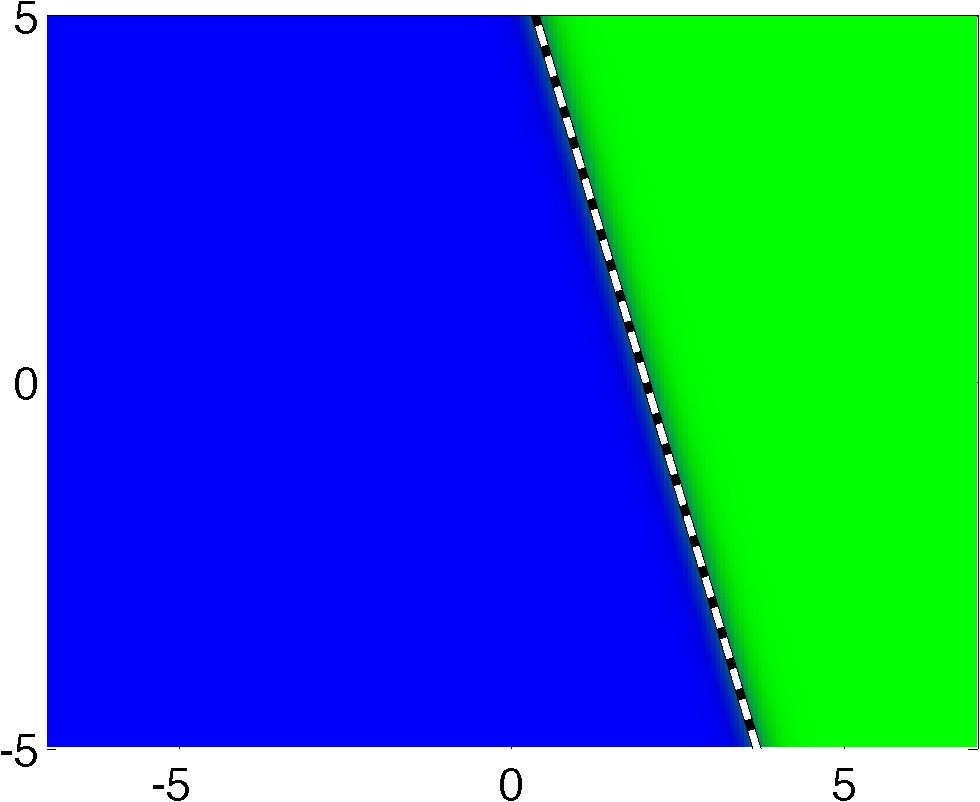}\\
\includegraphics[width=0.3\textwidth]{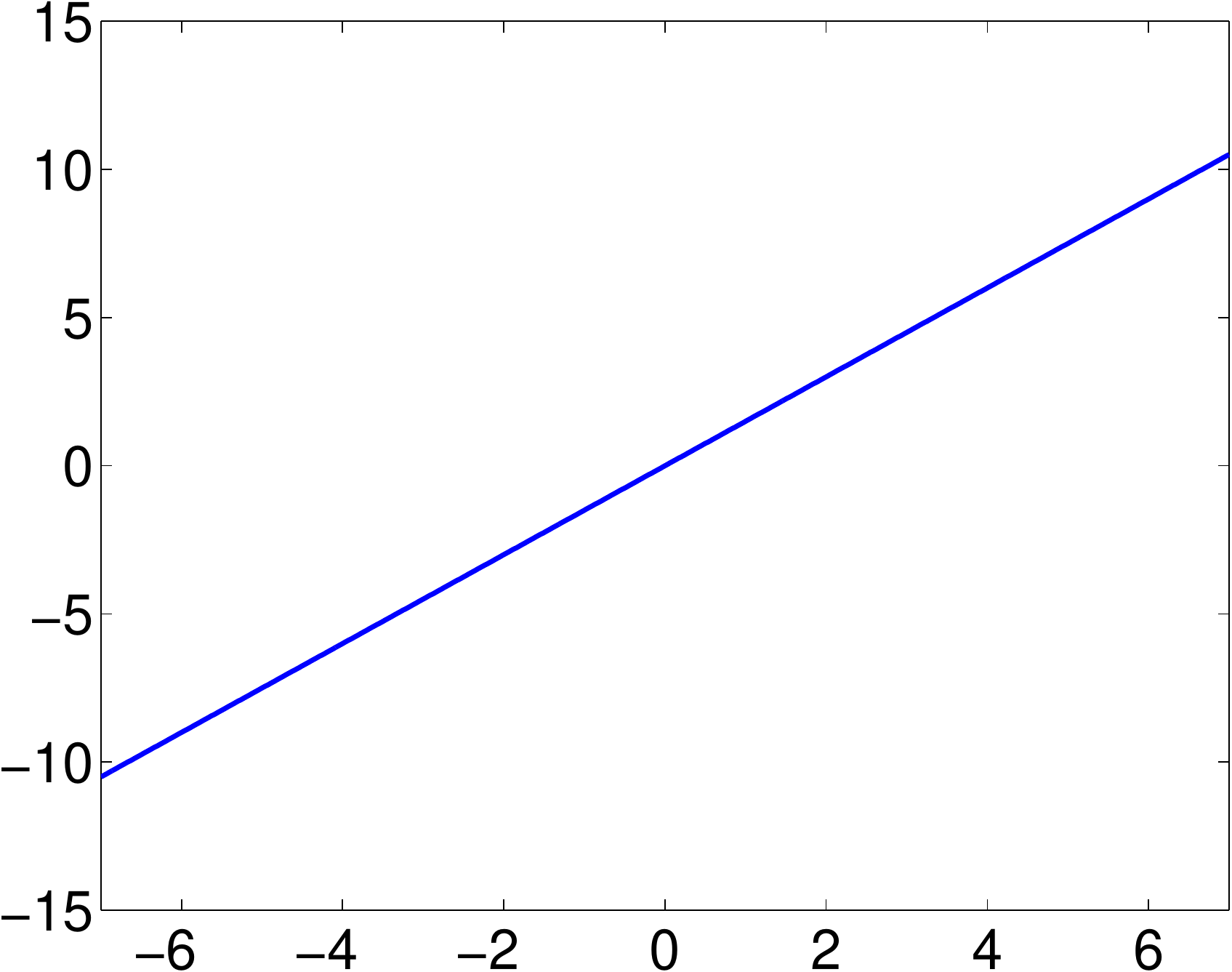}&
\includegraphics[width=0.3\textwidth]{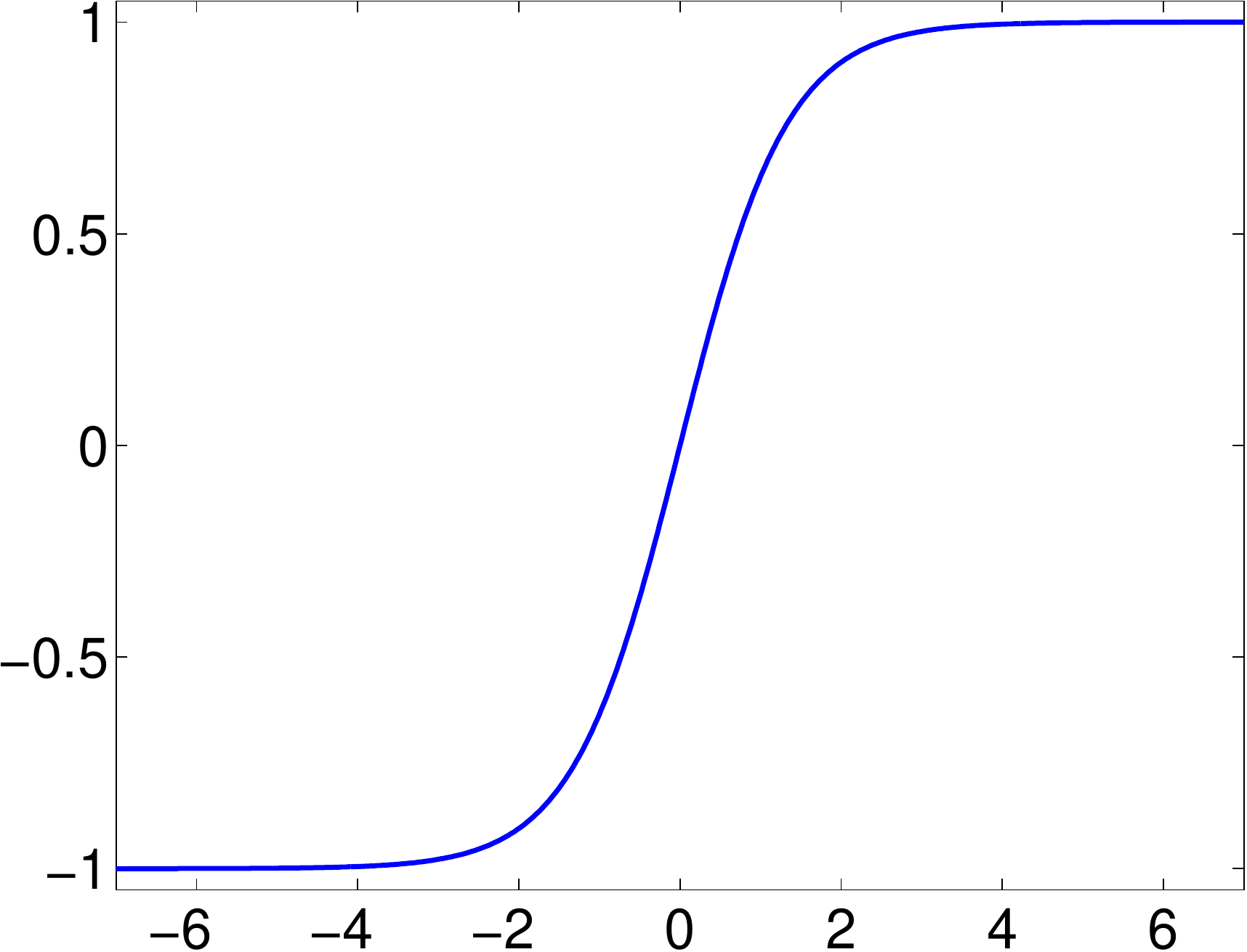}&
\includegraphics[width=0.3\textwidth]{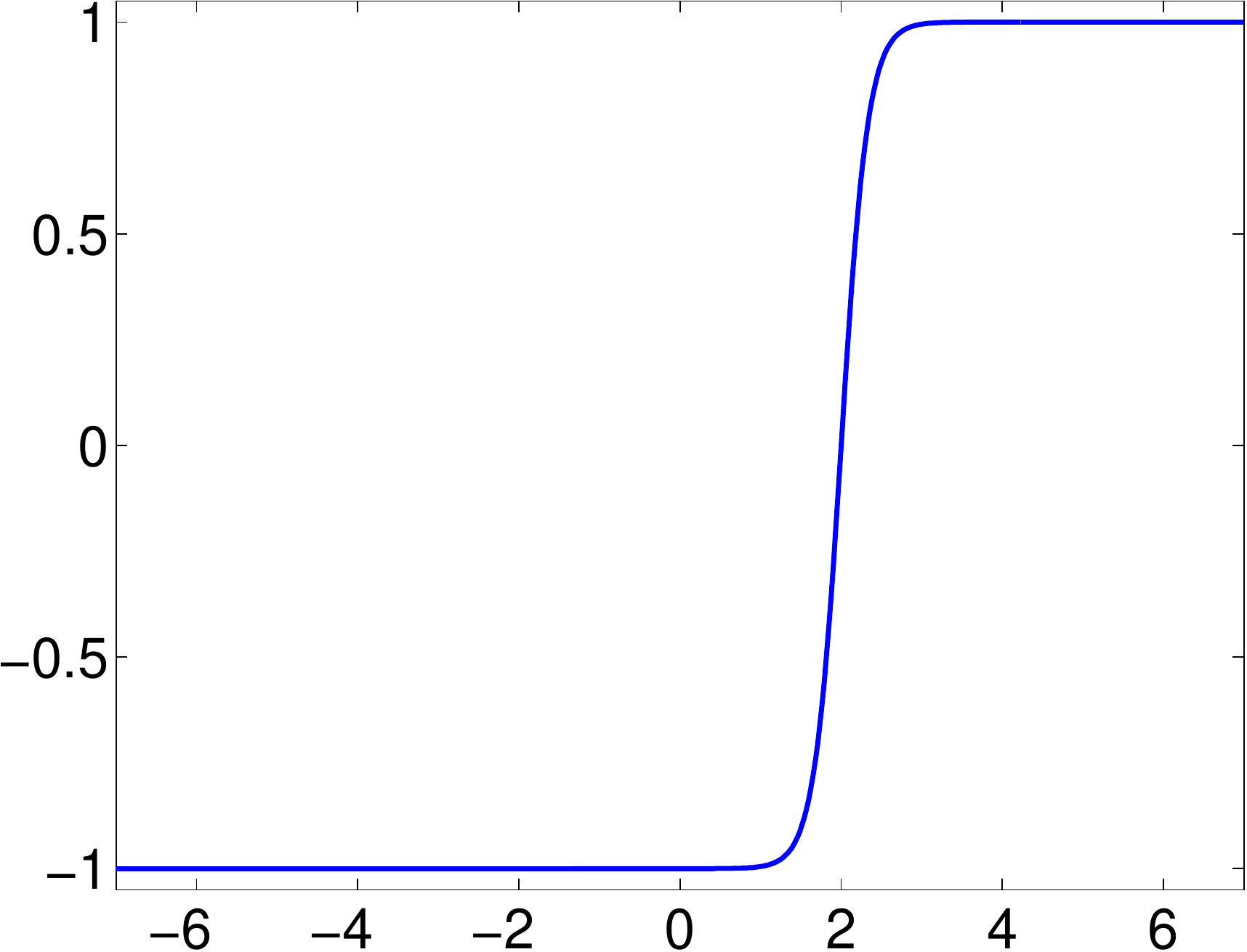}\\
({\bf{a}}) & ({\bf{b}}) & ({\bf{c}})
\end{tabular}
\caption{The mapping of points $x$ in the feature space using (a) the
  linear transformation $\langle a,x\rangle - \beta$ with $a = [1.5,
  0.5]^T$ and $\beta = 0$; (b) the nonlinearity $\sigma(\langle
  a,x\rangle -\beta)$ with the same values for $a$ and $\beta$; and
  (c) the nonlinearity $\sigma(\langle a',x\rangle -\beta)$ with $a' =
  4a$ and $\beta = 12$. Show are the output values for points in a
  rectangular region of the feature space (top row), and for points
  $x$ with $[x]_2 = 0$ (bottom row). }\label{Fig:Example1}
\end{figure}
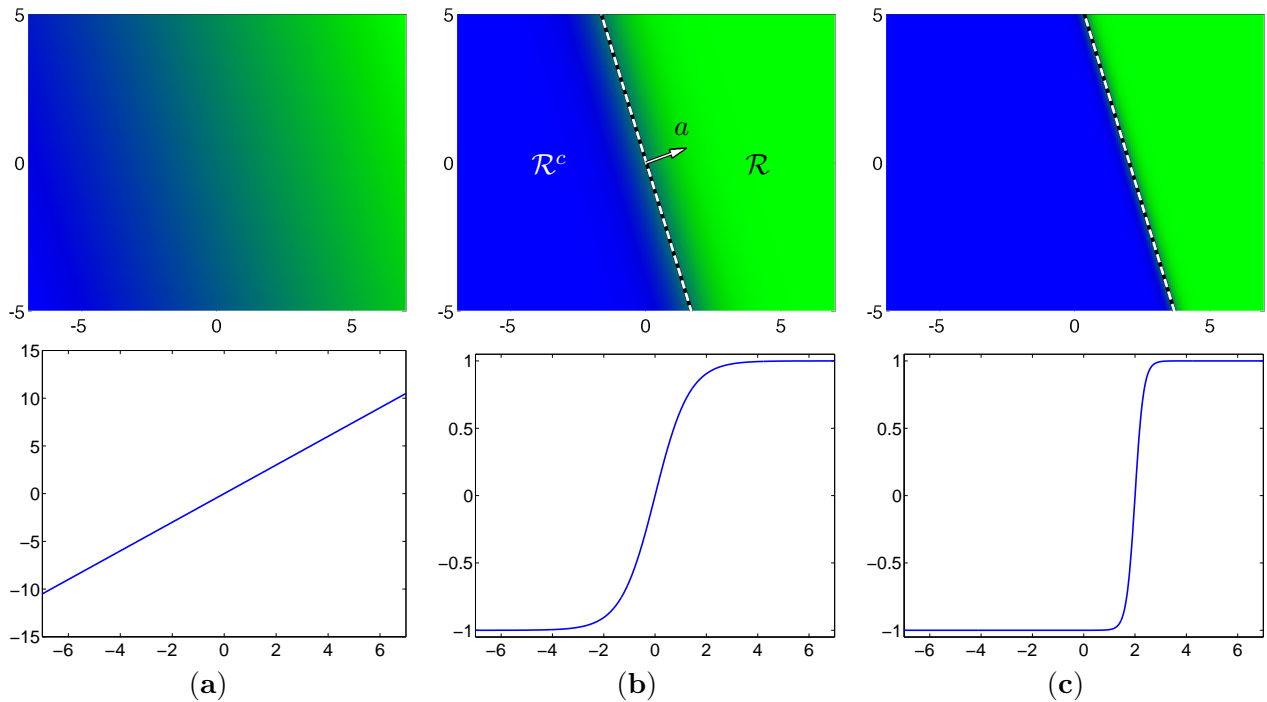

In addition to normal direction and location, halfspaces are
characterized by the sharpness of the transition region. This property
can be controlled in two similar ways (see also
Section~\ref{Sec:ConstraintsRegularization}). The first is to scale
both $a$ and $\beta$ by a positive scalar $\gamma$. Doing so does not
affect the location or orientation of the hyperplane but does scale
the input to the sigmoid by the same quantity. As a
consequence, choosing $\gamma > 1$ shrinks the transition region,
whereas choosing $\gamma < 1$ causes it to widen. The second way is to
replace the activation function $\sigma$ by $\sigma_{\gamma}$.
Scaling only $a$ affects the sharpness of the transition in the same
way, but also results in a shift of the hyperplane along the normal
direction whenever $\beta \neq 0$.  Note however that the activation
functions are typically fixed and the properties of the halfspaces are
therefore controlled only by the weight and bias
terms. Figure~\ref{Fig:Example1}(c) illustrates the sharpening of the
halfspace and the use of $\beta$ to change its location.

\subsection{Combination of intermediate regions}\label{Sec:SetOperations}

The second layer combines the halfspace regions defined in the first
layer resulting in new regions in each of the output nodes. In case
step-function activation functions are used, the operations used to
combine the regions correspond to set operations including complements
($^c$), intersection ($\cap$), and unions ($\cup$). The same
operations are used in subsequent layers, thereby enabling the
formation of increasingly complex regions. The use of a sigmoidal
function instead of the step function does not significantly change
the types of operations, although some care needs to be taken.

Some operations are best explained when working with input coming from
the logistic function (with values ranging from 0 to 1) rather than
from the sigmoid function (ranging from -1 to 1).  Note however that
output $x$ from a sigmoid function can easily be mapped to the output
$x' = (x-1)/2$ from a logistic function, and vice versa, by
appropriately scaling the weight and bias terms in the next layer.
Any linear operation $Ax' - b$ on the logistic output then becomes
$A(x - 1)/2 - b = \tilde{A}x - \tilde{b}$ with $\tilde{A} = A/2$ and
$\tilde{b} = b + Ae/2$. In other words, with appropriate changes in
$A$ and $b$ we can always choose which of the two activation functions
the input comes from, regardless of which function was actually used.

\subsubsection{Elementary Boolean operations}\label{Sec:ElementaryBoolean}

To make the operations discussed in this section more concrete we
apply them to input generated by a first layer with the following
parameters:
\[
A_1 = \left[\begin{array}{rr}9 & 1 \\ -2 & 6\end{array}\right],\quad
b_1 = \left[\begin{array}{r}-2\\-1\end{array}\right],\quad\mbox{and}\quad
\nu_1 = \sigma_3.
\]
The two resulting halfspace regions $\mathcal{R}_1$ and
$\mathcal{R}_2$ are illustrated in Figures~\ref{Fig:Example3}(a) and
\ref{Fig:Example3}(b). For simplicity we denote the parameters for the
second layer by $A$ and $b$, omitting the subscripts. In addition, we
omit all entries that are not relevant to the operation and apply
appropriate padding with zeros where needed is implied.

\begin{figure}[t]
\centering
\begin{tabular}{ccc}
\includegraphics[width=0.31\textwidth]{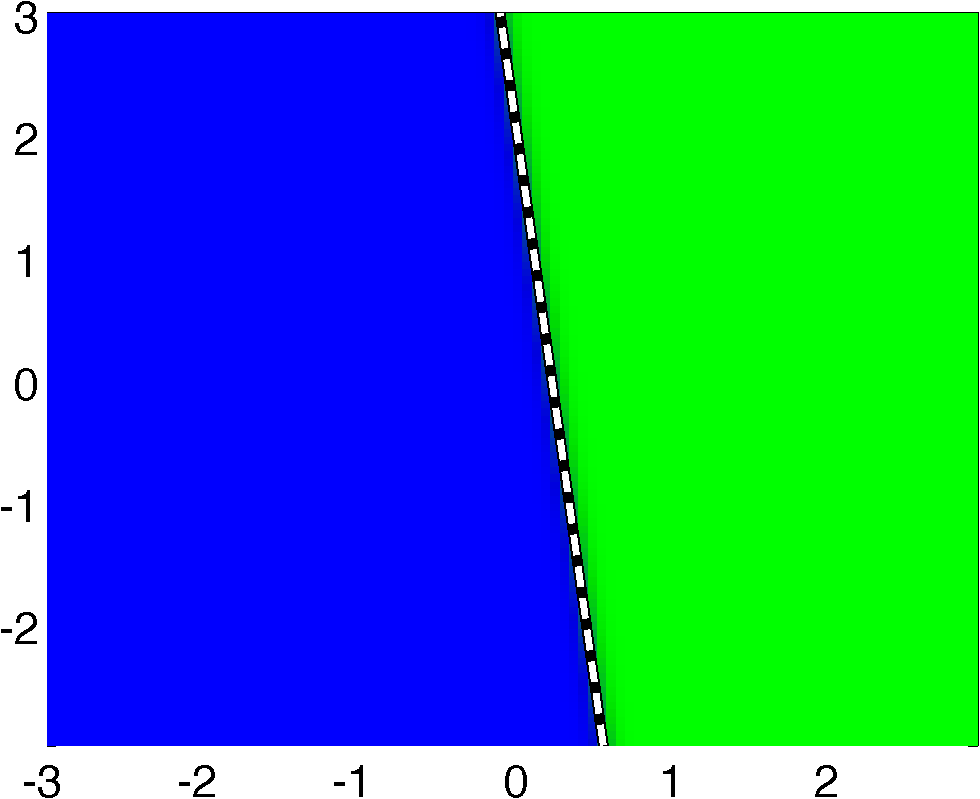}%
\begin{picture}(0,0)(0,0)%
\put(-39,65){\color{black}$\mathcal{R}_1$}
\end{picture}&
\includegraphics[width=0.31\textwidth]{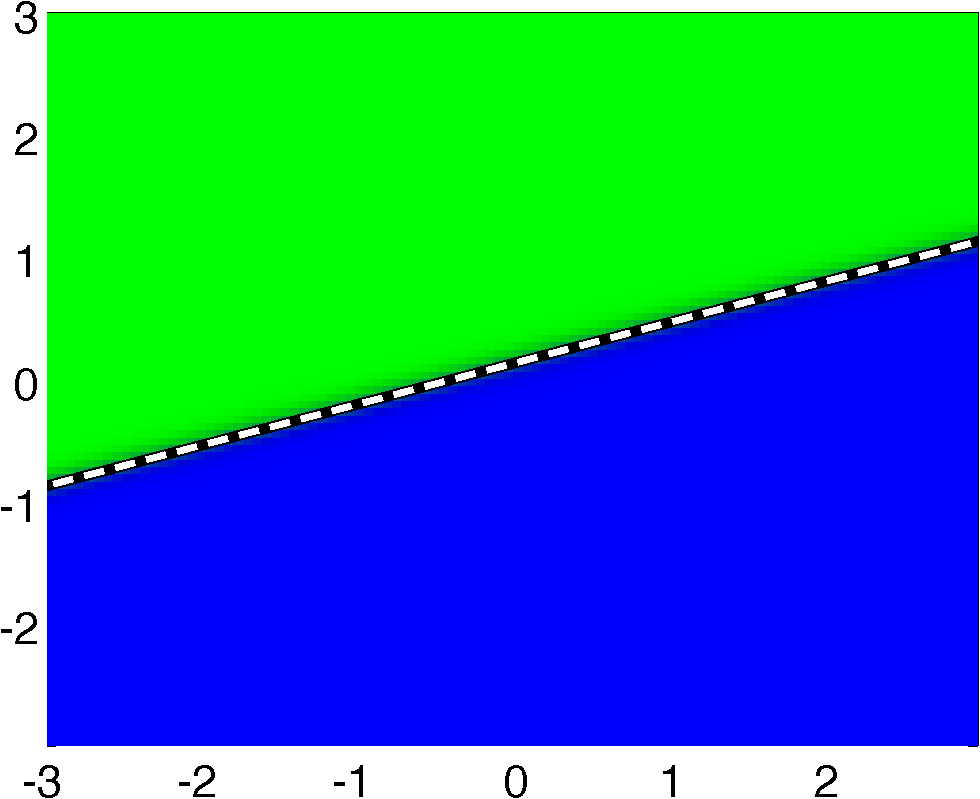}%
\begin{picture}(0,0)(0,0)%
\put(-84,96){\color{black}$\mathcal{R}_2$}
\end{picture}&
\includegraphics[width=0.31\textwidth]{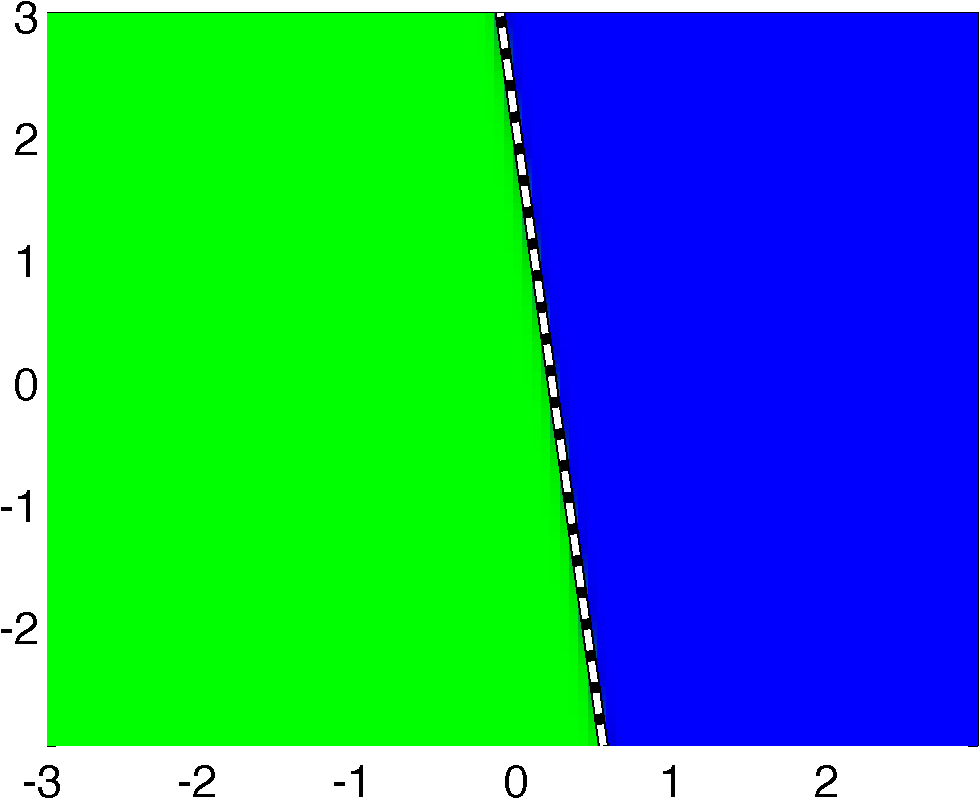}\\
({\bf{a}}) Region $\mathcal{R}_1$ &
({\bf{b}}) Region $\mathcal{R}_2$&
({\bf{c}}) Complement: $({\mathcal{R}}_1)^c$ \\[6pt]
\includegraphics[width=0.31\textwidth]{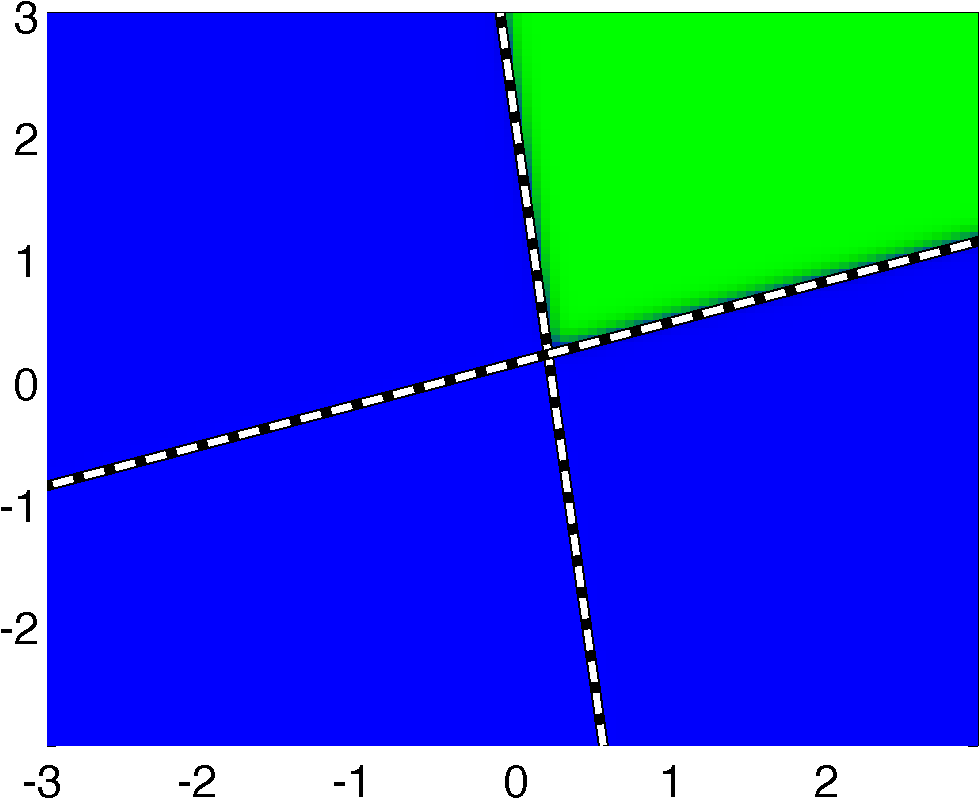}&
\includegraphics[width=0.31\textwidth]{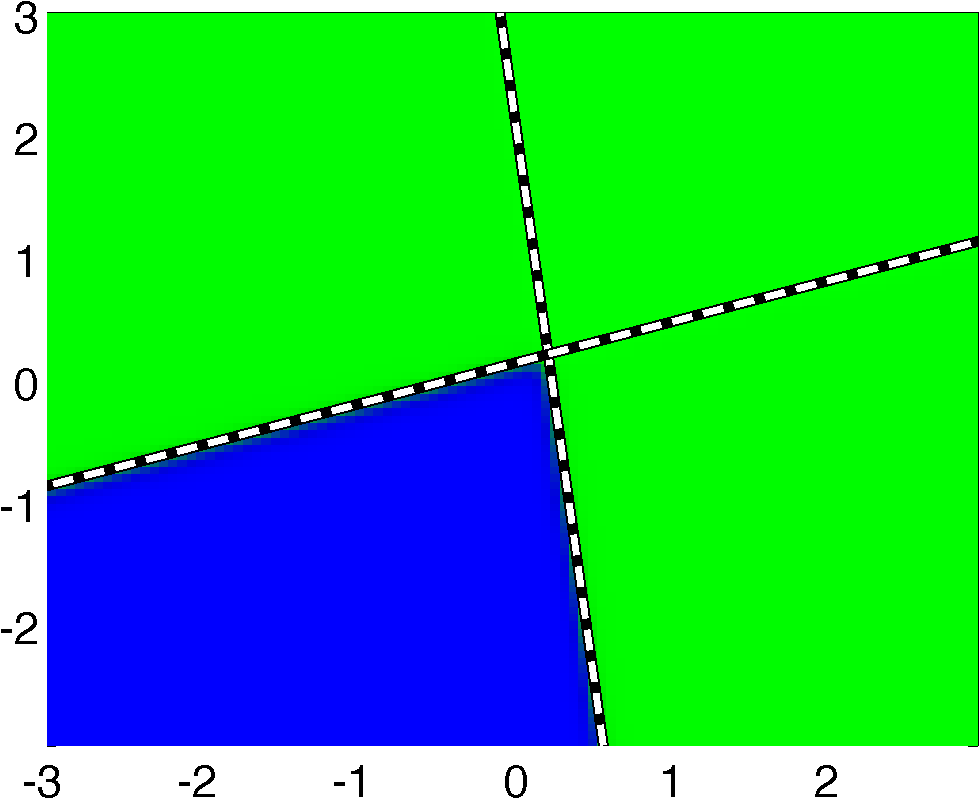}&
\includegraphics[width=0.31\textwidth]{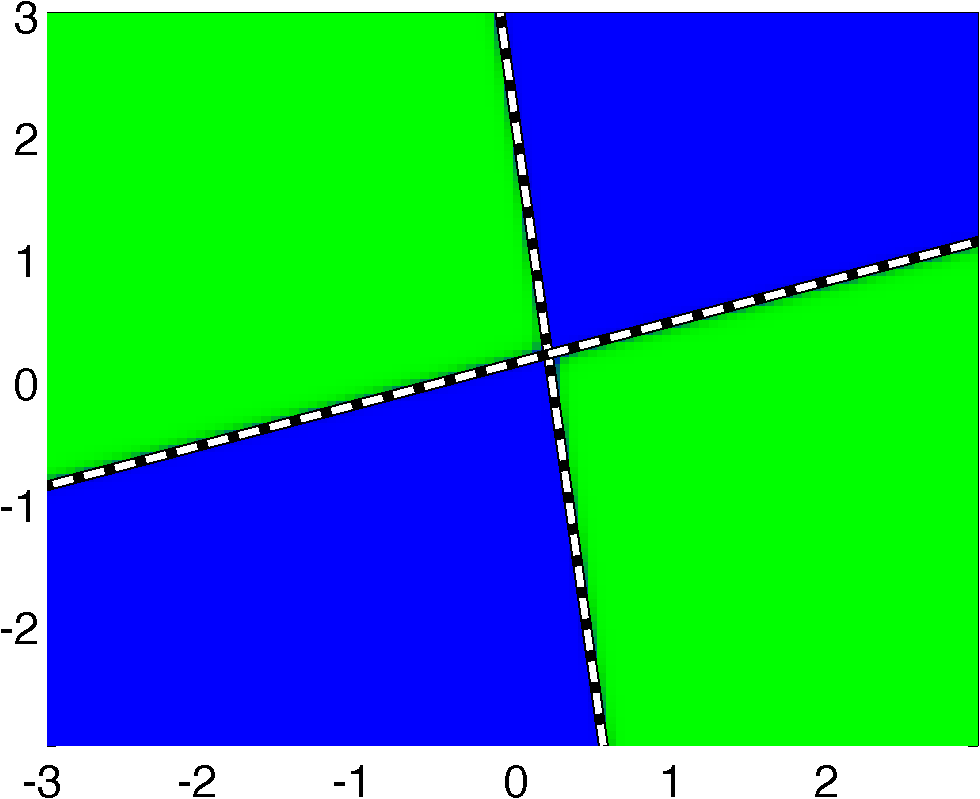}\\
({\bf{d}}) $\mathcal{R}_1 \cap \mathcal{R}_2$ &
({\bf{e}}) $\mathcal{R}_1 \cup \mathcal{R}_2$ &
({\bf{f}}) $((\mathcal{R}_1)^c \cap \mathcal{R}_2) \cup (\mathcal{R}_1
\cap (\mathcal{R}_2)^c)$ \\
&&
$=$ ($\mathcal{R}_1$ {\sc{xor}} $\mathcal{R}_2$)
\end{tabular}
\caption{Illustration of (a),(b) regions defined by the first neural
  network layer, and (c)--(f) various Boolean set operations applied
  to them in subsequent layers.}\label{Fig:Example3}
\end{figure}
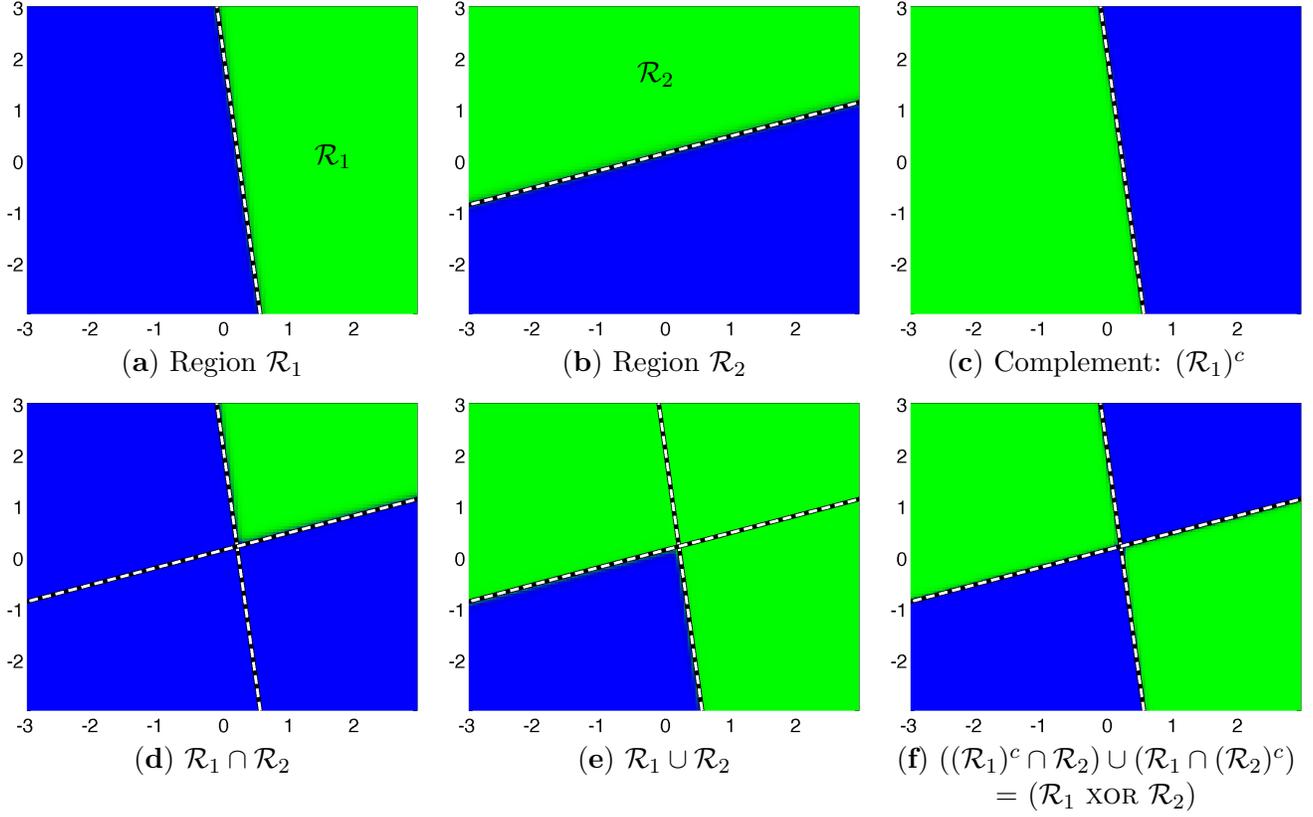

\paragraph{Constants}

Constants can be generated by choosing $A = 0$ and choosing a
sufficiently large positive or negative offset values. For example,
choosing $b = 100$ gives a region that spans the entire domain
(representing the logical {\sc{true}}), whereas choosing $b = -100$
results in the empty set (or logical {\sc{false}}).

\paragraph{Unary operations}

The simplest unary operation, the (Boolean) identity function, can be defined as
\[
A_{I} = 1, \qquad b_{I} = 0.\tag{Identity}
\]
This function works well when used in conjunction with a step
function, but has an undesirable damping effect when used with the
sigmoid function: input values up to 1 are mapped to output values up
to $\sigma(1) \approx 0.46$, and likewise for negative values. While
such scaling may be desirable in certain cases, we would like to
preserve the clear distinction between high and low confidence
regions. We can do this by scaling up $A$, which amplifies the input
to the sigmoid function and therefore its output.  Choosing $A_{I}=3$,
for example, would increase the maximum confidence level to $\sigma(3)
\approx 0.91$.  As noted towards the end of
Section~\ref{Sec:Halfspaces}, the same can be achieved by working with
the activation function $\sigma_3$, and to avoid getting distracted by
scaling issues like these we will work with $\nu_2 = \sigma_3$
throughout this section. We note that the identity function can be
approximated very well by scaling down the input and taking advantage
of the near-linear part of the sigmoid function around zero. The
output can then be scaled up again in the next layer to achieve the
desired result.
Similar to the identity function, we define the complement of a set as
\[
 A_{c} = -1,\qquad b_{c} =0.\tag{Complement}
 \]
 The application of this operator to $\mathcal{R}_1$ is illustrated in
 Figure~\ref{Fig:Example3}(c). Just to be clear, note that in this
 case the full parameters to the second layer would be $A = [-1,0]$
 and $b = 0$.

\paragraph{Binary operations}

When taking the intersection of regions $\mathcal{R}_1$ and
$\mathcal{R}_2$ we require that the output values of the corresponding
units in the network sum up to a value close to two. This is
equivalent to saying that when we subtract a relatively high value,
say $1.5$, from the sum, the outcome should remain positive. This
suggests the following parameters for binary intersection:
\[
A_{\,\cap} = [1,1],\qquad b_{\,\cap} = 1.5.\tag{Intersection}
\]

We now combine the intersection and complement operations to derive the
union of two sets, and to illustrate how complements of sets can be
applied during computations. By De Morgan's law, the union operator
can be written as $\mathcal{R}_1 \cup \mathcal{R}_2 =
((\mathcal{R}_1)^c \cap (\mathcal{R}_2)^c)^c$. Evaluation of this
expression is done in three steps: taking the individual complements
of $\mathcal{R}_1$ and $\mathcal{R}_2$, applying the intersection, and
taking the complement of the result. This can be written in linear form as
\[
A_{c}\left(A_{\,\cap}\left(\left[\begin{array}{cc}A_{I} & \\ &
         A_{c}\end{array}\right]x + \left[\begin{array}{c}b_{I}\\ b_{c}\end{array}\right]\right) + b_{\,\cap}\right) + b_c.
\]
Substituting the weight and bias terms and simplifying yields
parameters for the union:
\[
A_{\,\cup} = [1,1],\qquad b_{\,\cup} = -1.5.\tag{Union}\label{Eq:Union}
\]

It can be verified that the intersection can similarly be derived from
the union operator based using $\mathcal{R}_1 \cap \mathcal{R}_2 =
((\mathcal{R}_1)^c \cup (\mathcal{R}_2)^c)^c$ Results obtained with
both operators are shown in Figures~\ref{Fig:Example3}(e)
and~\ref{Fig:Example3}(f).

\subsubsection{General $n$-ary operations}\label{Sec:KofN}

We now consider general operations that combine regions from more than
two units. It suffices to look at a single output unit $\sigma(\langle
a,x\rangle - \beta)$ with weight vector $a$ and bias term $\beta$. Any
negative entry in $a$ mean that the corresponding input region is
negated and that its complement is used, whereas zero valued entries
indicate that the corresponding region is not used. Without loss
of generality we assume that all input regions are used and normalized
such that all entries in $a$ can be taken strictly positive. We again
start by looking at the idealized situation where inputs are generated
using a step function with outputs -1 or 1.  When $a$ is the vector of
all ones, and $k$ out of $n$ inputs are positive we have $\langle
a,x\rangle = k - (n-k) = 2k-n$. Choosing activation level $\beta =
2k-n-1$ therefore ensures that the output of the unit is positive
whenever at least $k$ out of $n$ inputs are positive. As extreme cases
of this we obtain the $n$-ary intersection with $k = n$, and the
$n$-ary union by choosing $k = 1$. Weights can be adjusted to indicate
how many times each region gets counted.

\begin{figure}[t]
\centering
\begin{tabular}{ccc}
\includegraphics[width=0.31\textwidth]{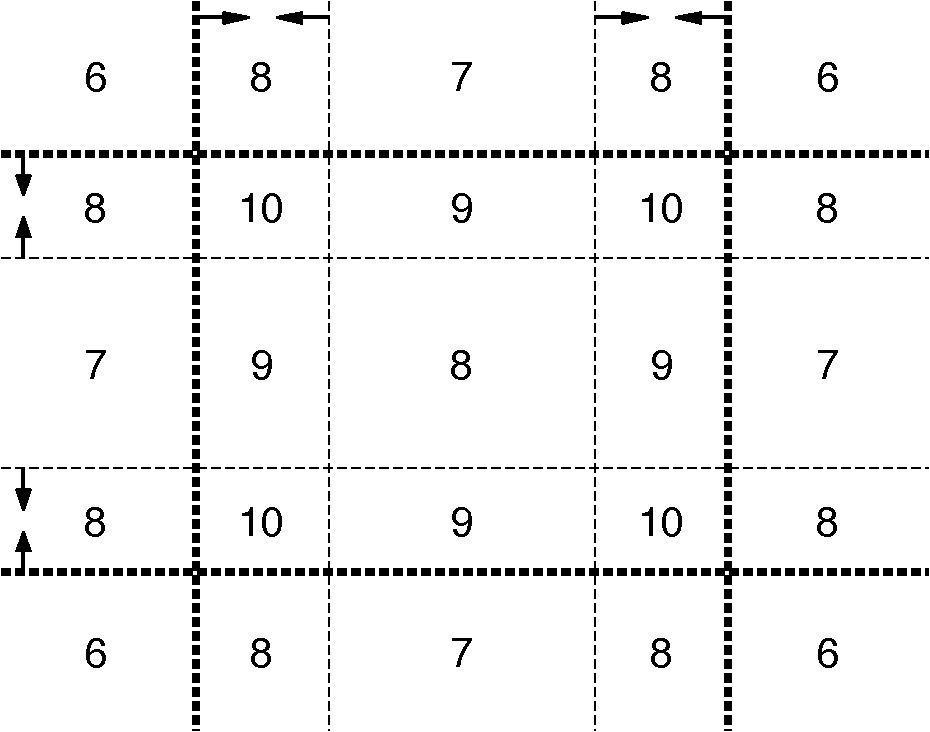}&
\includegraphics[width=0.31\textwidth]{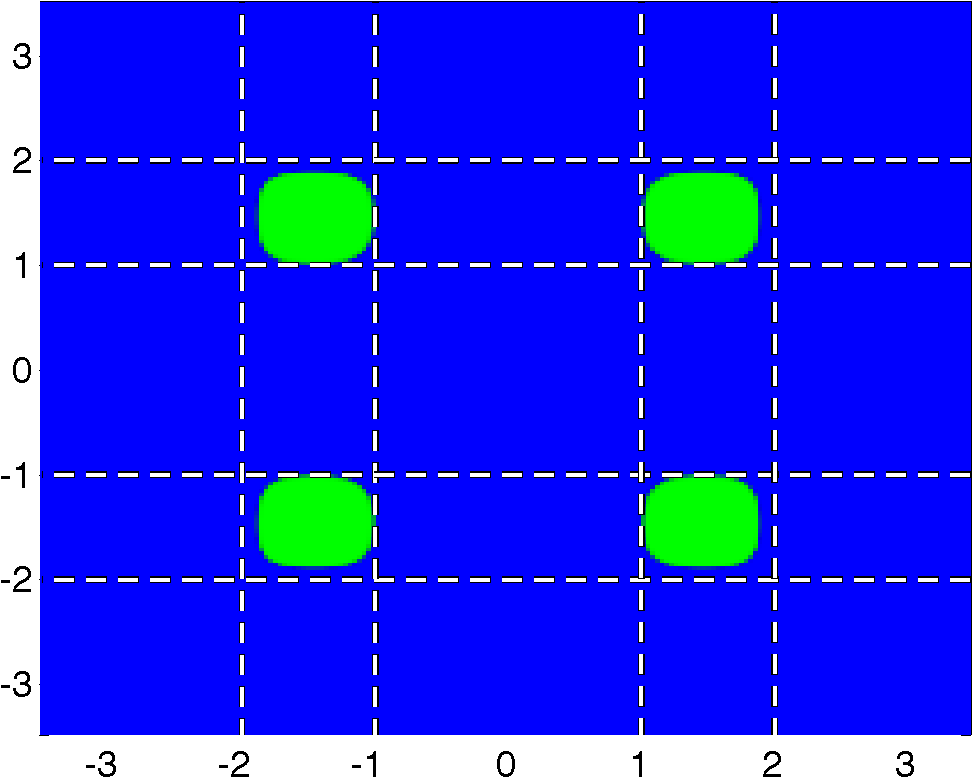}&
\includegraphics[width=0.31\textwidth]{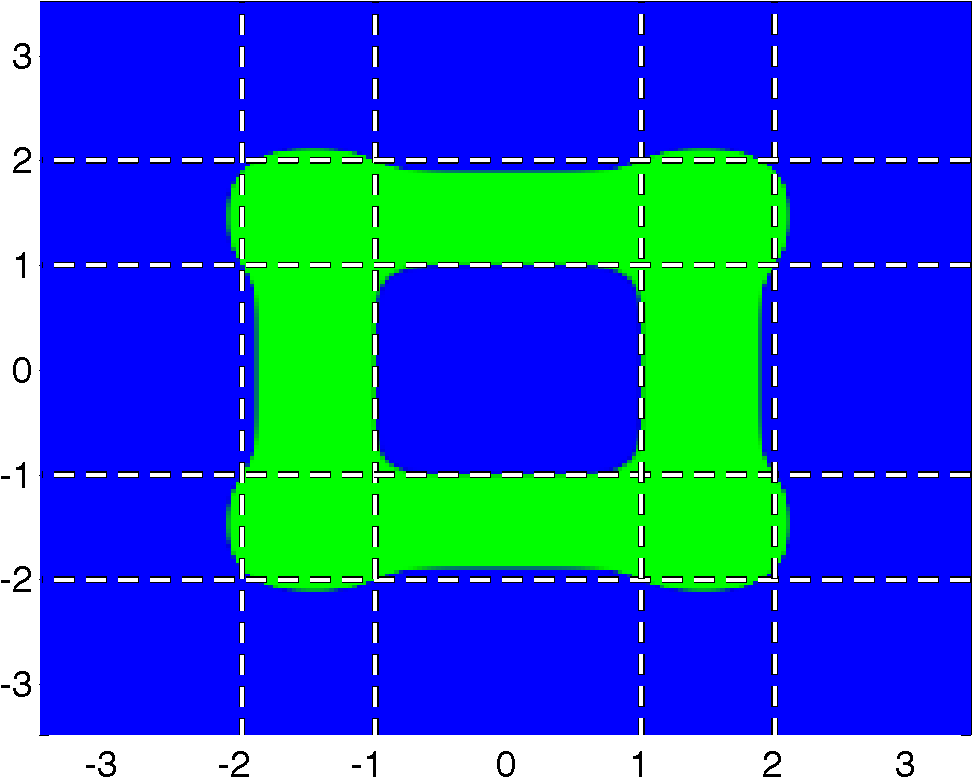}\\
({\bf{a}}) & ({\bf{b}}) & ({\bf{c}}) \\[6pt]
\includegraphics[width=0.31\textwidth]{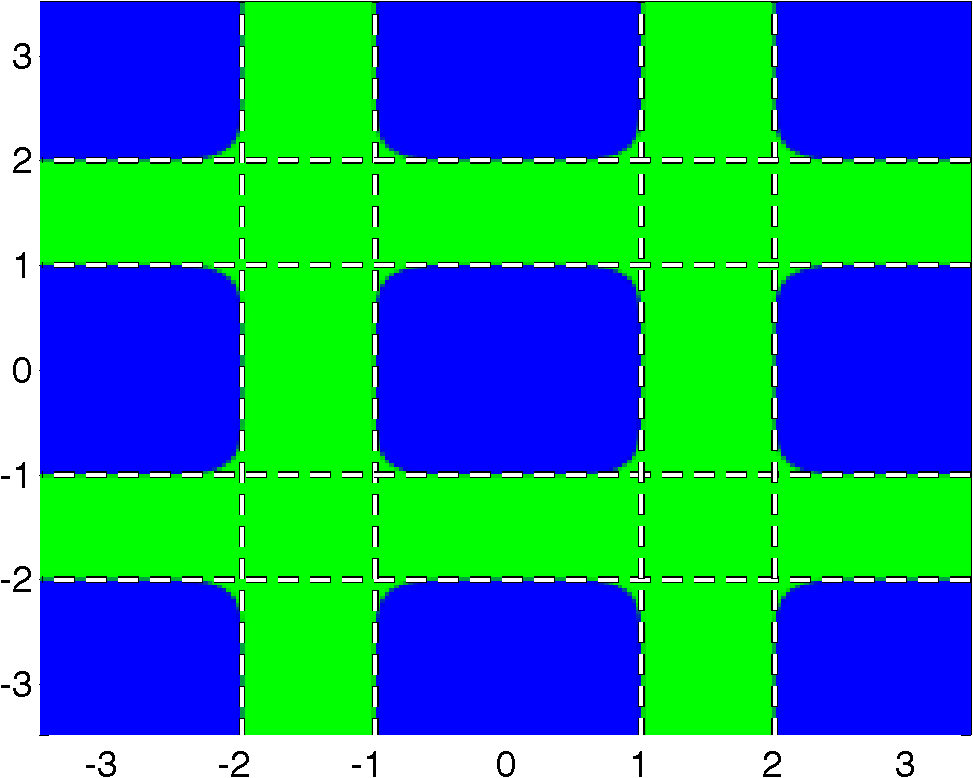}&
\includegraphics[width=0.31\textwidth]{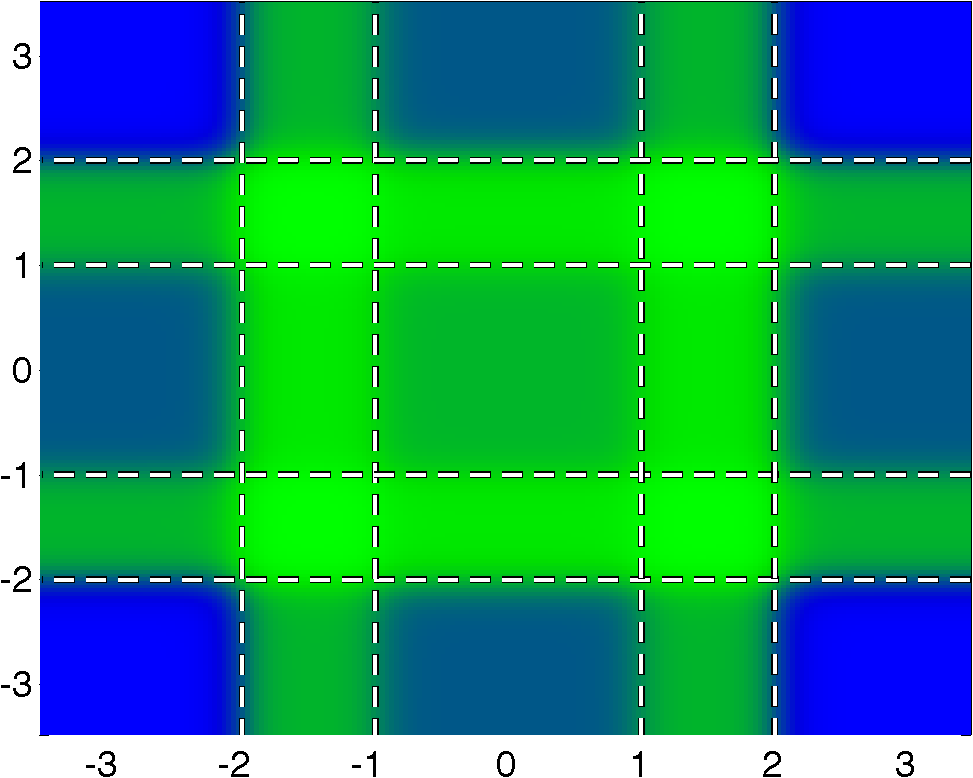}&
\includegraphics[width=0.31\textwidth]{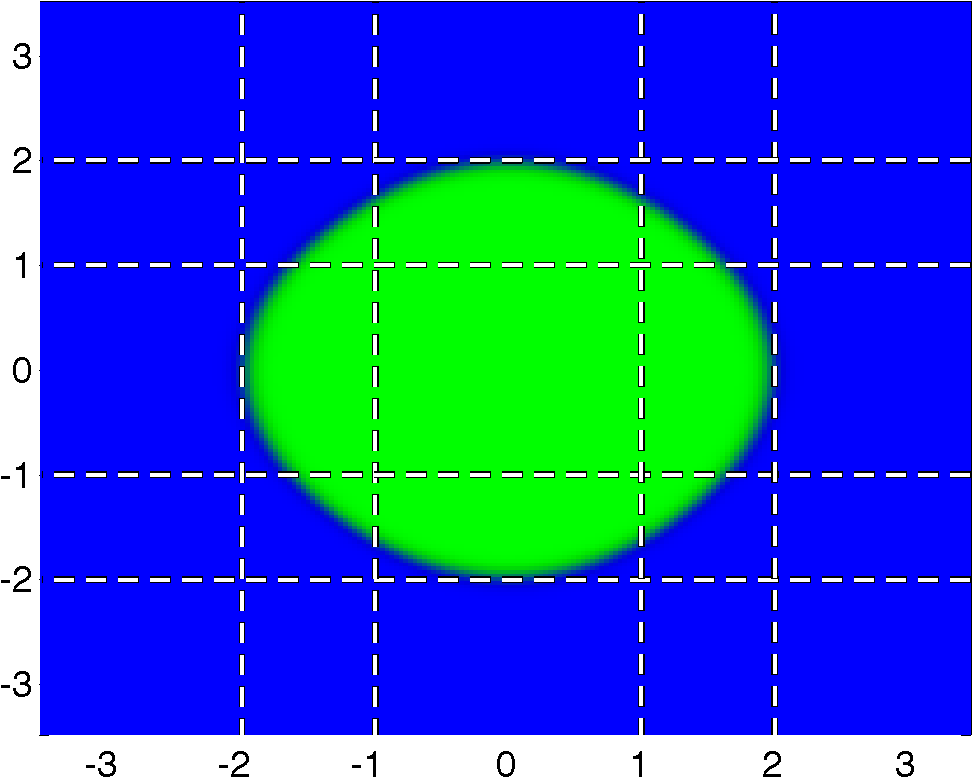}\\
({\bf{d}}) & ({\bf{e}}) & ({\bf{f}})\\[6pt]
\end{tabular}
\caption{Application of the $n$-ary operator. Plot (a) shows the
  location and orientation of the eight hyperplanes and the total
  weight for each of the cells when using weight two for the outer
  hyperplanes (thick dashed line) and unit weight for the inner
  hyperplanes (thin dashed line). Plots (b) and (c) show the regions
  formed when choosing $\beta = 9.5$ and $\beta = 8.5$ respectively,
  with activation function $\nu_1 = \ell_{10}$ and $\nu_2 =
  \sigma_{50}$. Plot (d) shows the region obtained when assigning unit
  weights to each hyperplanes and using $\beta = 4.5$. In plot (e) we
  change the settings from plot (b) by replacing the second activation
  function to $\sigma_1$ and using $\beta = 6.5$. The lack of
  amplification results in a region with four different confidence
  levels. Plot (f) illustrates the formation of a smooth circular
  region using only the outer four hyperplanes together with
  activation functions $\nu_1 = \ell_1$ and $\nu_2 = \sigma_{50}$, and
  threshold $\beta = 6.5$.}\label{Fig:ExampleKofN}
\end{figure}

It was noted by Huang and Littmann~\cite{HUA1988La} that complicated
and highly non-intuitive regions can be formed with the general
$n$-ary operations, even in the second layer. As an example, consider
the eight hyperplane boundaries plotted in
Figure~\ref{Fig:ExampleKofN}(a).  The weight assigned to each
hyperplane determines the contribution to each cell that lies within
the enclosed halfspace. The total contributions for each cell shown in
Figure~\ref{Fig:ExampleKofN}(a) represent the total weight obtained
when using weight two to the outer hyperplanes and a unit weight for
the inner hyperplanes, combined with step function input from
0 to 1. Adding up values for so many regions in a single step worsens
the scaling issue mentioned for the unitary operator: In this case
choosing a threshold of $\beta = 9.5$ leads to values ranging from
$-3.5$ to $0.5$ before application of the sigmoid function. Using
$\nu_1 = \ell_{10}$ and $\nu_2 = \sigma_{50}$ for amplification with
different weight vectors and threshold values we obtain the regions
shown in Figures~\ref{Fig:ExampleKofN}(b) to~\ref{Fig:ExampleKofN}(d).

Removing the large amplification factor in the second layer can lead
to regions with low or varying confidence levels. For the mixed
weights example, using $\nu_2 = \sigma_1$ and threshold $\beta = 6.5$
causes the intended region to have four distinct confidence levels, as
shown in Figure~\ref{Fig:ExampleKofN}(e). Low weights can also be
leveraged to obtain a parsimonious representation of smooth regions
that would otherwise require the many more halfspaces. An example of
this is shown in Figure~\ref{Fig:ExampleKofN}(f) in which the four
outer halfspaces with soft boundaries are combined to form a smooth
circular region.

\subsection{Boolean function representation using two layers}

As seen from Section~\ref{Sec:ElementaryBoolean} neural networks can
be used to take the union of intersections of (possibly negated)
sets. In Boolean logic this form is called disjunctive normal form
(DNF), and it is well known that any Boolean function can be expressed
in this form (see also \cite{ANT2003a}). Likewise we could reverse the
order of the union and intersection operators and arrive at
conjunctive normal form (CNF), which is equally powerful. Two-layer
networks are, in fact, far stronger than this and can be used to
approximate general smooth functions. More information on this can be
found in \cite[Sec. 4.3.2]{BIS1995a}.

\subsection{Boundary regions and amplification}\label{Sec:BoundaryShifts}

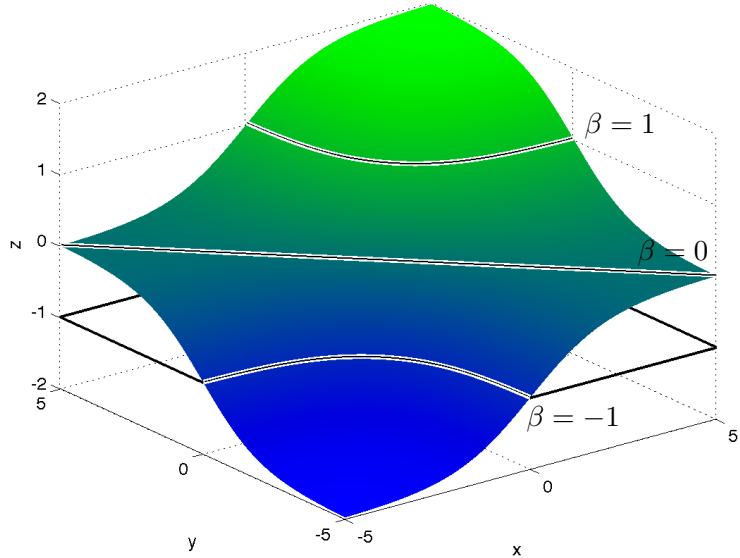
\begin{figure}
\centering
\includegraphics[width=0.55\textwidth]{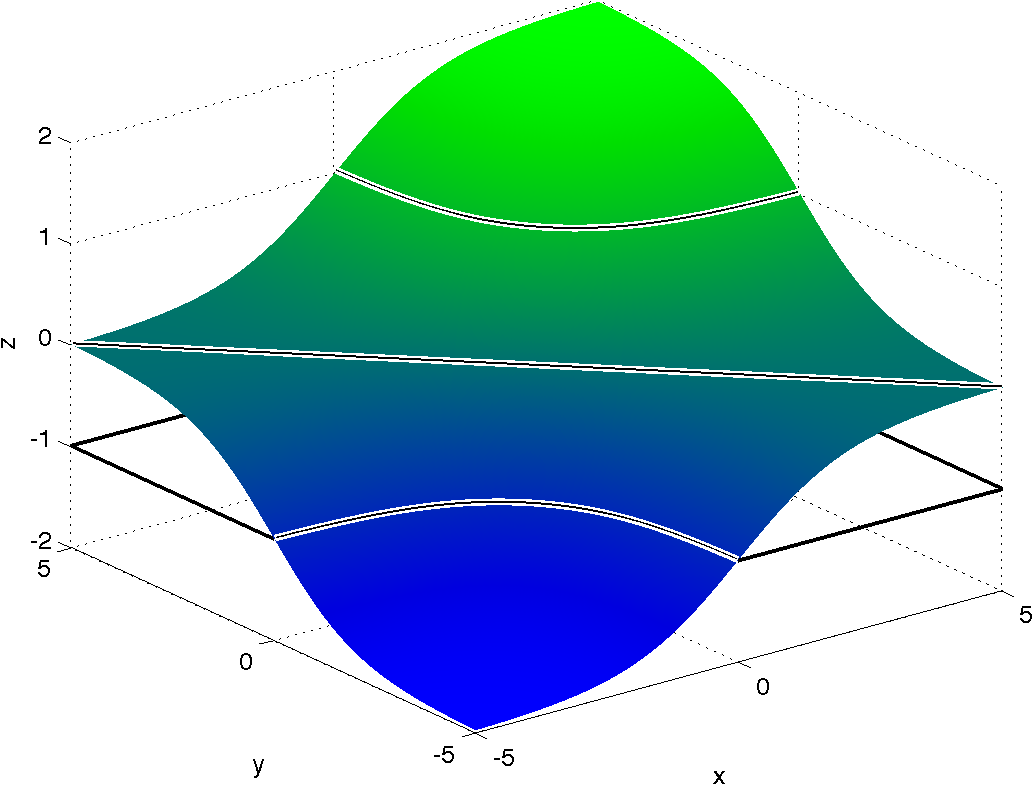}%
\begin{picture}(0,0)(0,0)
\put(-58,160){$\beta=1$}
\put(-38,112){$\beta=0$}
\put(-80,49){$\beta=-1$}
\end{picture}
\caption{Level sets of $z = \sigma(x) + \sigma(y)$ at values $\beta =
  -1,0,1$ along with a slice at $\beta = -1$.}\label{Fig:Levelsets2}
\end{figure}

The use of sigmoidal nonlinearity functions leads to continuous
transitions between different regions. The center of the transition
regions for a node can be defined as the set of feature points for
which the output of that node is zero. Given input $x_{k-1}$ for some
node at depth $k$ we first form $\langle a,x_{k-1}\rangle$, then
subtract the bias $\beta$, and apply the sigmoid function. The output
is zero if and only if $\langle a,x_{k-1}\rangle = \beta$, and the
transition center therefore corresponds to the level set of $\langle
a,x_{k-1}\rangle$ at $\beta$. For a fixed $a$ we can thus control the
location of the transition by changing $\beta$. As an example consider
a two-level neural network with the first layer parameterized by $A_1
= I$, $b_1 = 0$, and the second layer by $A_2 = [1,1]$, $b =
\beta$. Writing the input vector as $x_0 = [x,y]$ it can be seen that
$A_2 x_1 = \sigma(x) + \sigma(y)$, as illustrated in
Figure~\ref{Fig:Levelsets2}. All values greater than $\beta$ will be
mapped to positive values and, as discussed in
Section~\ref{Sec:ElementaryBoolean}, we again see that choosing $\beta
> 0$ approximates the intersection of the regions $x \geq 0$ and $y
\geq 0$, whereas choosing $\beta < 0$ approximates the union
(indicated in the figure by the lines at $z = -1$). What we are
interested in here is the location of the transition center. Clearly,
making the intersection more stringent by increasing $\beta$ causes
the boundary to shift and the resulting region to become
smaller. Another side effect is that the output range of the second
layer, which is given by $[\sigma(-2 - \beta),\sigma(2 - \beta)]$,
changes.  Choosing $\beta$ close to $2$, the supremum of the input
signal, means that the supremum of the output is close to zero,
whereas the infimum nearly reaches -1. To obtain larger positive
confidence levels in the output, without shifting the transition
center, we need to amplify the input by scaling $A_2$ and $b_2$ by
some $\gamma > 1$. In Figure~\ref{Fig:Levelsets} we study several
aspects of the boundary region corresponding to the setting used for
Figure~\ref{Fig:Levelsets2}, with the addition of scaling parameter
$\gamma$. For a given $\beta$ we choose $\gamma$ such that the maximum
output of $\sigma(\gamma(2 - \beta))$ is
$0.995$. Figures~\ref{Fig:Levelsets}(a)--(c) show the transition
region with values ranging from $-0.95$ and $0.95$ along with the
center of the transition with value $0$ and the region with values
exceeding $0.95$. Figure~\ref{Fig:Levelsets}(d) shows the required
scaling factors.

The ideal intersection of the two regions coincides with the positive
orthant and we define the shift in the transition boundary as the
limit of the $y$-coordinate of the zero crossing as $x$ goes to
infinity, giving
\[
\lim_{x\to\infty} \sigma^{-1}(\beta - \sigma(x)) = \sigma^{-1}(\beta - 1).
\]
The resulting shift values are show in
Figure~\ref{Fig:Levelsets}(e). Another property of interest is the
width of the transition region. Similar to the shift we quantify this
as the difference between the asymptotic $y$-coordinates of the
$-0.95$ and $0.95$ level set contours as $x$ goes to infinity. We
plot the results for several multiples of $\gamma$ in
Figure~\ref{Fig:Levelsets}(f). As expected, we can see that larger
amplification reduces the size of the transition intervals. The vertical
dashed line indicates the critical value of $\beta$ at which the
$-0.95$ contour becomes diagonal ($y = -x$) causing the transition
width to become infinite. The same phenomenon happens at smaller
$\beta$ when the multiplication factor is higher. Note that this break
down is due only to the definition of the transition width; the
transition region itself remains perfectly well defined throughout.

\begin{figure}[t]
\centering
\begin{tabular}{ccc}
\includegraphics[width=0.31\textwidth]{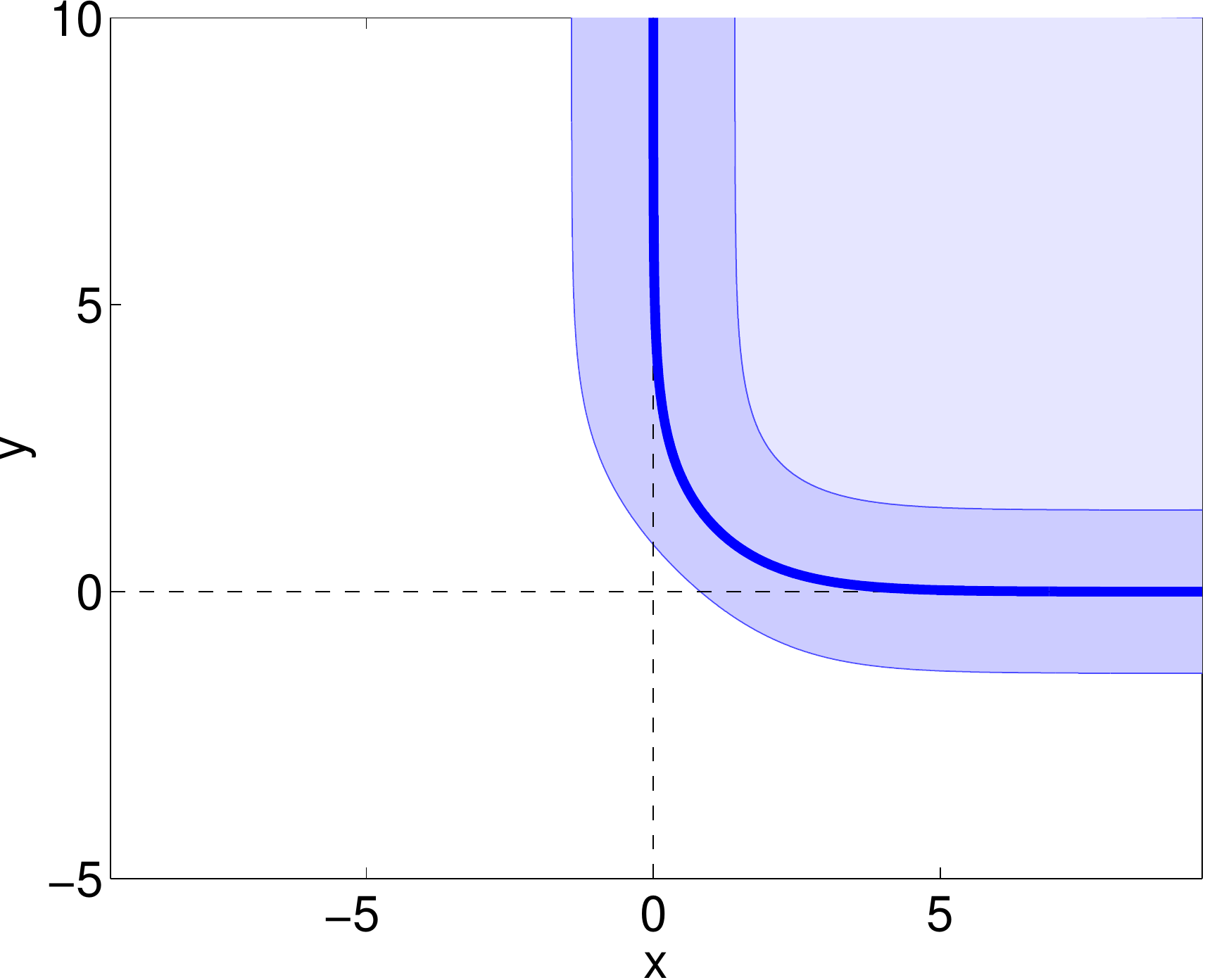}&
\includegraphics[width=0.31\textwidth]{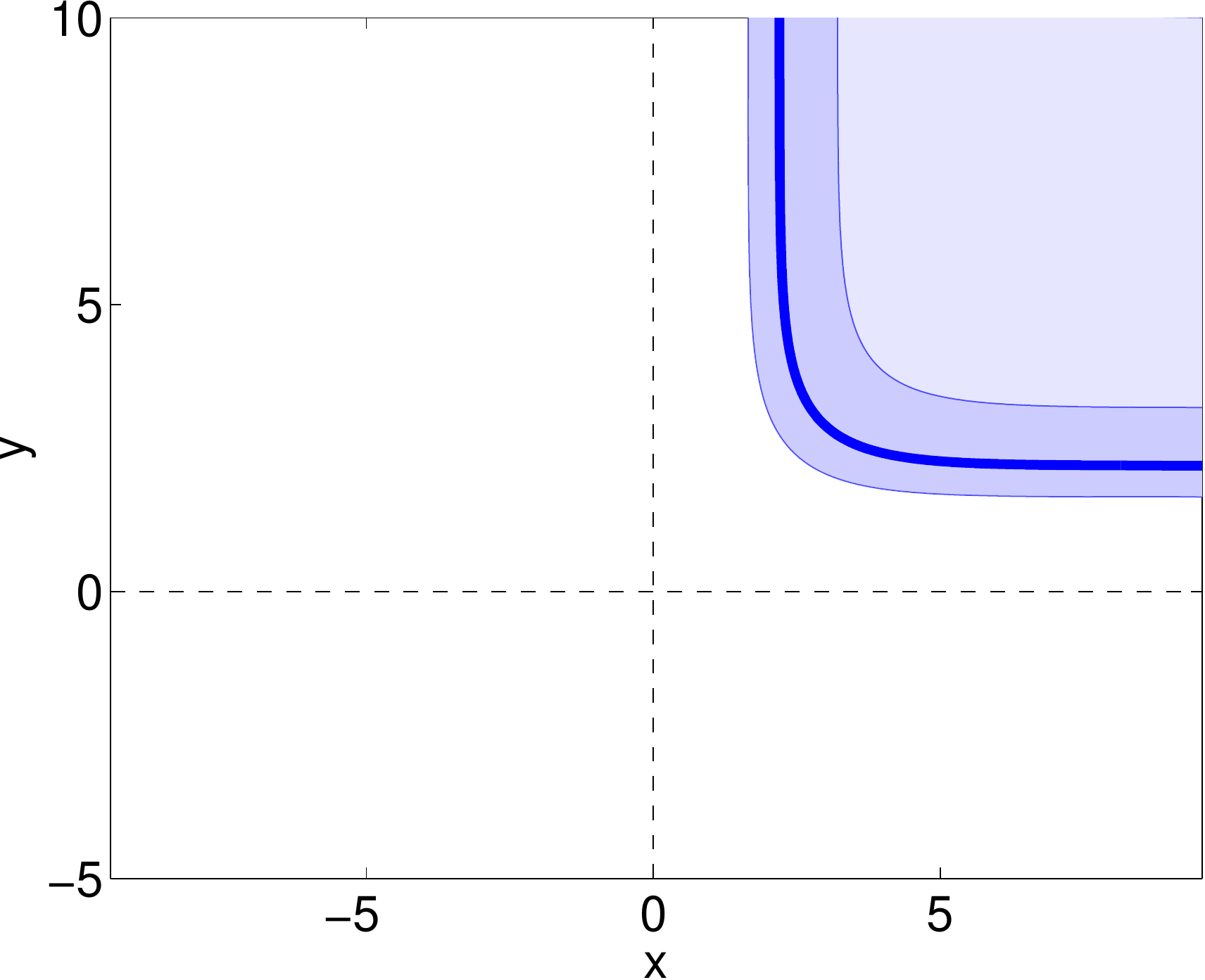}&
\includegraphics[width=0.31\textwidth]{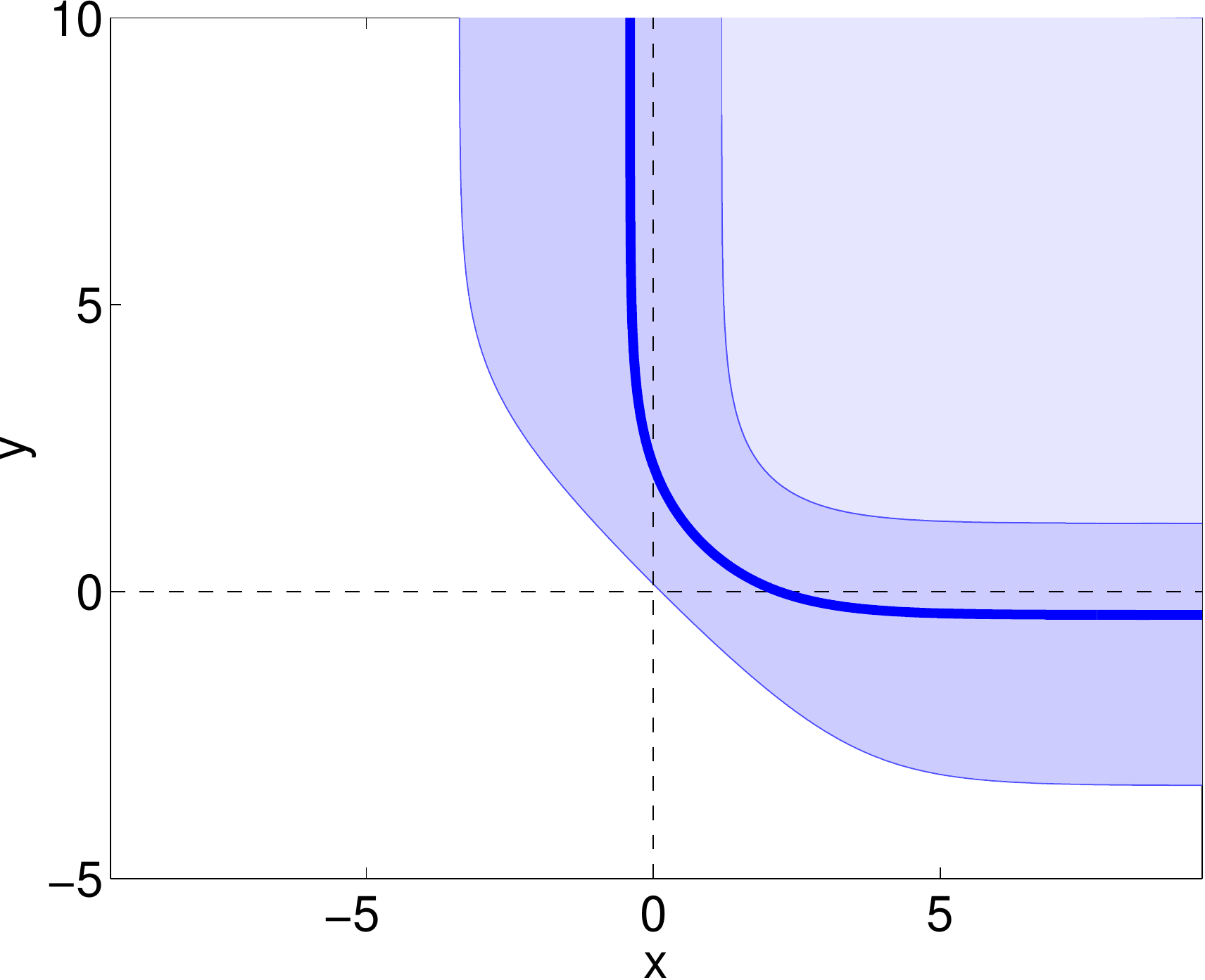}\\
({\bf{a}})  & ({\bf{b}})  & ({\bf{c}})\\[3pt]
\includegraphics[width=0.31\textwidth]{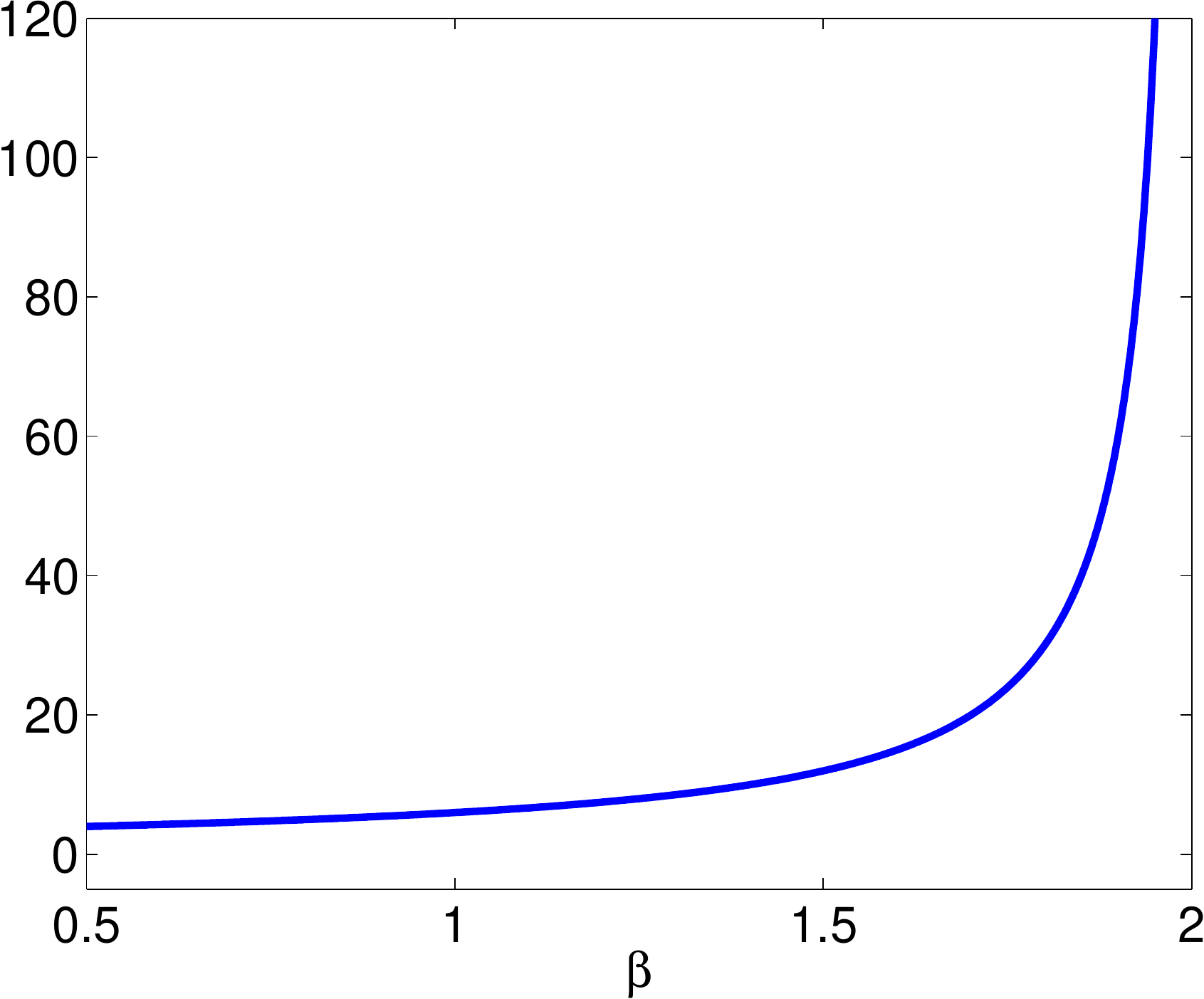}&
\includegraphics[width=0.31\textwidth]{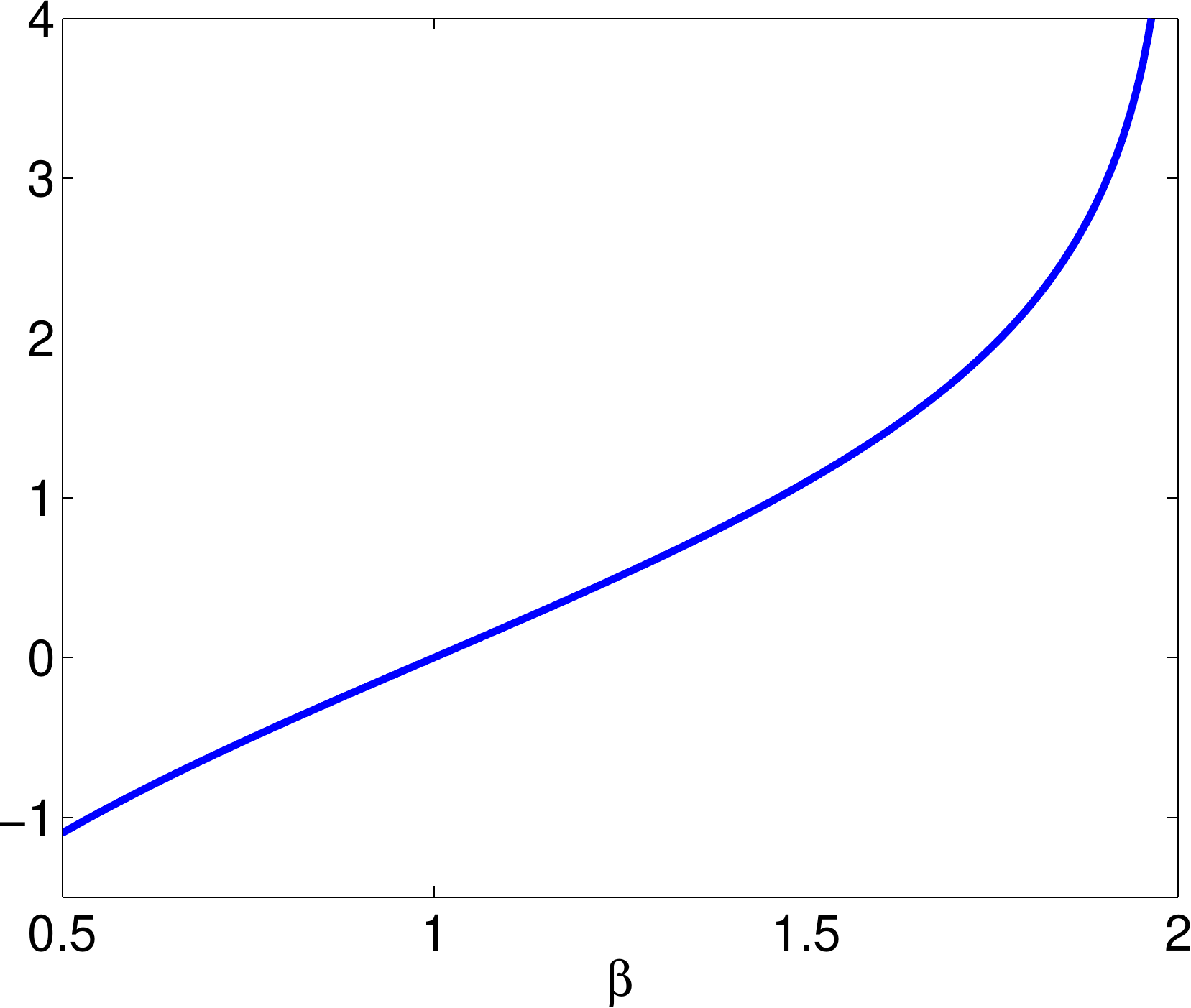} &
\includegraphics[width=0.31\textwidth]{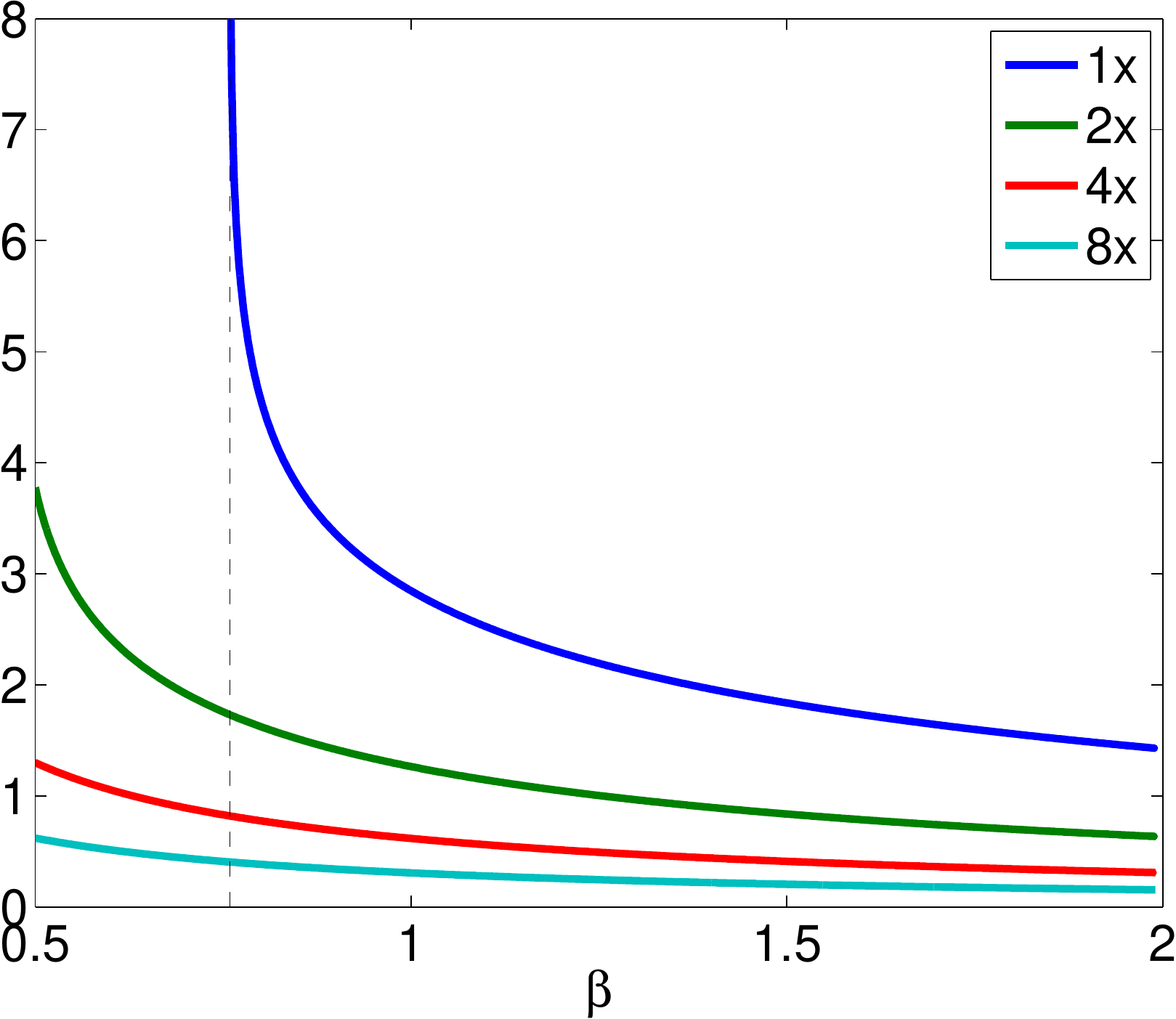}
\\
({\bf{d}}) & ({\bf{e}}) & ({\bf{f}})
\end{tabular}
\caption{Transition regions for
  $\sigma_\gamma(\sigma(x)+\sigma(y)-\beta)$ with contour lines at $0$
  and $0.95$ of the minimum and maximum values for (a) $\beta = 1$,
  (b) $\beta = 1.8$, and (c) $\beta = 0.6$. The value of scaling
  factor $\gamma$ is chosen such that output range reaches at least
  $\pm 0.995$. Plot (d) shows the required scaling factor as a
  function of $\beta$. The shift in the transition region is plotted
  in (e). Plot (f) shows the transition width as a function of
  $\beta$, using different multiples of $\gamma$.}\label{Fig:Levelsets}
\end{figure}

\subsection{Continuous to discrete}

The level-set nature of applying the nonlinearity as illustrated in
Figure~\ref{Fig:Levelsets2} allows the generation of decision
boundaries that look very different from any one of those used for its
input. One example of this was shown in
Figure~\ref{Fig:ExampleKofN}(f) in which a circular region was
generated by four axis-aligned hyperplanes, and we now describe
another. Consider the two hyperplanes in
Figures~\ref{Fig:RotateBoundary}(a,b), generated in the first layer
with respectively $a = [0.1,-0.1]$, $b = 0$, and $a = [0.1,0.1]$,
$b=0$. The small weights and the limited domain size cause the input
values to the nonlinearity to be small. As a result, the sigmoid
operates in its near-linear region around the origin and therefore
resembles scalar multiplication. Consequently, because the normals of
the first layers form a basis, we can use the second layer to
approximate any operation that would normally occur in the first
layer. For example we can choose $a_2 = [\cos(\alpha + \pi/4),
\sin(\alpha + \pi/4)]$ and $b_2 = 0$ to generate a close approximation
of a hyperplane at angle $\alpha$ (up to a scaling factor this weight
matrix is formed by multiplying the desired normal vector $a$ by the
rotation on the inverse of the weight matrix of the first layer). The
resulting regions of the second layer are shown in
Figures~\ref{Fig:RotateBoundary}(c,d) for $\alpha = 90^\circ$ and
$\alpha = 70^\circ$, respectively. This illustrates that, although
somewhat contrived, it is technically possible, at least locally, to
change hyperplane orientation after the first layer.

As decision regions propagate and form through one or more layers with
modest or large weights, their boundaries become sharper and we see a
gradual transition from continuous to discrete network behavior. In
the continuous regime, where the transitions are still gradual, the
decision boundaries emerge as level sets of slowly varying smooth
functions and therefore change continuously and considerably with the
choice of bias term. As the boundary regions become sharper the
functions tend to piecewise constant causing the level sets to change
abruptly only at several critical values while remaining fairly
constant otherwise, thus giving more discrete behavior. In
Figures~\ref{Fig:RotateBoundary}(e,f) we show intermediate stages in
which we scale the weights in the first layer
Figures~\ref{Fig:RotateBoundary}(d) by a factor of 10 and 20,
respectively. In addition, it can be seen that scaling in this case
does not just sharpen the boundaries, but actually severely distorts
them. Finally, it can be seen that the resulting region becomes
increasingly diagonal (similar to its sharpened input) as the weights
increase. This again emphasizes the more discrete nature of region
combinations once the boundaries of the underlying regions are sharp.

\begin{figure}
\centering
\begin{tabular}{ccc}
\includegraphics[width=0.31\textwidth]{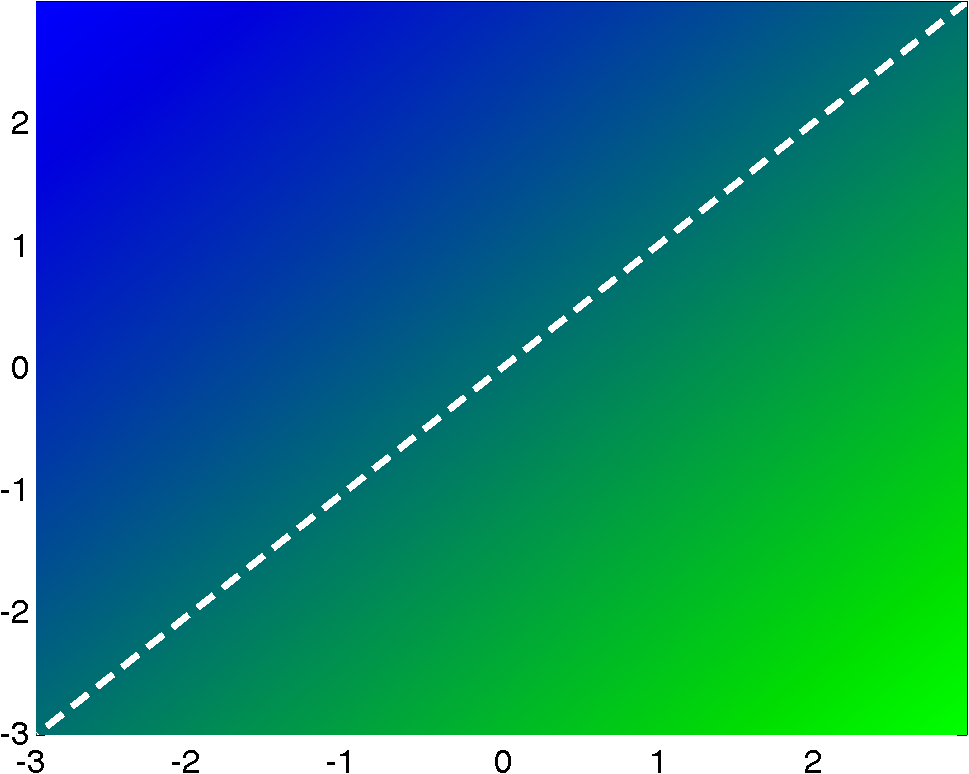}&
\includegraphics[width=0.31\textwidth]{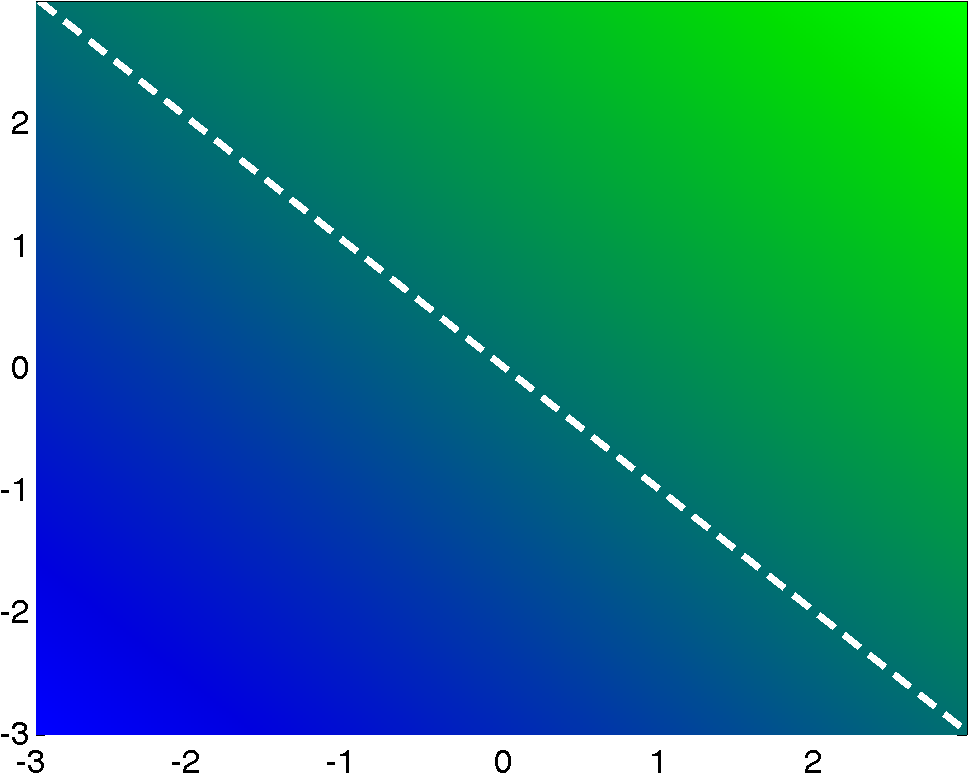}&
\includegraphics[width=0.31\textwidth]{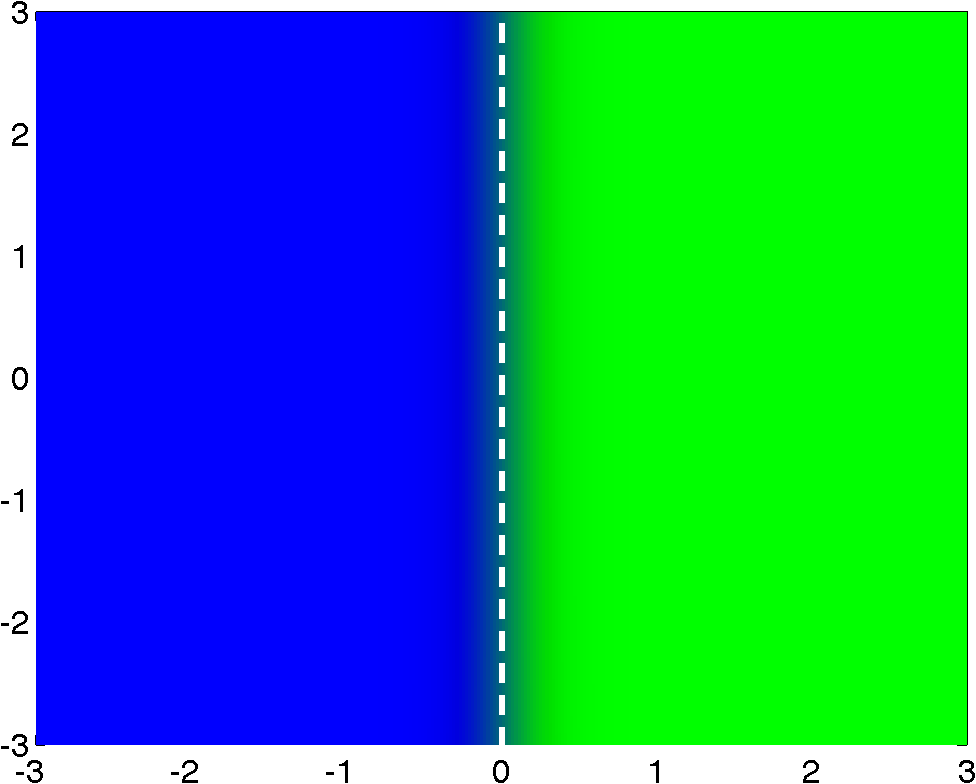}\\
({\bf{a}}) & ({\bf{b}}) & ({\bf{c}}) \\[4pt]
\includegraphics[width=0.31\textwidth]{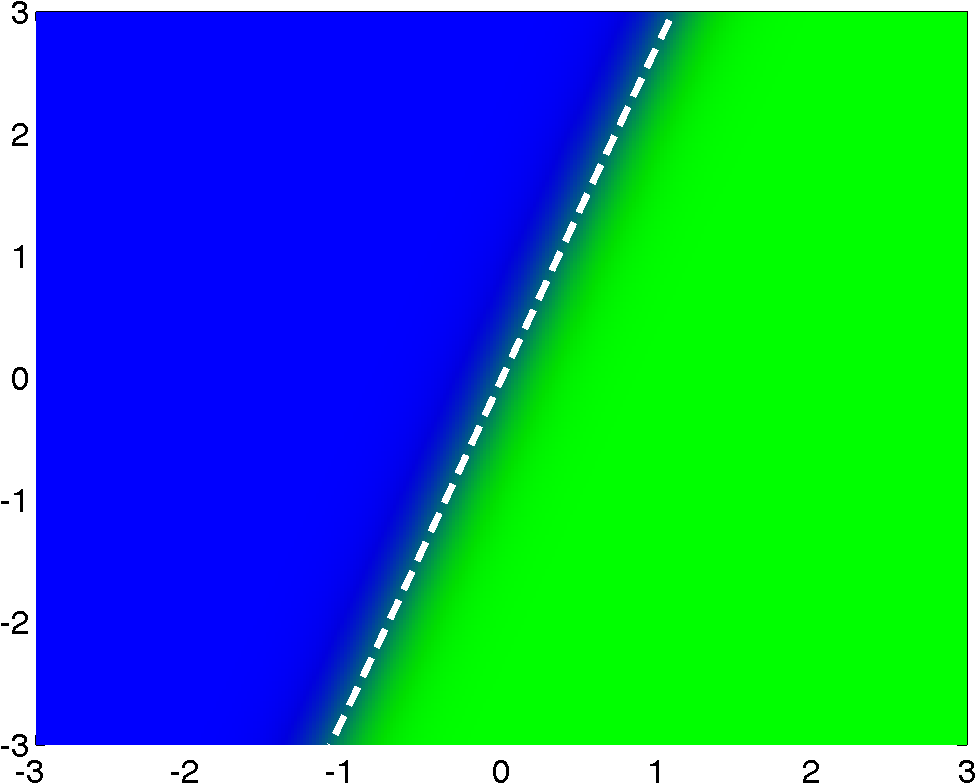}&
\includegraphics[width=0.31\textwidth]{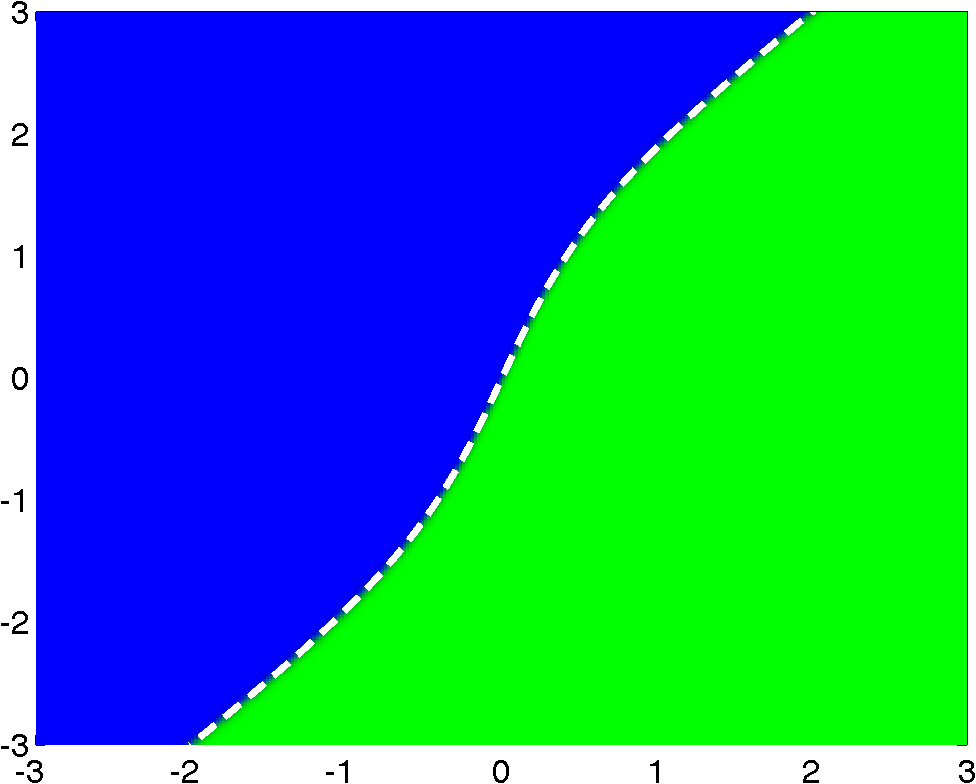}&
\includegraphics[width=0.31\textwidth]{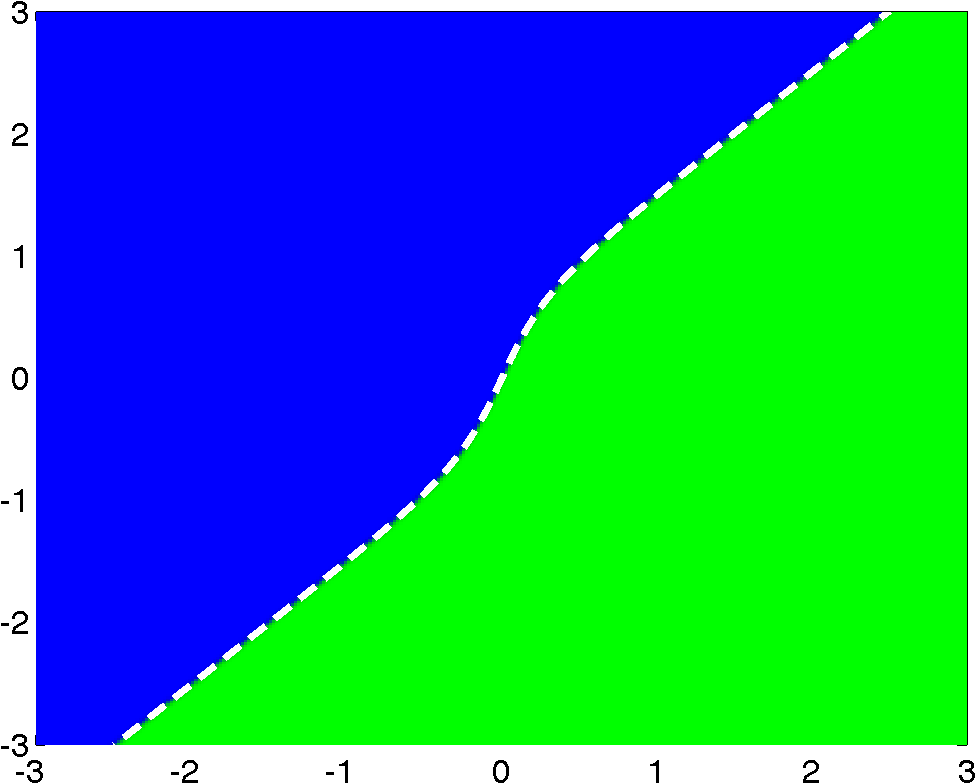}\\
({\bf{d}}) & ({\bf{e}}) & ({\bf{f}})
\end{tabular}
\caption{Combination of (a,b) two smooth regions defined by a diagonal
  hyperplane into (c) a vertical region, and (d) a region at an angle
  of $70$ degrees. Using the setup for (d) with scaled weights in the
  first layer gives the region shown in (e) for weight factor 10, and (f) for
  weight factor 20.}\label{Fig:RotateBoundary}
\end{figure}

\subsection{Generalized functions for the first layer}

The nodes in the first layer define geometric primitives, which are
combined in subsequent layers. Depending on the domain it may be
desirable to work with primitives other than halfspaces, or to provide
a set of different types. This can be achieved by replacing the inner
products in the first layer by more general functions $f_{\theta}(x)$
with training examples $x$ and (possibly shared) parameters
$\theta$. The traditional hyperplane is given by
\[
f_{\theta}(x) = \langle a, x\rangle + \beta,\qquad \theta = (a,\beta)
\]
For ellipsoidal regions we could then use
\[
f_{\theta}(x) = \alpha \norm{Ax-b}_2^2 + \beta,\qquad
\theta = (A,b,\alpha,\beta).
\]
More generally, it is possible to redefine the entire unit by
replacing both the inner-product and the nonlinearity with a general
function to obtain, for example, a radial-basis function unit
\cite{LEC1998BOMa}. In Figure~\ref{Fig:Primitives} we illustrate how a
mixture of two types of geometric primitives can form regions that
cannot be expressed concisely with either type alone.

\begin{figure}
\centering
\begin{tabular}{ccc}
\includegraphics[width=0.3\textwidth]{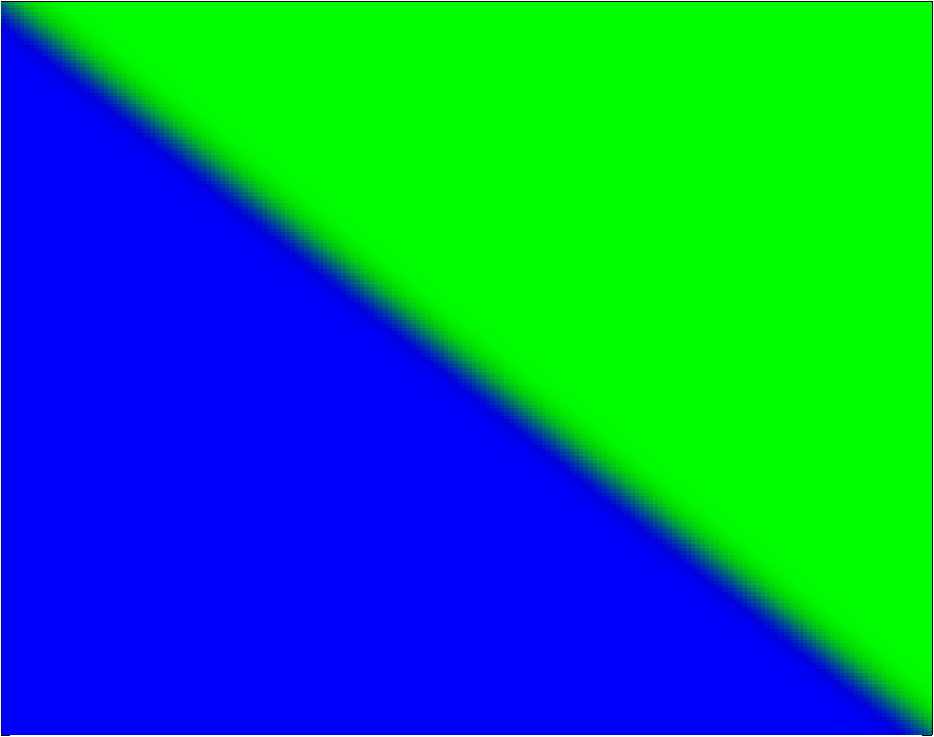}&
\includegraphics[width=0.3\textwidth]{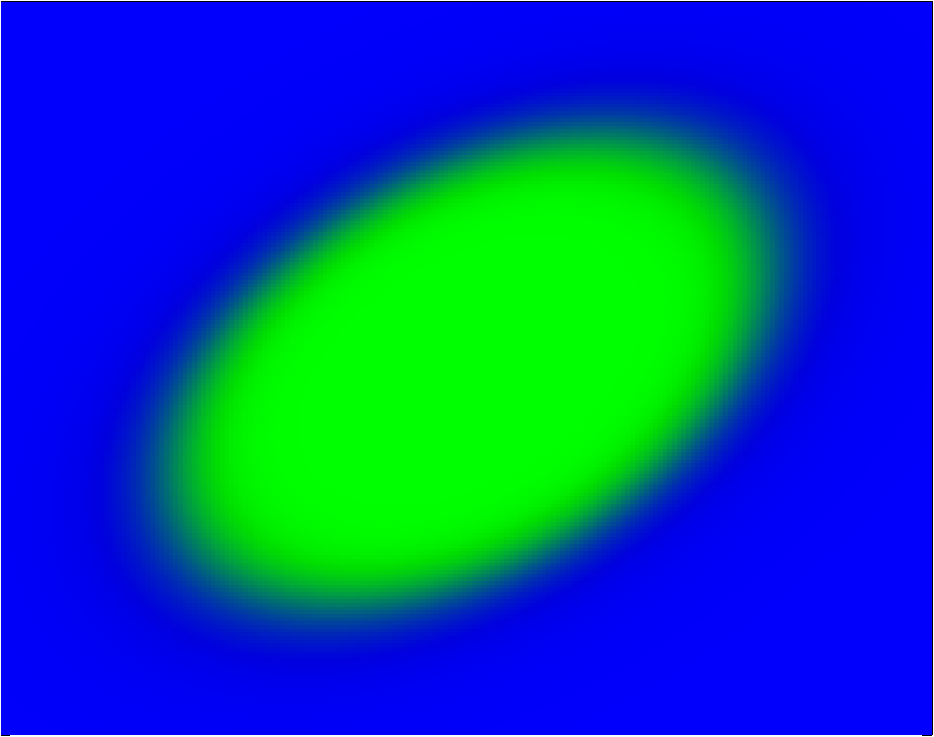}&
\includegraphics[width=0.3\textwidth]{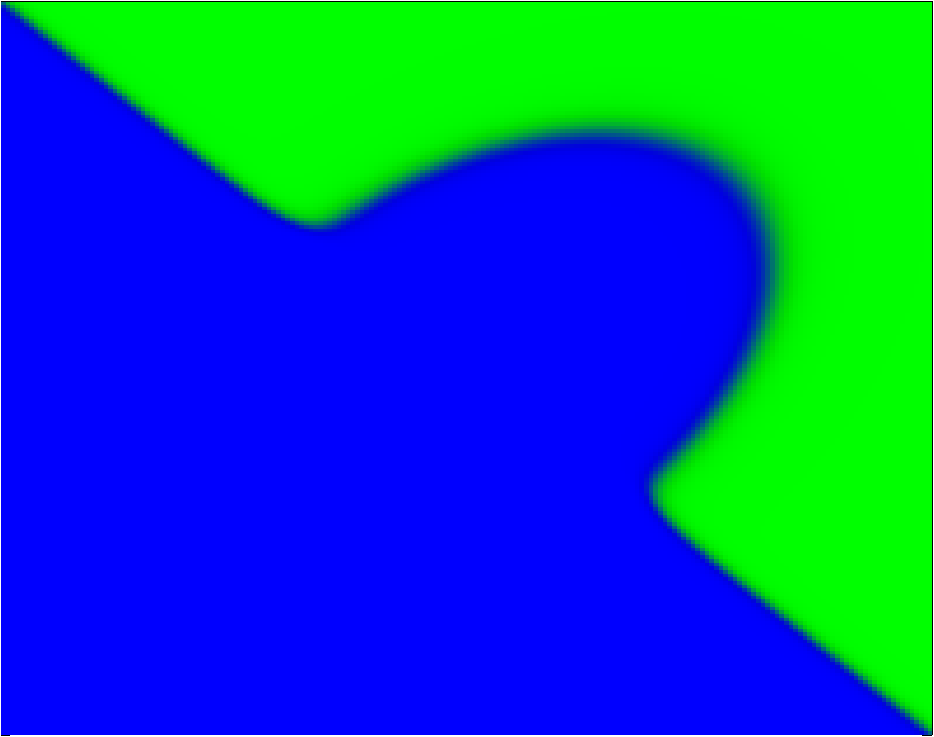}\\
({\bf{a}}) Linear, $\mathcal{R}_1$ &
({\bf{b}}) Gaussian, $\mathcal{R}_2$ &
({\bf{c}}) $\mathcal{R}_1 \cap \mathcal{R}_2^c$ \\
\end{tabular}
\caption{Shape primitives of type (a) halfspace, and (b) Gaussian,
  are combined to obtain (c).}\label{Fig:Primitives}
\end{figure}

\section{Region properties and approximation}\label{Sec:Regions}
The hyperplanes defined by the first layer of the neural network
partition the space into different regions. In this section we discuss
several combinatorial and approximation theoretic properties of these
regions.

\subsection{Number of regions}\label{Sec:NumberOfRegions}

One of the most fundamental properties to consider is the maximum
number of regions into which $\mathbb{R}^d$ can be partitioned using
$n$ hyperplanes. The exact maximum is well known to be
\begin{equation}\label{Eq:MaximumNumberOfRegions}
r(n,d) = \sum_{i=0}^{d}{n \choose i},
\end{equation}
and is attained whenever the hyperplanes are in general position
\cite[p.39]{SCH1901a}. With the hyperplanes in place, the subsequent
logic layers in the neural network can be used to identify each of
these regions by taking the union of (complements of)
halfspaces. Individual regions can then be combined using the union
operator.

%
%
%
%
%
%
%
%
%
%

\subsection{Approximate representation of classes}

\subsubsection{Polytope representation}
When the set of points $\mathcal{C} \subset \mathbb{R}^d$ belonging to
a class form a bounded convex set, we can approximate it by a polytope
$\mathcal{P}$ given by the bounded intersection of a finite number of
halfspaces. The accuracy of such an approximation can be expressed as
the Hausdorff distance between the two sets, defined as:
\[
\rho_{H}(C,\mathcal{P}) := \max\left[\sup_{x\in\mathcal{C}}
d(x,\mathcal{P}),\ \sup_{x\in\mathcal{P}} d(x,\mathcal{C})\right],
\]
with
\[
d(x,\mathcal{S}) := \inf_{y \in \mathcal{S}} \norm{x-y}_2.
\]
For a given class of convex bodies $\Sigma$, denote
$\delta_{H}(\mathcal{C},\Sigma) := \inf_{\mathcal{V}\in\Sigma}
\rho_{H}(\mathcal{C},\mathcal{V})$. We are interested in
$\delta_H(\mathcal{C},\Sigma)$ when $\Sigma = \mathfrak{R}_{(n)}^d$,
the set of all polytopes in $\mathbb{R}^d$ with at most $n$ facets
(i.e., generated by the intersection of up to $n$ halfspaces), and in
particular how it behaves as a function of $n$. The following result
obtained independently by \cite{BRO1975Ia,DUD1974a} is given in
\cite{BRO2008a}. For every convex body $\mathcal{U}$ there exists a
constant $c(\mathcal{U})$ such that
\[
\delta_H(\mathcal{U},\mathfrak{R}_{(n)}^d) \leq \frac{c(\mathcal{U})}{n^{2/(d-1)}}.
\]
More interesting perhaps is a lower bound on the approximation
distance. For the unit ball $\mathcal{B}$ we have the following:

\begin{theorem}\label{Thm:HausdorffSphere}
  Let $\mathcal{B}$ denote the unit ball in $\mathbb{R}^d$. Then for
  sufficiently large $n$ there exists a constant $c_d$ such that
\[
\delta_H(\mathcal{B},\mathfrak{R}_{(n)}^{d}) \geq
\frac{c_d}{n^{2/(d-1)}},
\]
\end{theorem}
\begin{proof}
  For $n$ large enough there exists a polytope $\mathcal{P} \in
  \mathfrak{R}_{(n)}^d$ with $n$ facets and $\delta :=
  \delta_H(\mathcal{B},\mathcal{P}) \leq 9/64$. Each of the $n$ facets
  in $\mathcal{P}$ is generated by a halfspace, and we can use each
  halfspace to generate a point on the unit sphere in $\mathbb{R}^d$
  such that the surface normal at that point matches the outward
  normal of the halfspace. We denote the set of these points by
  $\mathcal{N}$, with $\vert\mathcal{N}\vert = n$. Now, take any point
  $x$ on the unit sphere. From the definition of $\delta$ it follows
  that the maximum distance between $x$ and the closest point on one
  of the hyperplanes bounding the halfspaces is no greater than
  $\delta$. From this it can be shown that the distance to the nearest
  point in $\mathcal{N}$ is no greater than $\epsilon :=
  2\sqrt{\delta}$. Moreover, because the choice of $x$ was arbitrary,
  it follows that $\mathcal{N}$ defines an $\epsilon$-net of the unit
  sphere. Lemma~\ref{Lemma:EpsilonNet} below shows that the
  cardinality $\vert\mathcal{N}\vert \geq
  \frac{c'}{\epsilon^{(d-1)}}$.
%
%
Substituting $\epsilon = 2\sqrt{\delta}$ gives
\[
n \geq \frac{c'}{2^{(d-1)} \delta^{(d-1)/2}},\quad\mbox{or}\quad
\delta \geq \frac{c_d}{n^{2/(d-1)}}.
\]
\end{proof}

\begin{lemma}\label{Lemma:EpsilonNet}
  Let $\mathcal{N}$ be an $\epsilon$-net of the unit sphere
  $\mathcal{S}^{d-1}$ in $\mathbb{R}^{d}$ with $\epsilon \leq 3/4$,
  then
\[
\vert\mathcal{N}\vert \geq \sqrt{2\pi (d-1)/d}\cdot\epsilon^{1-d}.
\]
\end{lemma}
\begin{proof}
  By definition of the $\epsilon$-net, we obtain a cover for
  $\mathcal{S}^{d-1}$ by placing balls of radius $\epsilon$ at all $x
  \in \mathcal{N}$. The intersection of each ball with the sphere
  gives a spherical cap. The union of the spherical caps covers the
  sphere and $\abs{\mathcal{N}}$ times the area of each spherical cap
  must therefore be at least as large as the area of the sphere. A
  lower bound on the number of points in $\mathcal{N}$ is therefore
  obtained by the ratio $\nu$ between the area of the sphere and that
  of the spherical cap (see also \cite[Lemma 2]{WYN1967a}). Denoting
  by $\varphi = \arccos(1-\half \epsilon^2)$ the half-angle of the
  spherical cap it follows from \cite[Corollary 3.2(iii)]{BOR2003Wa}
  that $\nu$ satisfies
\[
1/\nu < \frac{1}{\sqrt{2\pi
    (d-1)}}\cdot\frac{1}{\cos\varphi}\cdot\sin^{d-1}\varphi,
\]
whenever $\varphi \leq \arccos 1/\sqrt{d}$. This bound can be
substituted into the second term above to obtain $\sqrt{d}$, and it
can be verified to hold whenever $\epsilon \leq 3/4$. It further holds
that $\sin \varphi < \epsilon$ which, after rewriting, gives the
desired result.
\end{proof}

\subsubsection{More efficient representations}

From Theorem~\ref{Thm:HausdorffSphere} we see that a large number of
supporting hyperplanes is needed to define a polytope that closely
approximates the unit $\ell_2$-norm ball. Approximating such a ball or
any other convex sets by the intersection of a number of halfspaces
can be considered wasteful, however, since it uses only a single
region out of the maximum $r(n,d)$ given by
\eqref{Eq:MaximumNumberOfRegions}. This fact was recognized by Cheang
and Barron \cite{CHE2000Ba}, and they proposed an alternative
representation for unit balls that only requires
$\mathcal{O}(d^2/\delta^2)$ halfspaces---far fewer than the
conventional $\mathcal{O}(1/\delta^{(d-1)/2})$. The construction is as
follows: given a set of $n$ suitably chosen halfspaces $\mathcal{H}_i$
and the indicator function $1_{\mathcal{H}_i}(x)$ which is one if $x
\in\mathcal{H}_i$ and zero otherwise. Typically these halfspaces are
used to define polytope $\mathcal{P} := \{ x \in \mathbb{R}^d \mid
\sum_i 1_{\mathcal{H}_i}(x) = n\}$, that is, the intersection of all
halfspaces. The (non-convex) approximation proposed in
\cite{CHE2000Ba} is of the form
\[
\mathcal{Q} := \{ x \in \mathbb{R}^d \mid
\sum_i 1_{\mathcal{H}_i}(x) \geq k\},
\]
which consists of all points that are contained in at least $k$
halfspaces. This representation is shown to provide far more efficient
approximations, especially in high dimensions.  As described in
Section~\ref{Sec:KofN}, this construction can easily be implemented as
a neural network. A similar approximation for the Euclidean ball,
which also takes advantage of smooth transition boundaries is shown in
Figure~\ref{Fig:ExampleKofN}(f).

\subsection{Bounds on the number of separating hyperplanes}

In many cases, it suffices to simply distinguish between the different
classes instead of trying to exactly trace out their boundaries. Doing
so may reduce the number of parameters and additionally help reduce
overfitting.  The bound in Section~\ref{Sec:NumberOfRegions} gives the
maximum number of regions that can be separated by a given number of
hyperplanes. Classes found in practical applications are extremely
unlikely to exactly fit these cells, and we can therefore expect that
more hyperplanes are needed to separate them. We now look at the
maximum number of hyperplanes that is needed.

\subsubsection{Convex sets}
In this section we assume that the classes are defined by convex sets
whose intersection is either empty or of measure zero.  We are
interested in finding the minimum number of hyperplanes needed such
that each pair of classes is separated by at least one of the
hyperplanes. In the worst case, a hyperplane is needed between any
pair of $n$ classes, giving a maximum of $n \choose 2$ hyperplanes,
independent of the ambient dimension. That this maximum can be reached
was shown by Tverberg \cite{TVE1979a} who provides a construction due
to K.P.~Villanger of a set of $n$ lines in $\mathbb{R}^3$ such that
any hyperplane that separates one pair of lines, intersects all
others. Here we describe a generalization of this construction for odd
dimensions $d \geq 3$.

\begin{theorem}\label{Thm:TverbergRd}
  Let $A = [A_1,A_2,\ldots,A_n]$ be a full-spark\cite{DON2003Ea} matrix with blocks
  $A_i$ of size $d \times (d-1)/2$, with odd $d \geq 3$. Let $b_i$,
  $i=1,\ldots,n$ be vectors in $\mathbb{R}^d$ such that
  $[A_i,A_j,b_i-b_j]$ is full rank for all $i\neq j$. The subspaces
\[
\mathcal{S}_i = \{ x \in \mathbb{R}^d \mid x = A_i v + b_i,\ v \in \mathbb{R}^{(d-1)/2}\}.
\]
are pairwise disjoint and any hyperplane separating $\mathcal{S}_i$
and $\mathcal{S}_j$, $i\neq j$, intersects all $\mathcal{S}_k$, $k\neq i,j$.
\end{theorem}
\begin{proof}
  Any pair of subspaces $\mathcal{S}_i$ and $\mathcal{S}_j$ intersects
  only if there exist vectors $u$, $v$ such that
\[
A_i u + b_i = A_j v + b_j,\quad\mbox{or}\quad
\left[ A_i, A_j\right]\left[\begin{array}{c}u\\-v\end{array}\right] = b_j - b_i.
\]
It follows from the assumption that $[A_i,A_j,b_j-b_i]$ is full rank,
that no such two vectors exist, and therefore that all subspaces are
pairwise disjoint.

Any hyperplane $\mathcal{H}_{i,j}$ separating $\mathcal{S}_i$ and
$\mathcal{S}_j$ is of the form $a^Tx = \beta$. To avoid intersection
with $\mathcal{S}_i$ we must have $a^T(A_i v + b) \neq \beta$ for all
$v \in\mathbb{R}^{d-1}$, which is satisfied if an only if $a^T A_i =
0$. It follows that we must also have $a^T A_j = 0$, and therefore
that $a$ is a normal vector to the $(d-1)$-subspace spanned by
$[A_i,A_j]$. From the full-spark assumption on $A$ it follows that
$a^TA_k \neq 0$ for all $k\neq i,j$, which shows that
$\mathcal{H}_{i,j}$ intersects the corresponding $\mathcal{S}_k$. The
result follows since the choice of $i$ and $j$ was arbitrary.
\end{proof}

Random matrices $A$ and vectors $b_i$ with entries i.i.d. Gaussian
satisfy the conditions in Theorem~\ref{Thm:TverbergRd} with
probability one, thereby showing the existence of the desired
configurations. A simple extension of the construction to dimension
$d+1$ is obtained when generating subspaces $\mathcal{S}_i' \subset
\mathbb{R}^{d+1}$ by matrices $A_i'$, formed by appending a row of
zeros to $A_i$ and adding a column corresponding to the last column of
the $d\times d$ identify matrix, and vectors $b_i' = [b_i; 0]$. Pach
and Tardos \cite{PAC2001Ta} further show that the lines in the
construction described by Tverberg can be replaced by appropriately
chosen unit segments. Adding a sufficiently small ball in the
Minkowski sense then results in $n$ bounded convex sets with non-empty
interior whose separation requires the maximum $n \choose 2$
hyperplanes.

\subsubsection{Point sets}

When separating a set of $n$ points, the maximum number of hyperplanes
needed is easily seen to be $n-1$; we can cut off a single extremal
point of subsequent convex hulls until only a single point is
left. This maximum can be reached, for example when all points lie on
a straight line. For a set of points in general position, it is shown
in \cite{BOL1995Ua} that the maximum number $f(n,d)$ of hyperplanes
needed satisfies
\[
\lceil (n-1)/d\rceil \leq f(n,d) \leq \lceil (n - 2^{\lceil\log
  d\rceil})/d\rceil + \lceil \log d\rceil.
\]
Based on this we can expect the number of hyperplanes needed to
separate a family of unit balls to be much smaller than the maximum
possible $n \choose 2$, whenever $n > d+1$.

\subsubsection{Non-convex sets}

The interface between two non-convex sets can be arbitrarily complex,
which means that there are no meaningful bounds on the number of
hyperplanes needed to separate general sets.

\section{Gradients}\label{Sec:Gradients}
Parameters in the neural network are learned by minimizing a loss
function over the training set, using for example
stochastic gradient descent on
the formulation shown in \eqref{Eq:NNTraining}. The gradient of such a
loss function decouples over the training samples and can be written
as
\begin{equation}\label{Eq:GradientSum}
\nabla \phi(s) =
{\textstyle\frac{1}{\vert\mathcal{T}\vert}}\sum_{(x,c)\in\mathcal{T}}\nabla f(s; x,c)
\end{equation}
where each term can be evaluated using
backpropagation~\cite{RUM1986HWa}.  The idea of the section is to
explore how points contribute when they are part of a training
set. That is, for a given parameter set $s$, and with the class
information $c = c(x)$ assumed to be known, we are interested in
$\nabla f(x)$; the behavior of $\nabla f$ as a function of $x$ over
the entire feature space. We will see that some points in the training
set contribute more to the gradient than others. So, instead of just
looking at the total gradient, we look at the contribution to the
gradient of each point: points that have a large relative
contribution to the gradient can be said to be more informative than
those that do not contribute much (the amount of contribution of each
point typically changes during optimization).  Throughout this and the next
section we use the word `gradient' loosely and also use it to refer to
blocks of gradient entries corresponding to the parameters a layer,
individual entries, or the gradient field $\nabla f(x)$ of those
quantities over the entire feature space. The exact meaning should be
clear from the context.

\subsection{Motivational example}

We illustrate the relative importance of different training samples
using a simple one-dimensional example. We define a basic two-layer
neural network in which the first layer defines a hyperplane $\alpha x
= \beta$ with nonlinearity $\nu_1 = \sigma$, and in which the second
layer applies the identity function follow by nonlinearity $\nu_2 =
\ell_5$ for amplification (for simplicity we look only at one class
and use a logistic function instead of the softmax function). Choosing
$\alpha = 1$ and $\beta = 0$ defines the region shown in
Figure~\ref{Fig:ExampleGradient1D}(a). Now, suppose that all points $x
\in [-12,12]$ belong to the same class and should therefore be part of
this region. Intuitively, it can be seen that slight changes in the
location of the hyperplane or in the steepness of the transition will
have very little effect on the output of the neural network for input
points $\abs{x} \geq 5$, say, since values close to one or zero remain
so after the perturbation. As such, we expect that in these regions
the gradient with respect to $\alpha$ and $\beta$ will be small. For
points in the transition region the change will be relatively large,
and the gradient at those points will therefore be larger. This
suggests that training points away from the transition region provide
little information when deciding in which direction to move the
hyperplane and how sharp the transition should be; this information
predominantly comes from the training points in the transition region.

\begin{figure}
\centering
\begin{tabular}{ccc}
\includegraphics[height=125pt]{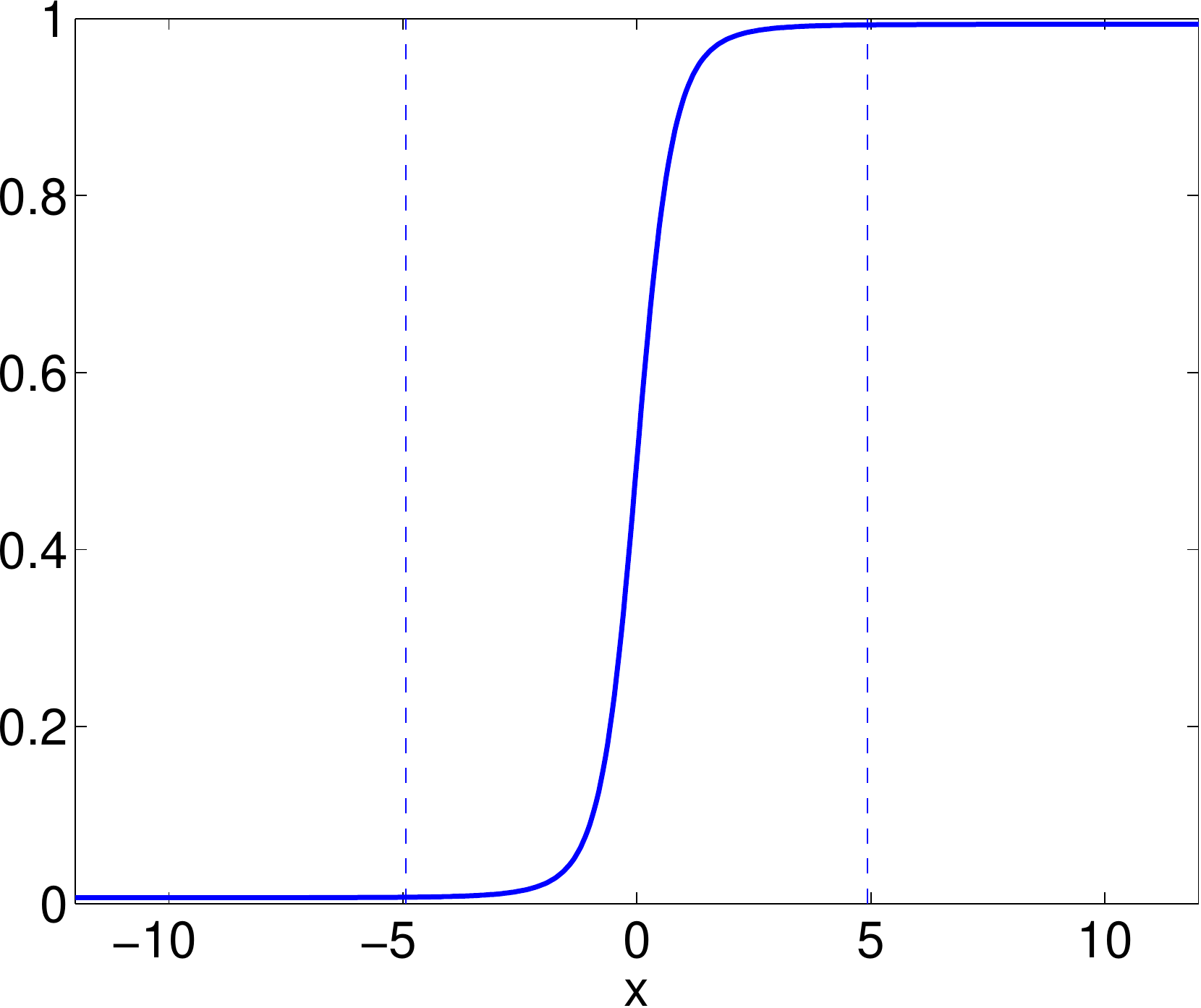}&
\includegraphics[height=125pt]{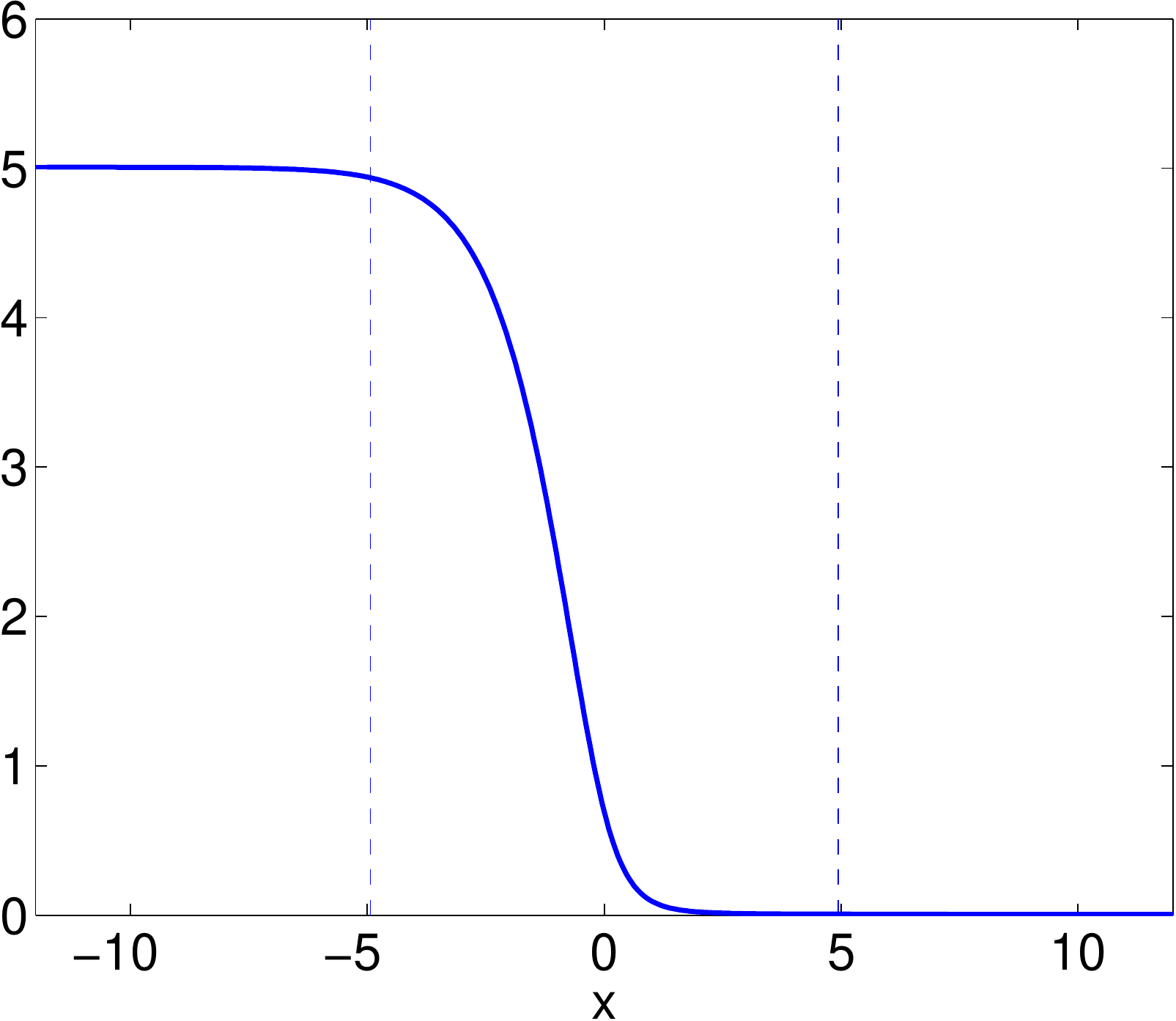}&
\includegraphics[height=125pt]{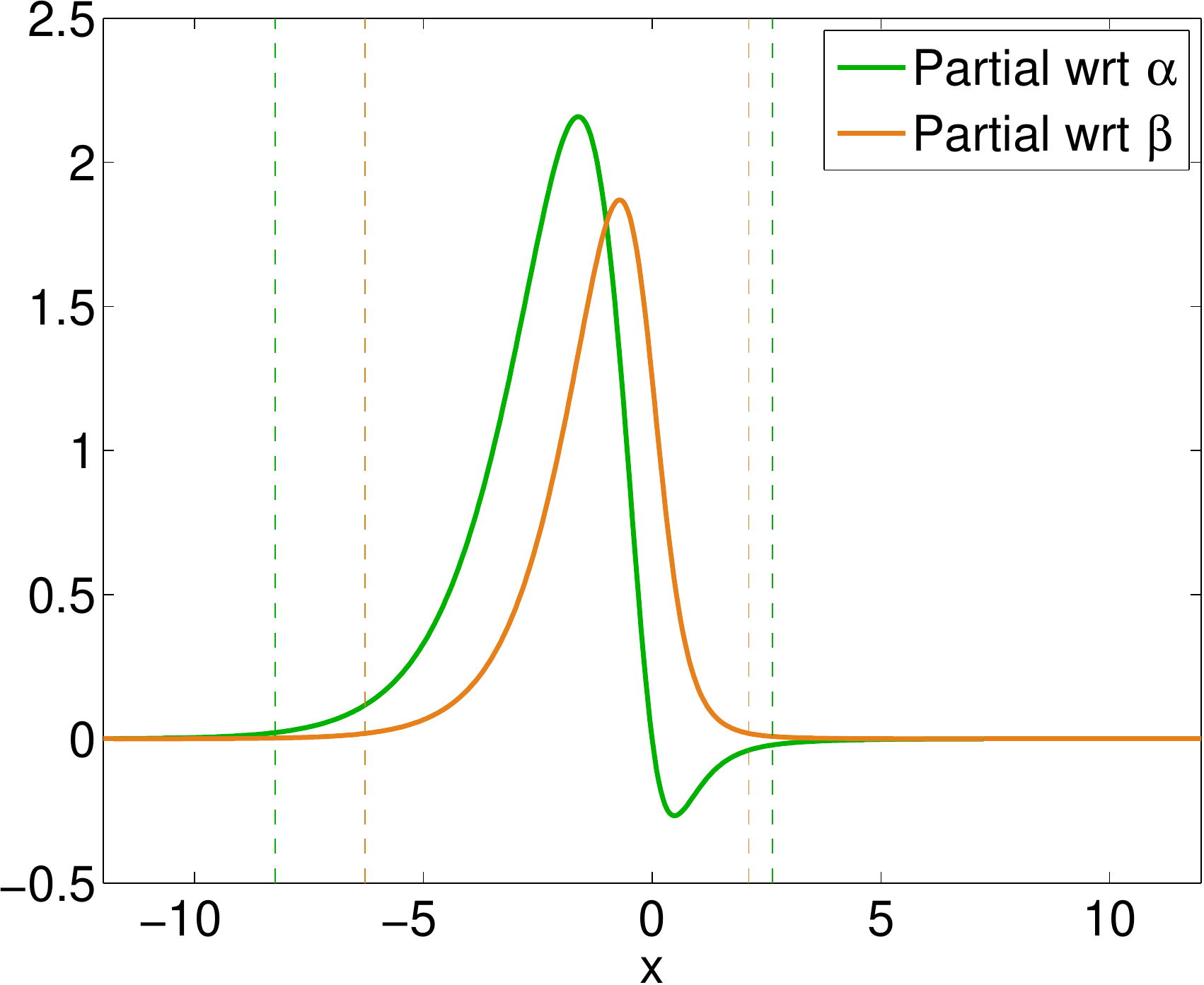}\\
({\bf{a}}) & ({\bf{b}}) & ({\bf{c}})
\end{tabular}
\caption{Classification of points on the $x$-axis, with (a) the
  decision region: $\hat{p}(x) = \ell_5(\sigma(\alpha x - \beta))$
  with $\alpha = 1$, $\beta = 0$; (b) the cross entropy:
  $f(x) = -\log \hat{p}(x)$; (c) partial derivatives of the loss
  function with respect to $\alpha$ and $\beta$}\label{Fig:ExampleGradient1D}
\end{figure}

More formally, consider the minimization of the negative log
likelihood loss function for this network, given by
\[
f(x) = -\log \hat{p}(x)\qquad\mbox{with}\qquad\hat{p}(x) =
\ell_5(\sigma(\alpha x - \beta)).
\]
For the gradient, we need the derivative of the sigmoid function,
$\sigma_\gamma'(x) = 2\ell_\gamma'(x)$ with
\[
\ell_\gamma'(x) = \frac{\gamma e^{-\gamma x}}{(1 + e^{-\gamma x})^2} =
\gamma\left( \frac{1+e^{-\gamma x}}{(1 + e^{-\gamma x})^2} -
  \frac{1}{(1 + e^{-\gamma x})^2}\right) =
\gamma[\ell_{\gamma}(x) - \ell_{\gamma}^2(x)],
\qquad 
\]
and the derivative of the negative log of the logistic function:
\[
\frac{d}{dx}\left[-\log(\ell_\gamma(x))\right]
= -\frac{\ell_{\gamma}'(x)}{\ell_\gamma(f(x))}
=  -\gamma \frac{\ell_{\gamma}(x) -
   \ell_{\gamma}^2(x)}{\ell_{\gamma}(x)} =
\gamma [\ell_{\gamma}(x) - 1].
\]
Combining the above we have
\begin{eqnarray*}
{\partial f}/{\partial \alpha} &=& \phantom{-}\gamma\left[\ell_{\gamma}(\sigma(\alpha x - \beta)) -
  1\right]\cdot \sigma'(\alpha x - \beta)\cdot x\\
{\partial f}/{\partial \beta} &=& -\gamma\left[\ell_{\gamma}(\sigma(\alpha x - \beta)) -
  1\right]\cdot \sigma'(\alpha x - \beta),
\end{eqnarray*}
with $\gamma = 5$. The loss function and partial derivatives with
respect to $\alpha$ and $\beta$ are plotted in
Figure~\ref{Fig:ExampleGradient1D}(b) and (c). The vertical lines in
plot (c) indicate where the gradients fall below one percent of their
asymptotic value. As expected, points beyond these lines do indeed
contribute very little to the gradient, regardless of whether they are
on the right or the wrong side of the hyperplane.

\subsection{General mechanism}

For the contribution of each sample to the gradient in general
settings we need to take a detailed look at the backpropagation
process. This is best illustrated using a concrete three-layer neural
network:
\begin{equation}\label{Eq:NNExample2D}
\begin{array}{lll}
A_1 = \left[\begin{array}{rr}1.0 & 0.3 \\ 0.4 &
    -1.0\end{array}\right]&
b_1 = \left[\begin{array}{r}-1.0 \\ 0.5 \end{array}\right]\ \ &
\nu_1 = \sigma_3,\\[17pt]
A_2 =
\left[\begin{array}{rr}-1&-1\\1&-1\\1&1\\-1&1\end{array}\right]&
b_2 = \left[\begin{array}{r} 1\\1\\ 1\\ 1\end{array}\right]&
\nu_2 = \sigma_3,\\[30pt]
A_3 = \left[\begin{array}{cccc}1&0&1&0\\0&1&0&1\end{array}\right]\ \ &
b_3 = \left[\begin{array}{c}-1.1\\ -1.1\end{array}\right]&
\nu_3 = \mu.
\end{array}
\end{equation}
Denoting by $f(x)$ the negative log likelihood of $\mu(x)$, the
forward and backward passes through the network can be written as
\begin{equation}\label{Eq:GradientComputation}
\begin{array}{ll}
v_1 = A_1x_0 - b_1 \qquad\qquad & \\[9pt]
x_1 = \sigma_3(v_1)
& \displaystyle y_1 = \frac{\partial x_1}{\partial v_1}\cdot z_2 =
\sigma_{3}'(v_1)\cdot z_2\\[9pt]
v_2 = A_2x_1 - b_2
& \displaystyle z_2 = \frac{\partial v_2}{\partial x_1}\cdot y_2 = A_2^T y_2 \\[9pt]
x_2 = \sigma_3(v_2)
& \displaystyle y_2 = \frac{\partial x_2}{\partial v_2}\cdot z_3 =
\sigma_{3}'(v_2)\cdot z_3\\[9pt]
v_3 = A_3x_2 - b_3
& \displaystyle z_3 = \frac{\partial v_3}{\partial x_2}\cdot y_3 = A_3^T y_3\\[9pt]
x_3 = f(v_3)
& \displaystyle y_3 = \frac{\partial x_3}{\partial v_3} =
\frac{\partial f}{\partial v_3} = \nabla f(v_3),
\end{array}
\end{equation}
where the left and right columns respectively denote the stages in the
forward and backward pass. The regions formed during the forward pass
are shown in Figure~\ref{Fig:NNExample2D}. With this, the partial
differentials with respect to weight matrices and bias vectors are of
the following form:
\begin{equation}\label{Eq:PartialDiffsAb}
\frac{\partial f}{\partial [A_3]_{i,j}} = \frac{\partial v_3}{\partial
  [A_3]_{i,j}}\cdot y_3 = [x_2]_{j}\cdot [y_3]_{i},\qquad\mbox{and}\qquad
\frac{\partial f}{\partial b_3} = \frac{\partial v_3}{\partial
  b_3}\cdot \frac{\partial f}{\partial
  v_3} = \frac{\partial v_3}{\partial
  b_3}\cdot y_3 = -y_3.
\end{equation}

\begin{figure}[t]
\centering
\begin{tabular}{ccc}
\includegraphics[width=0.265\textwidth]{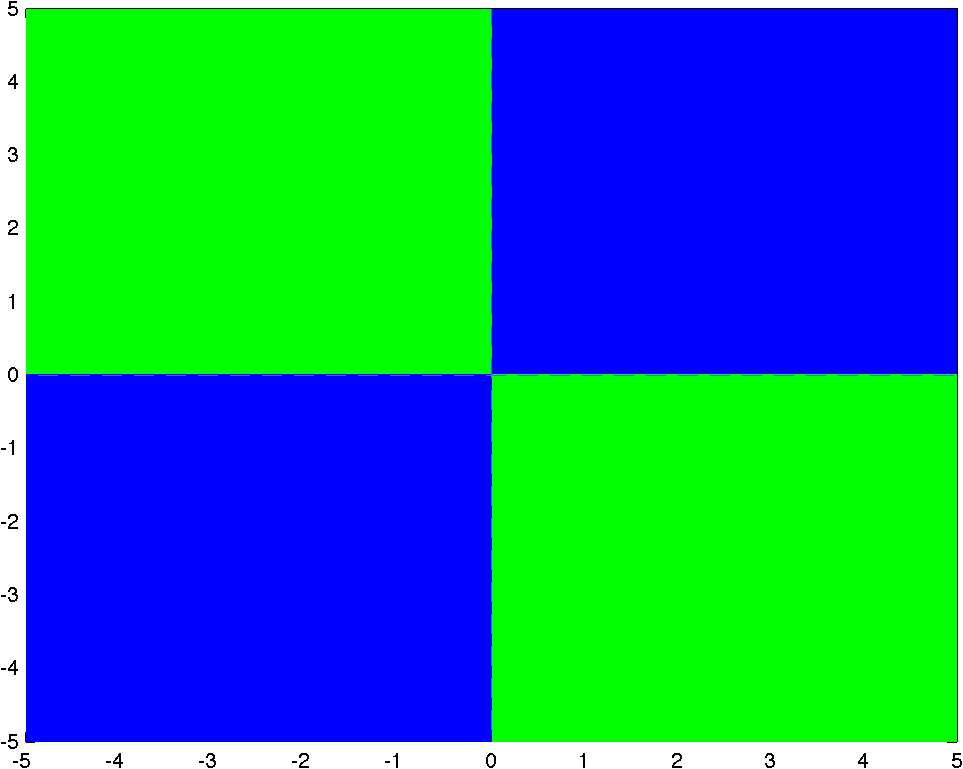}%
\begin{picture}(0,0)(0,0)%
\put(-101,76){$\mathcal{C}_1$}
\put(-37,76){{\color{white}$\mathcal{C}_2$}}
\put(-101,26){{\color{white}$\mathcal{C}_2$}}
\put(-37,26){$\mathcal{C}_1$}
\end{picture}&
\includegraphics[width=0.265\textwidth]{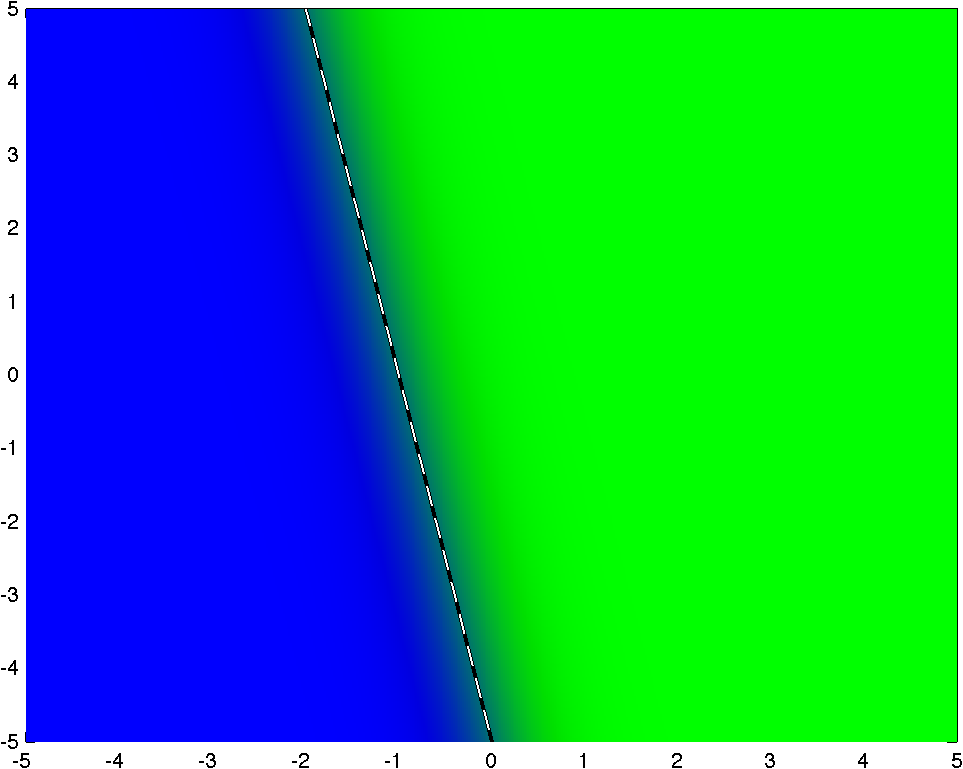}&
\includegraphics[width=0.265\textwidth]{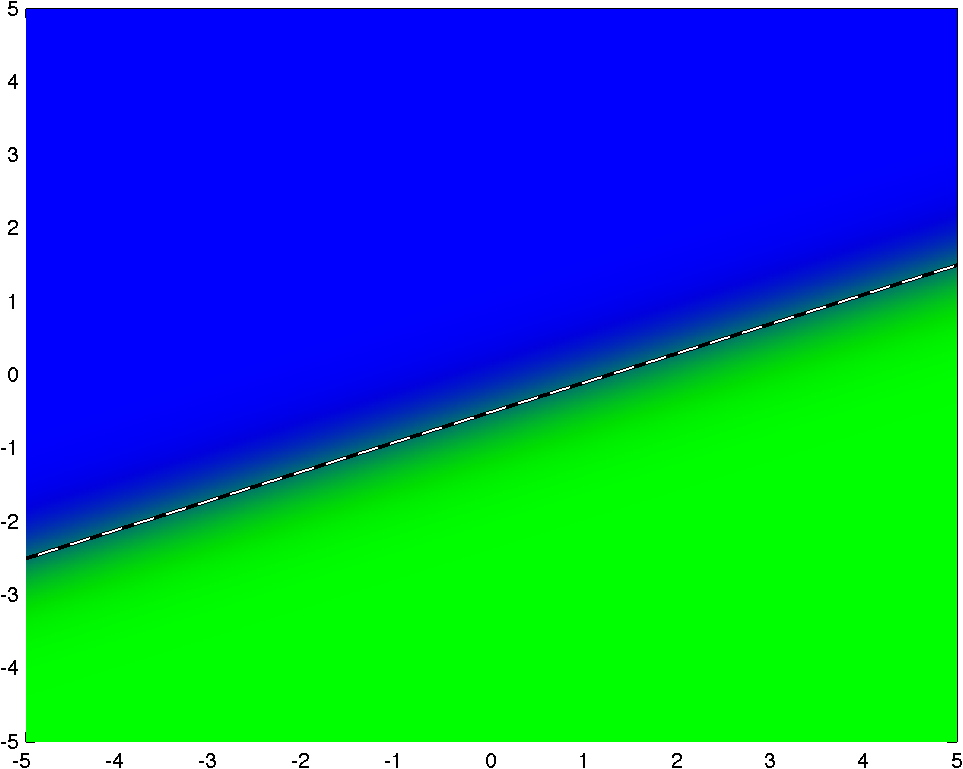}\\
({\bf{a}}) & ({\bf{b}}) $[x_1]_1$ & ({\bf{c}}) $[x_1]_2$\\[3pt]
\includegraphics[width=0.265\textwidth]{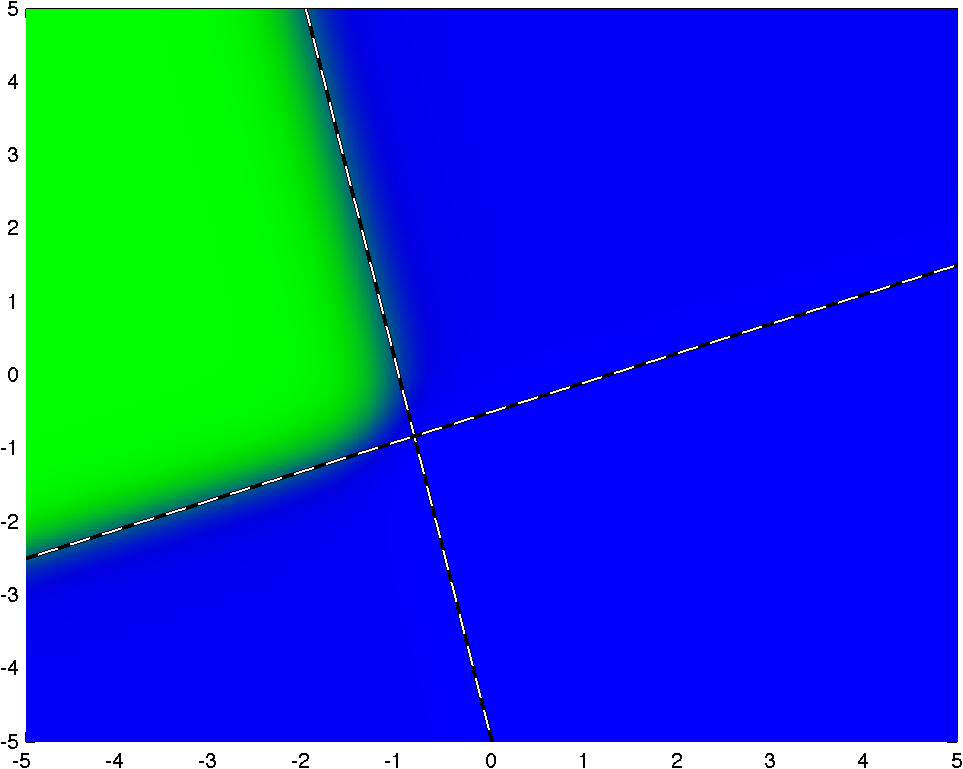}&
\includegraphics[width=0.265\textwidth]{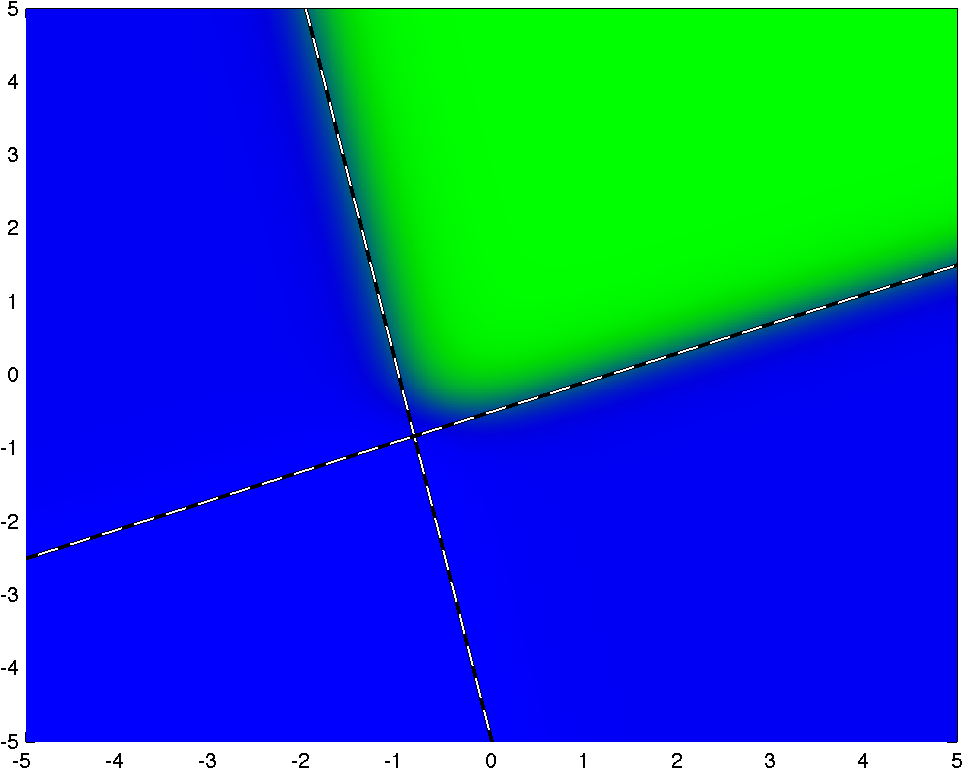}&
\includegraphics[width=0.265\textwidth]{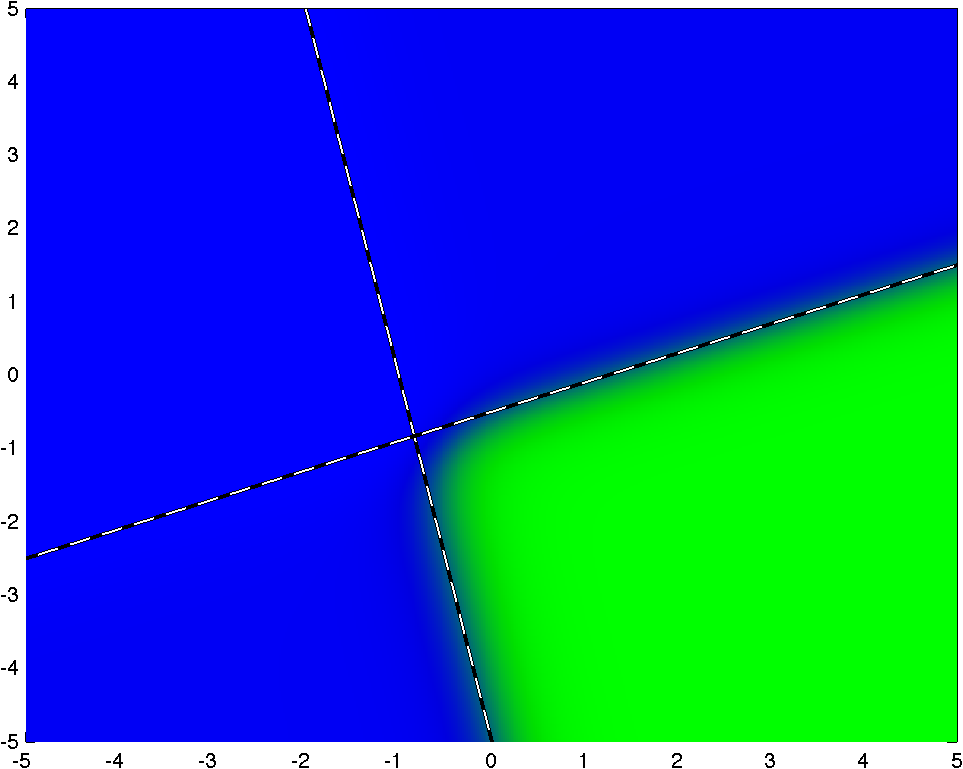}\\
({\bf{d}}) $[x_2]_1$ & ({\bf{e}}) $[x_2]_2$ & ({\bf{f}}) $[x_2]_3$\\[3pt]
\includegraphics[width=0.265\textwidth]{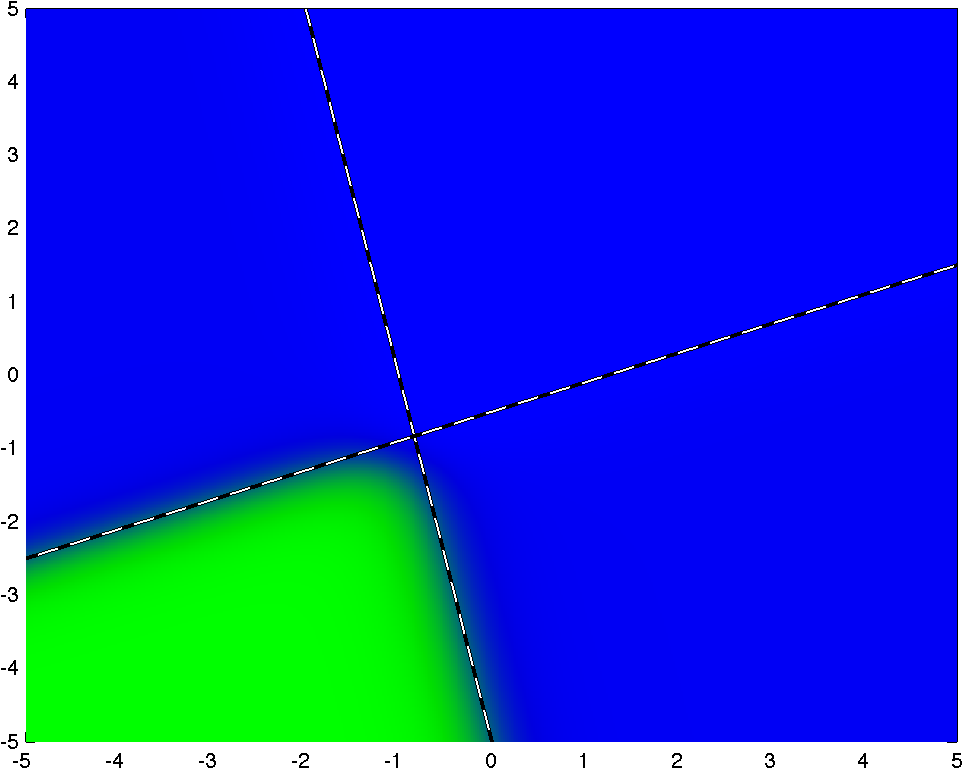}&
\includegraphics[width=0.265\textwidth]{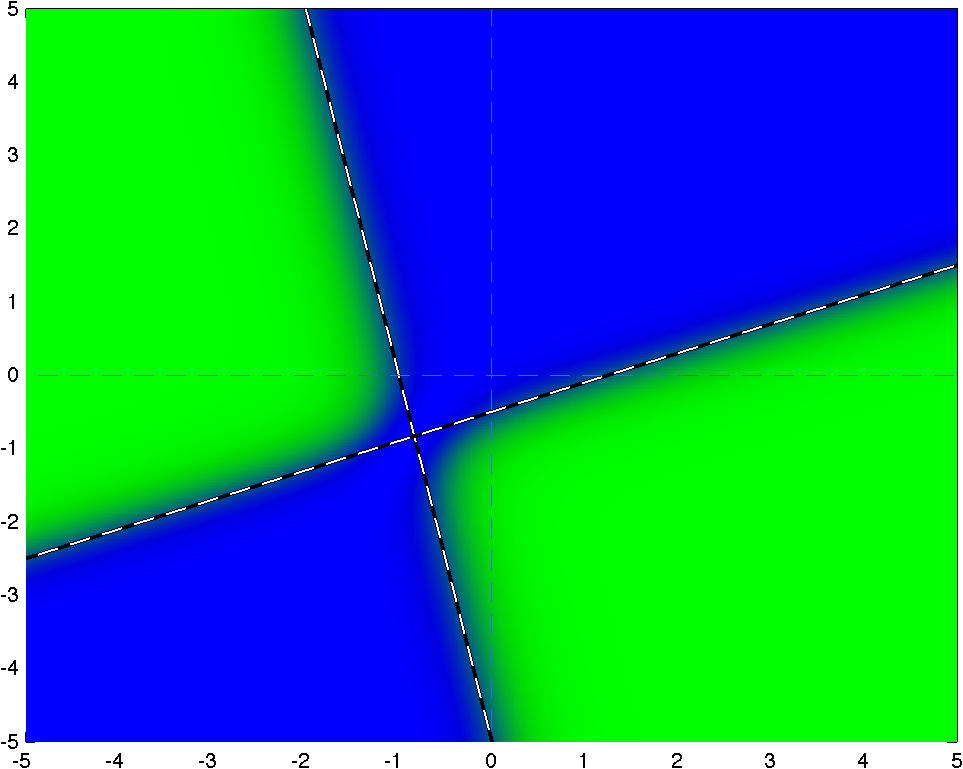}&
\includegraphics[width=0.265\textwidth]{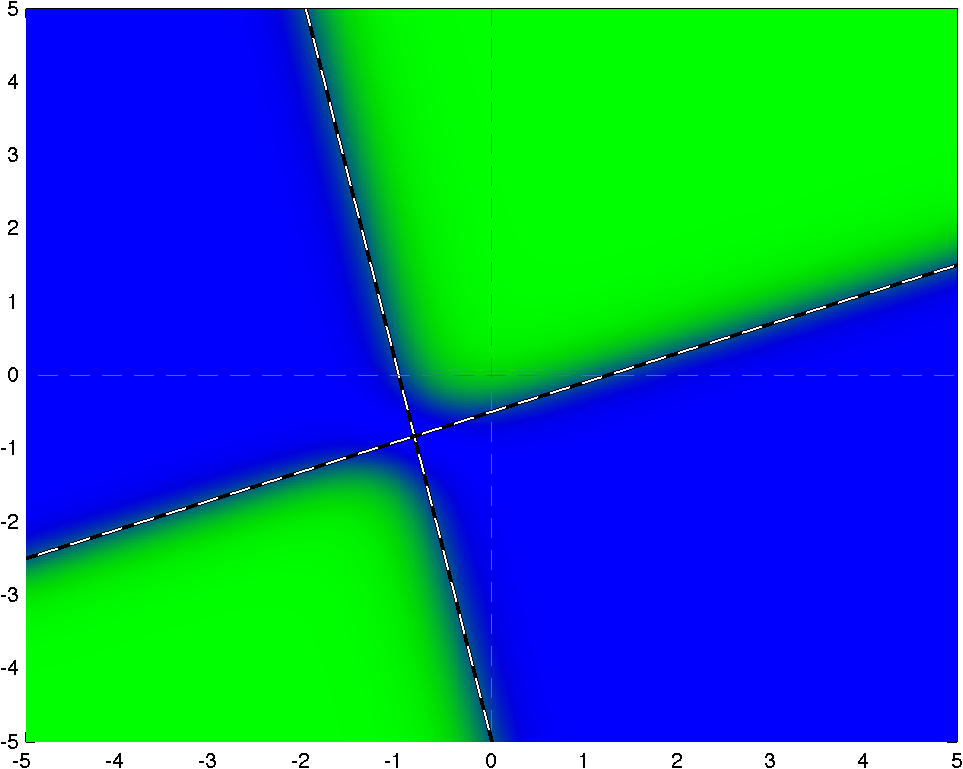}\\
({\bf{g}})  $[x_2]_4$ & ({\bf{h}}) $[\mu(v_3)]_1$ & ({\bf{i}}) $[\mu(v_3)]_2$
\end{tabular}
\caption{Regions corresponding to the neural network with weights
  defined by \eqref{Eq:NNExample2D}, with (a) the two ground-truth
  classes; (b,c) output regions of the first layer; (d--g) output
  regions of the second layer; and (h,i) output regions of the third
  layer (showing only the intermediate output $\mu(v_3)$ instead of
  the loss function output) with desired class boundaries
  superimposed.}\label{Fig:NNExample2D}
\end{figure}
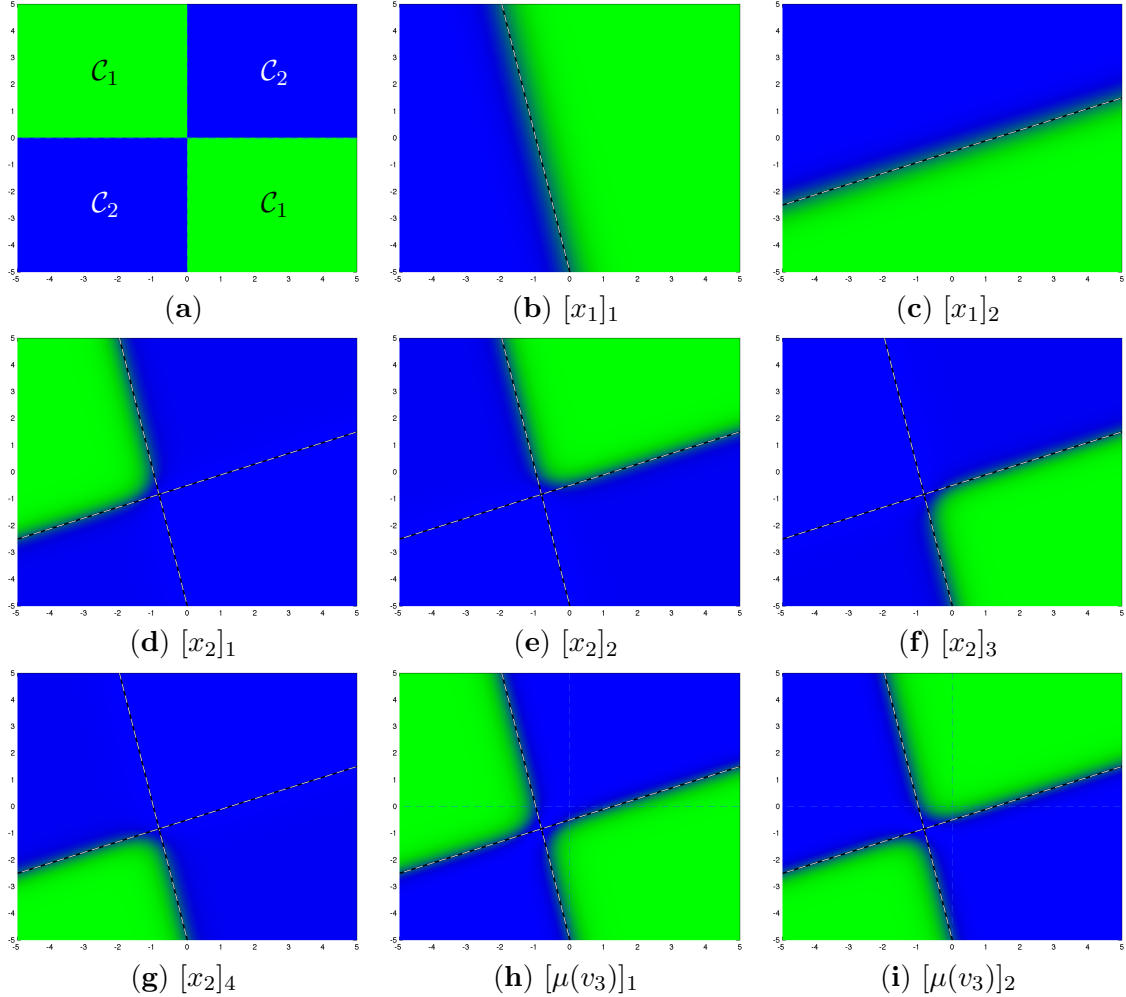

We now analyze each of the backpropagation steps to explain the
relationship between the regions of high and low confidence at each of
the neural network layers and the gradient values or importance of
different points in the feature space. In all plots we only show the
absolute values of the quantities of interest because we are mostly
interested in the relative magnitudes over the feature space rather
than their signs. After a forward pass through the network we can
evaluate the loss function and its gradients, shown in
Figure~\ref{Fig:Backprop}(a).  In this particular example we have
$[y_3]_1 = -[y_3]_2$, so we only show the former.  Given $y_3$ we can
use~\eqref{Eq:PartialDiffsAb} to compute the partial differentials of
$f$ with respect to the entries in $A_3$ and $b_3$.  The partial
differential with respect to $b_3$ simply coincides with $-y_3$, and
is therefore not very interesting. On the other hand, we see that the
partial differential with respect to $[A_3]_{i,j}$ is formed by
multiplying $[y_3]_i$ with the output value $[x_2]_i$. When looking at
the feature space representation for the specific case of
$[A_3]_{1,2}$ and using absolute values, this corresponds to the
pointwise multiplication of the values in Figure~\ref{Fig:Backprop}(a)
with the mask shown Figure~\ref{Fig:Backprop}(b). This multiplication
causes the partial differential to be reduced in areas of low
confidence in $[x_2]_2$. In addition, it causes the partial
differential to vanish at points at the zero crossing of the boundary
regions, as illustrated by the white curve in the upper-right corner
of Figure~\ref{Fig:Backprop}(c).

\begin{figure}
\centering
\begin{tabular}{ccc}
\includegraphics[width=0.32\textwidth]{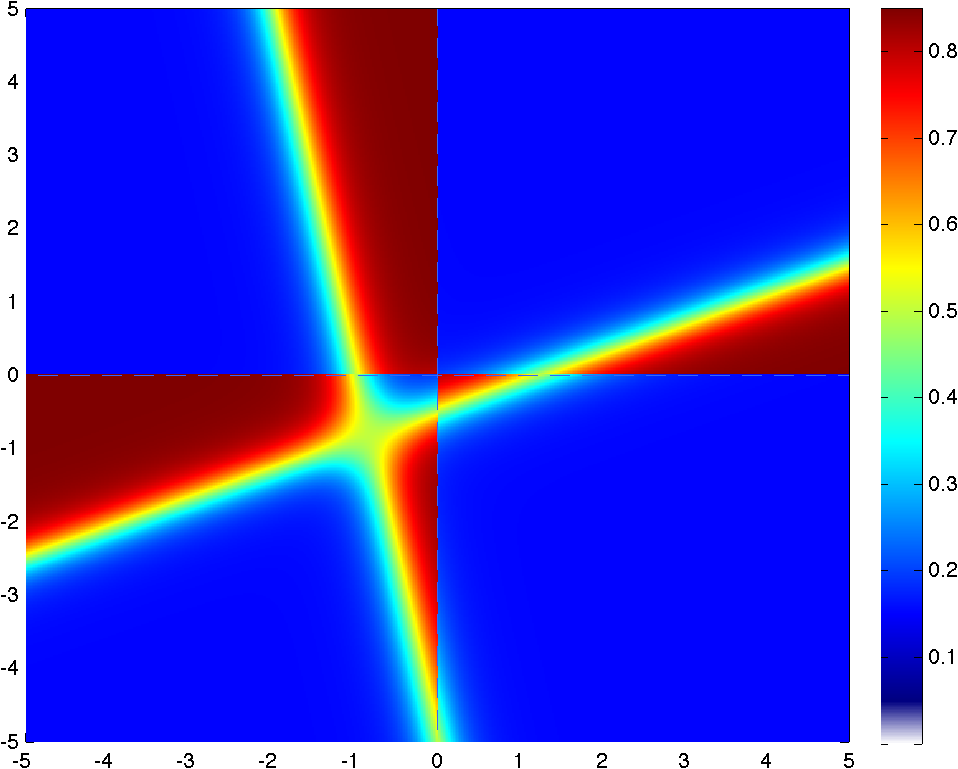} &
\includegraphics[width=0.32\textwidth]{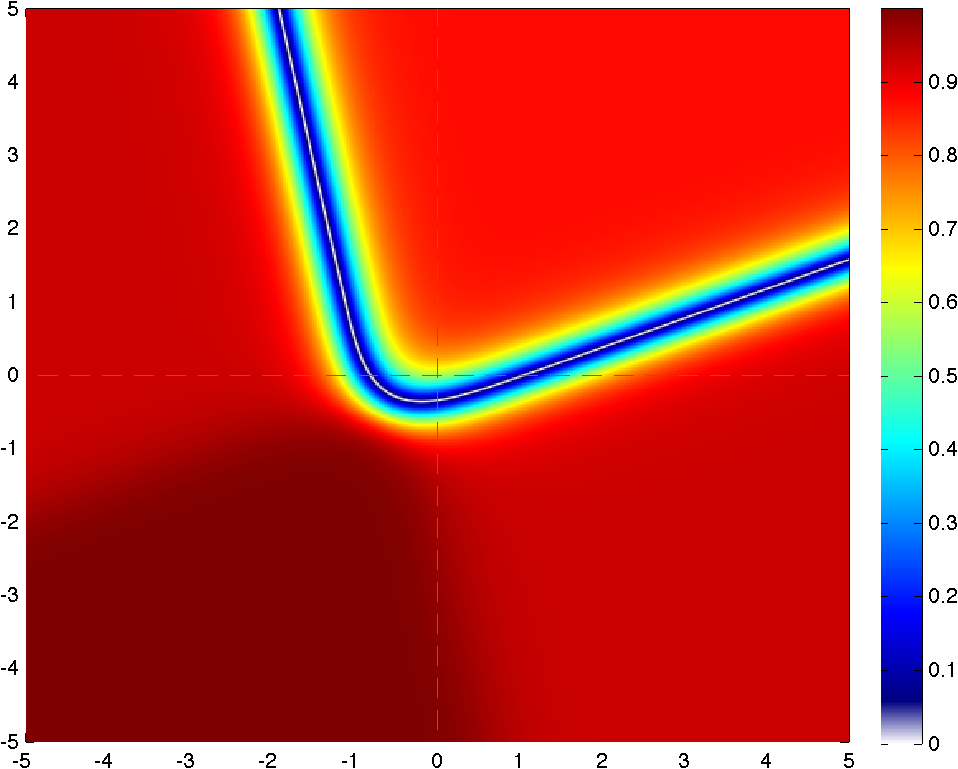} &
\includegraphics[width=0.32\textwidth]{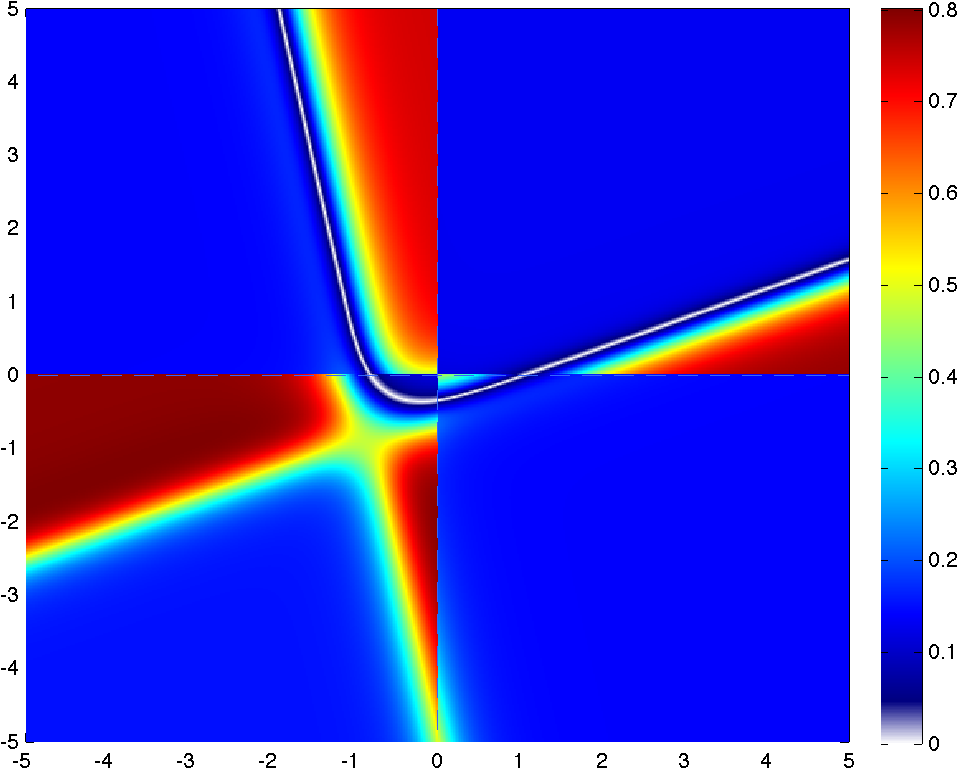}\\
 ({\bf{a}}) $\abs{[y_3]_1}$ & ({\bf{b}})
$\abs{[x_2]_2}$ & ({\bf{c}}) $\abs{\partial f/\partial [A_3]_{1,2}}$ \\[6pt]
\includegraphics[width=0.32\textwidth]{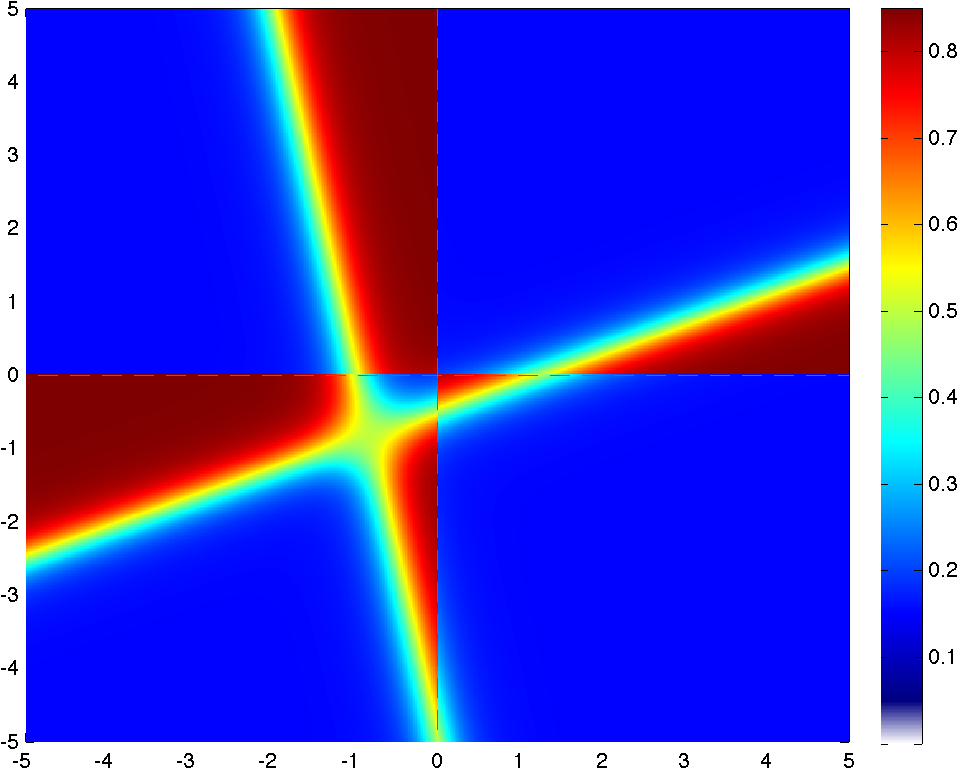} &
\includegraphics[width=0.32\textwidth]{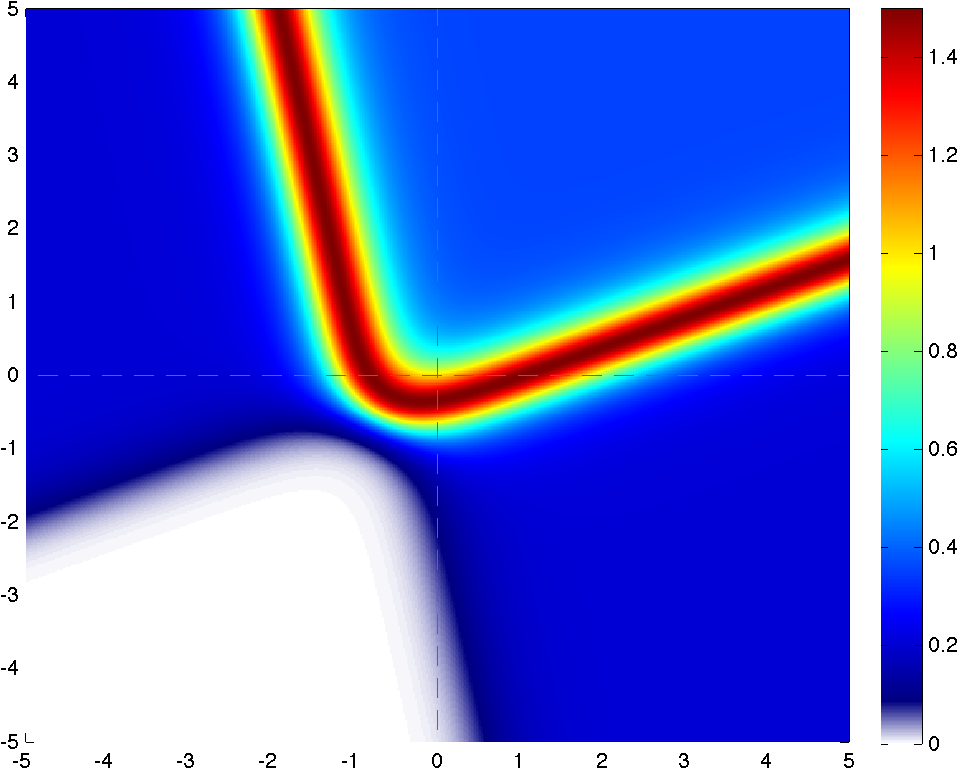} &
\includegraphics[width=0.32\textwidth]{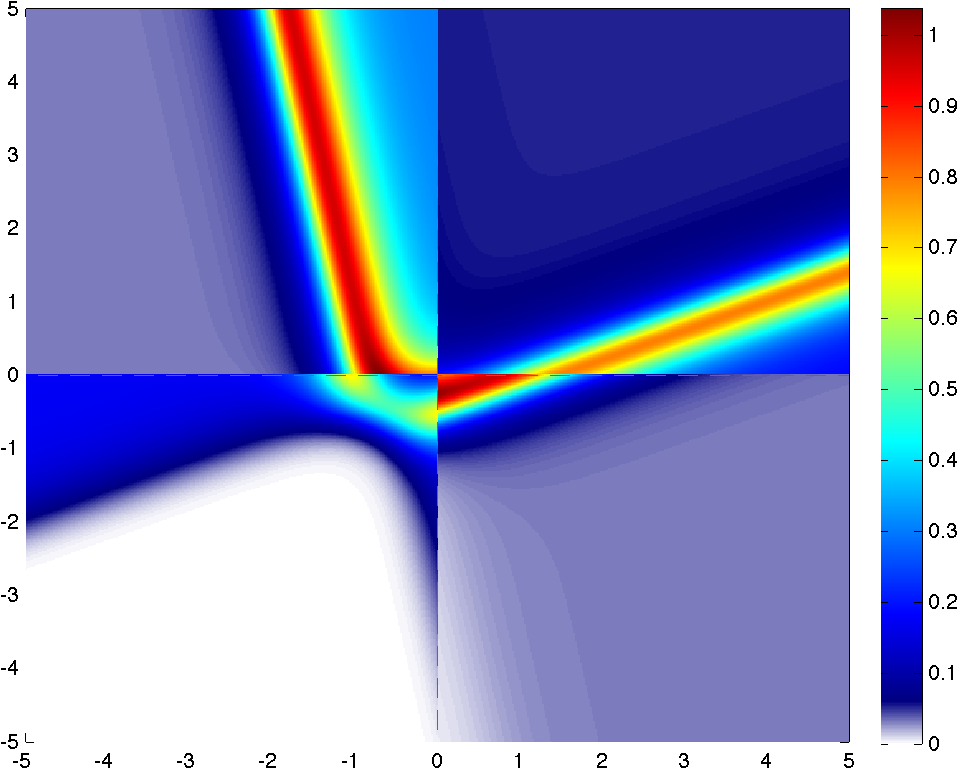} \\
({\bf{d}}) $\abs{[z_3]_2}$ & ({\bf{e}}) Mask: $[\sigma_3'(v_2)]_2$ & ({\bf{f}})
$\abs{[y_2]_2}$ \\[6pt]
\includegraphics[width=0.32\textwidth]{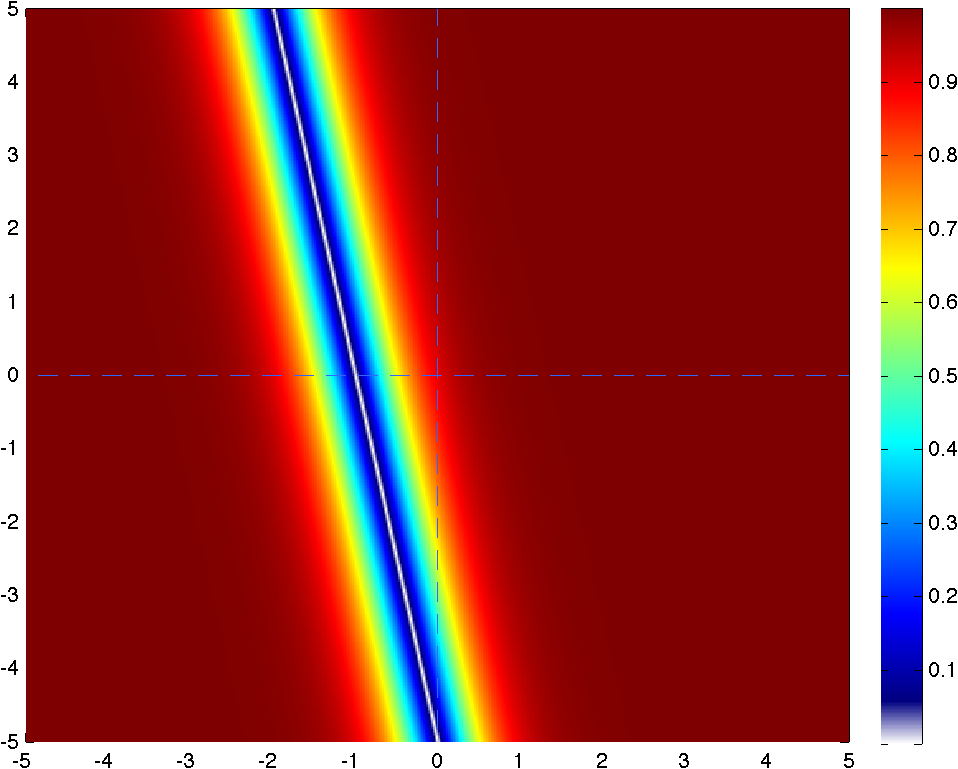} &
\includegraphics[width=0.32\textwidth]{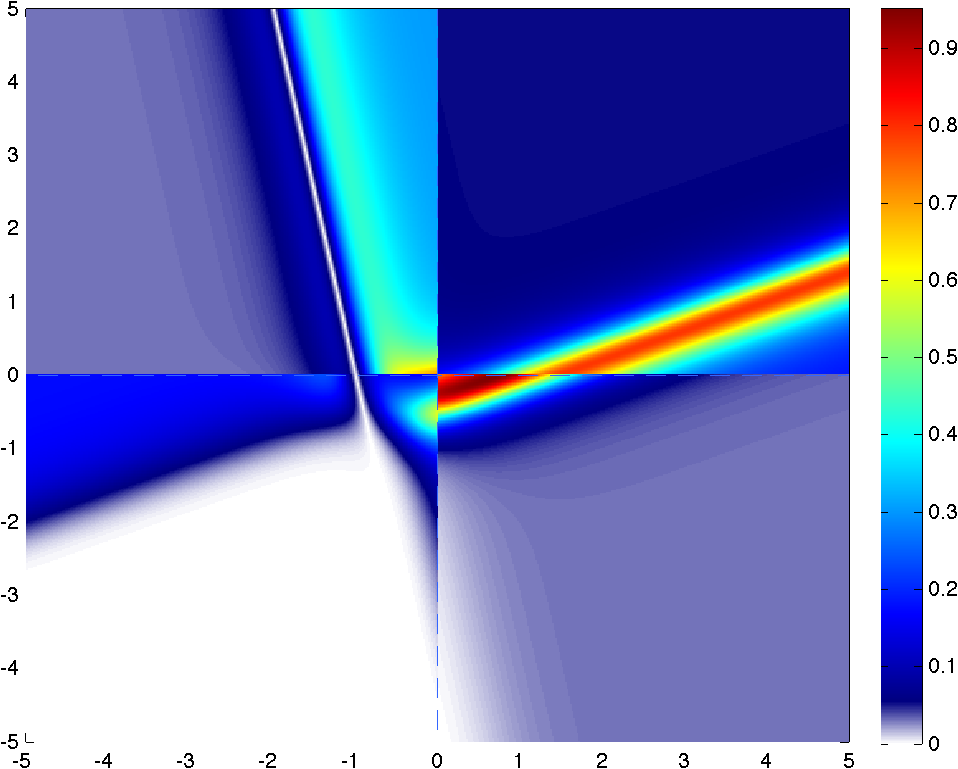} &
\includegraphics[width=0.32\textwidth]{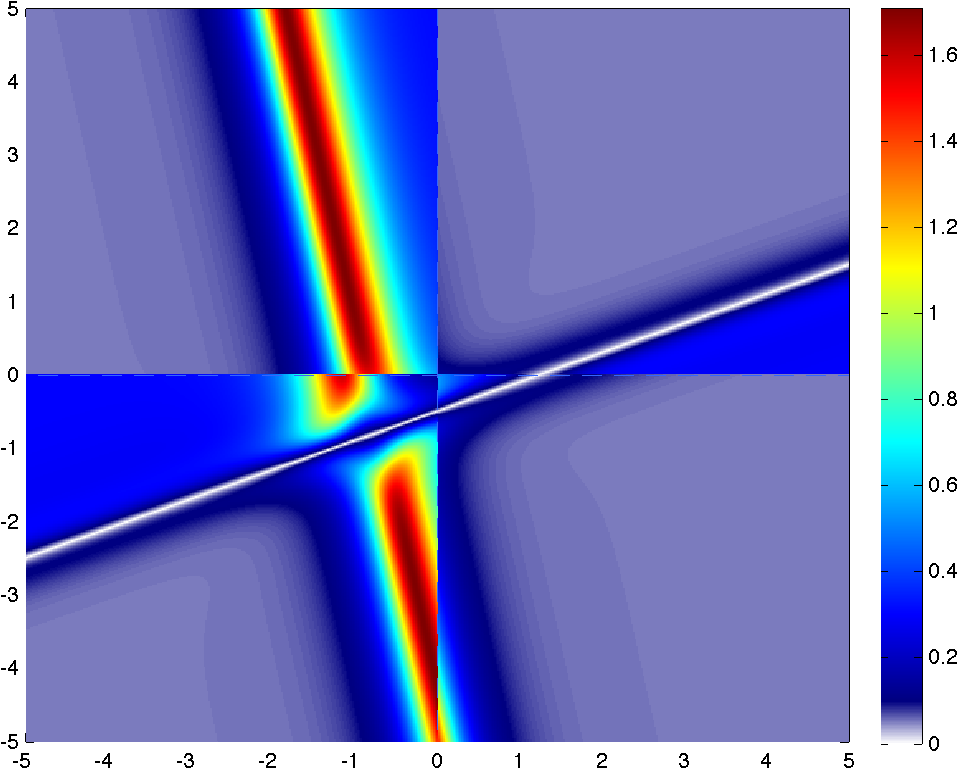} \\
 ({\bf{g}}) $\abs{[x_1]_1}$ & ({\bf{h}})
$\abs{\partial f/\partial [A_2]_{2,1}}$ & ({\bf{i}}) $\abs{[z_2]_1}$\\[6pt]
\includegraphics[width=0.32\textwidth]{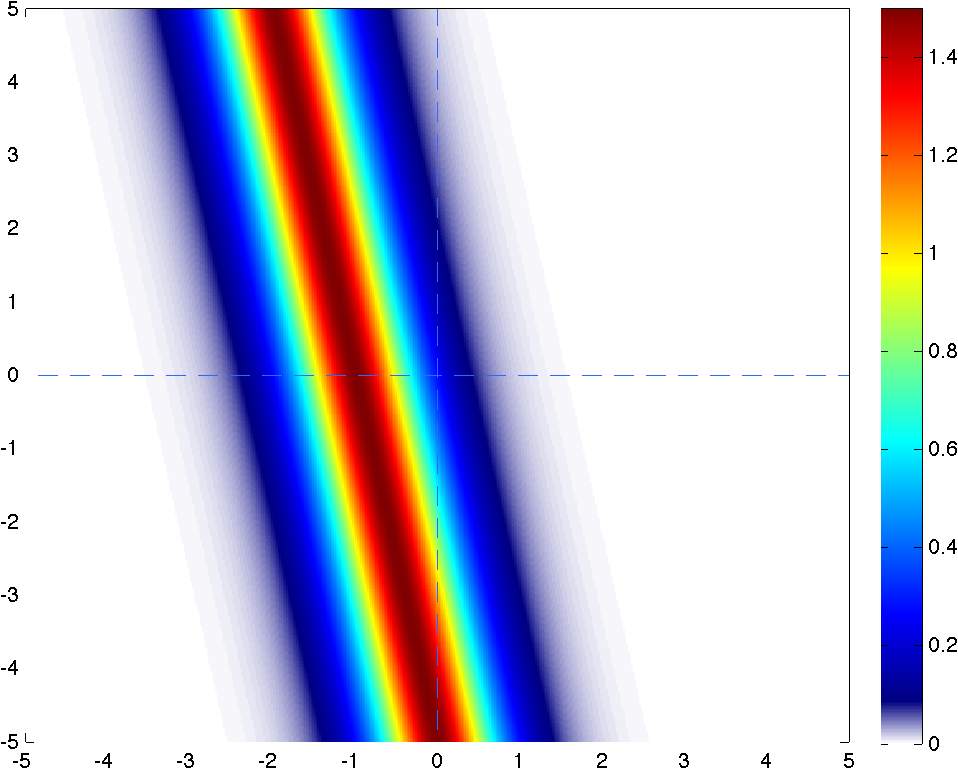} &
\includegraphics[width=0.32\textwidth]{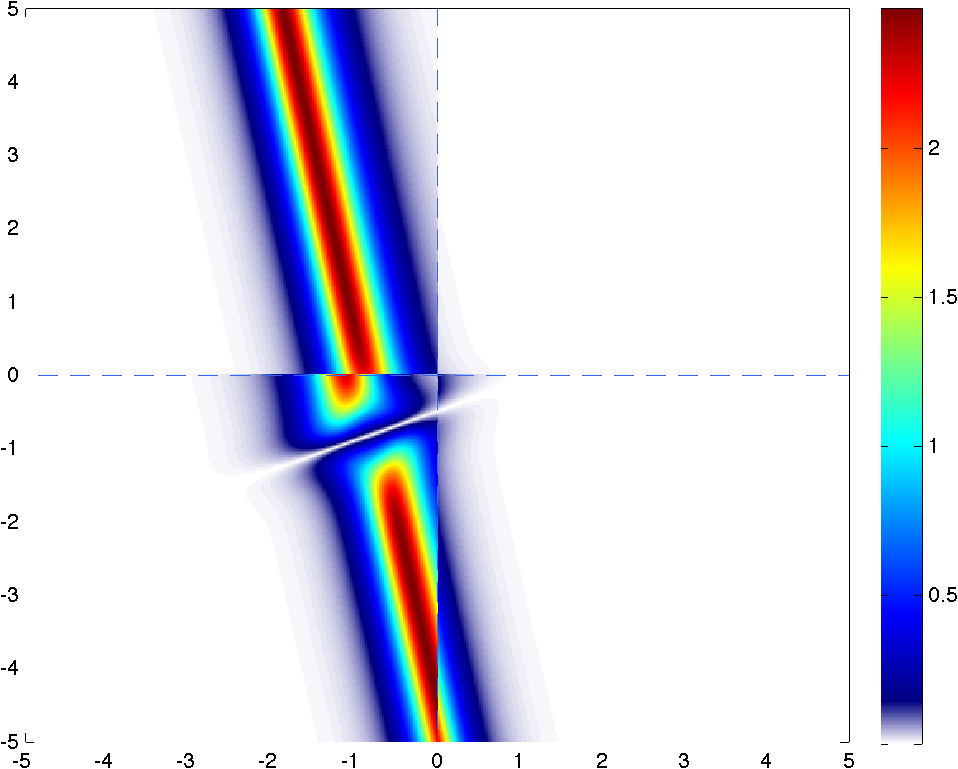}&
\includegraphics[width=0.32\textwidth]{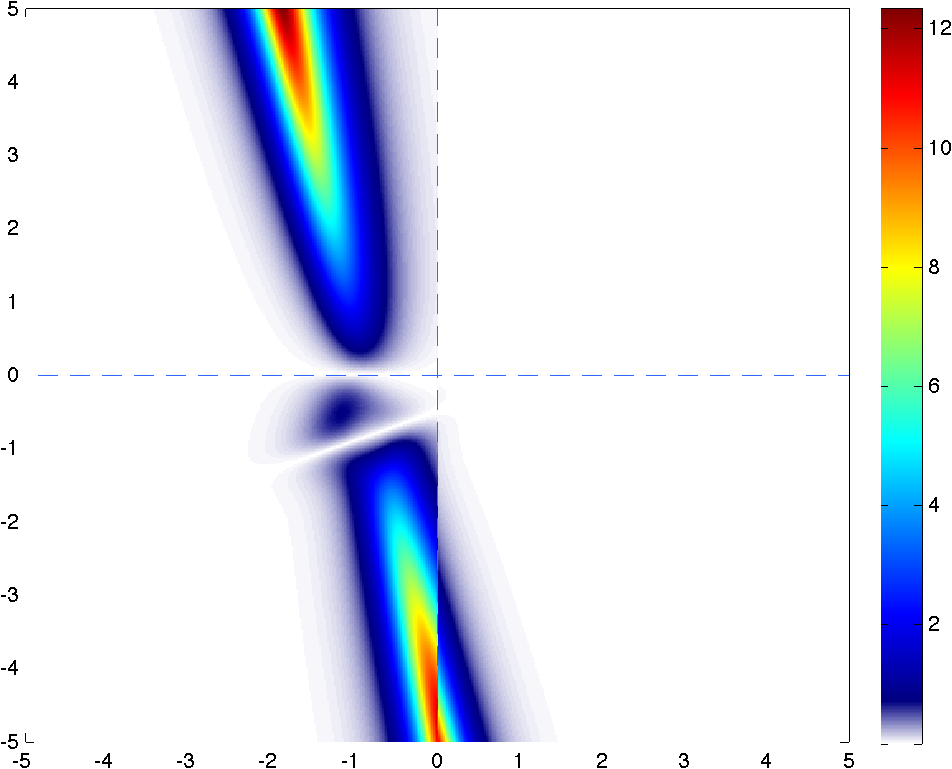} \\
 ({\bf{j}}) Mask: $[\sigma_3'(v_1)]_1$ &
({\bf{k}}) $\abs{[y_1]_1}$ &
({\bf{l}}) $\abs{\partial f/\partial [A_1]_{1,2}}$
\end{tabular}
\caption{Illustration of the error backpropagation
  process.}\label{Fig:Backprop}
\end{figure}

\clearpage

\begin{figure}[ht]
\centering
\begin{tabular}{ccc}
$\max_{i,j} \vert{\partial [A_3]_{i,j}}\vert$ &
$\max_{i,j} \vert{\partial [A_2]_{i,j}}\vert$ &
$\max_{i,j} \vert{\partial [A_1]_{i,j}}\vert$ \\
\includegraphics[width=0.282\textwidth]{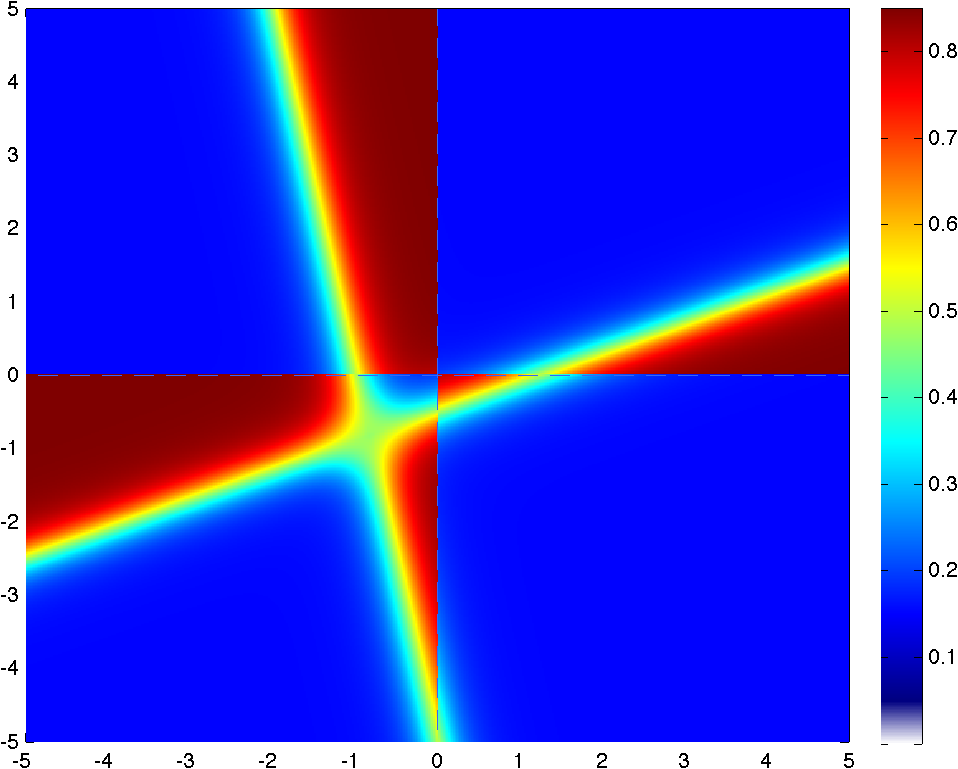}&
\includegraphics[width=0.282\textwidth]{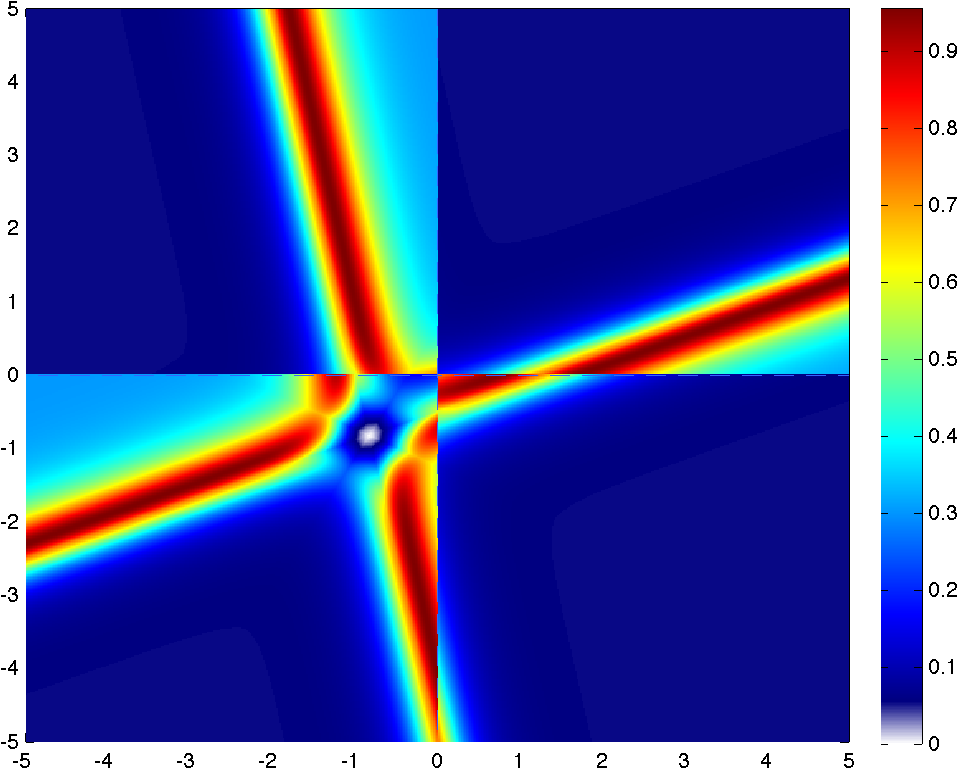}&
\includegraphics[width=0.282\textwidth]{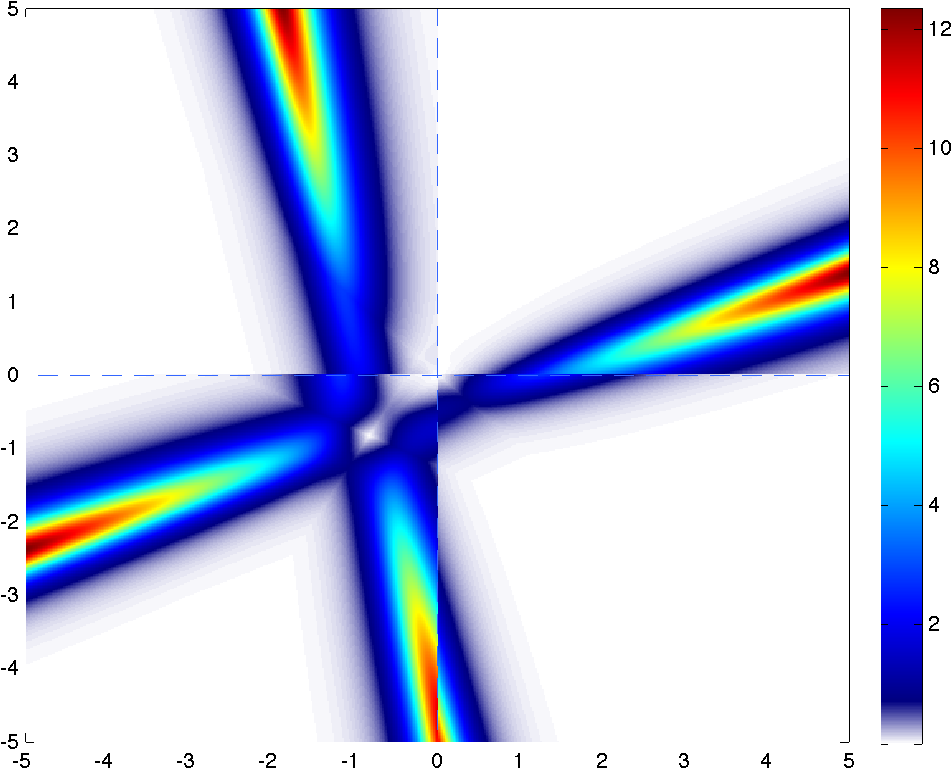}\\
\multicolumn{3}{c}{({\bf{a}}) $\gamma = 1$} \\[6pt]
\includegraphics[width=0.282\textwidth]{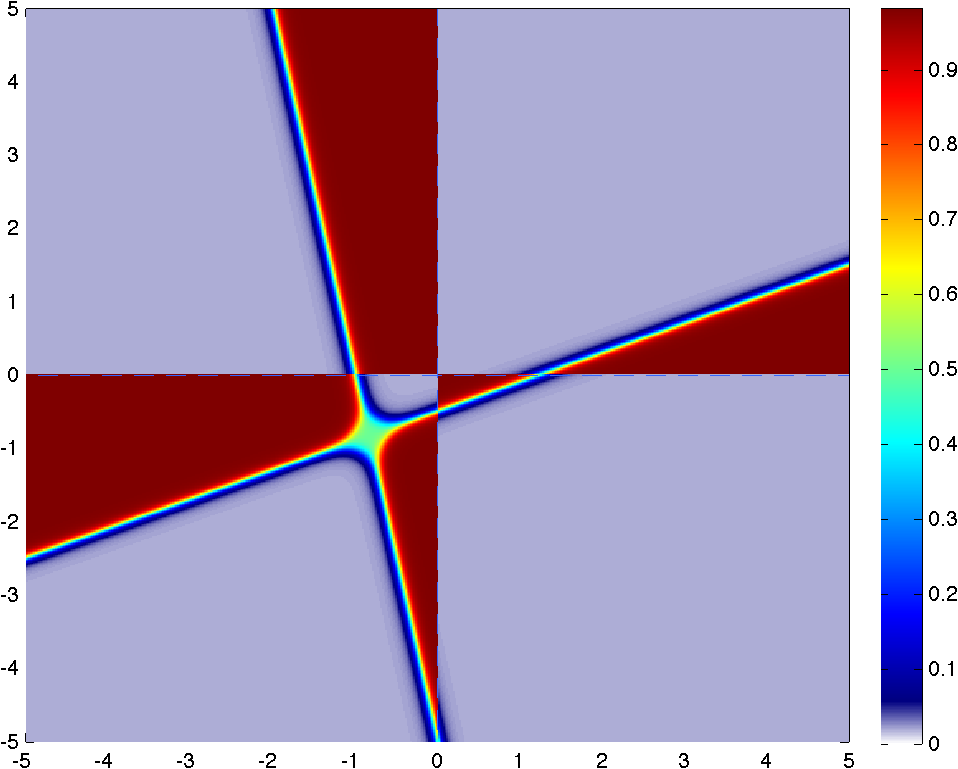}&
\includegraphics[width=0.282\textwidth]{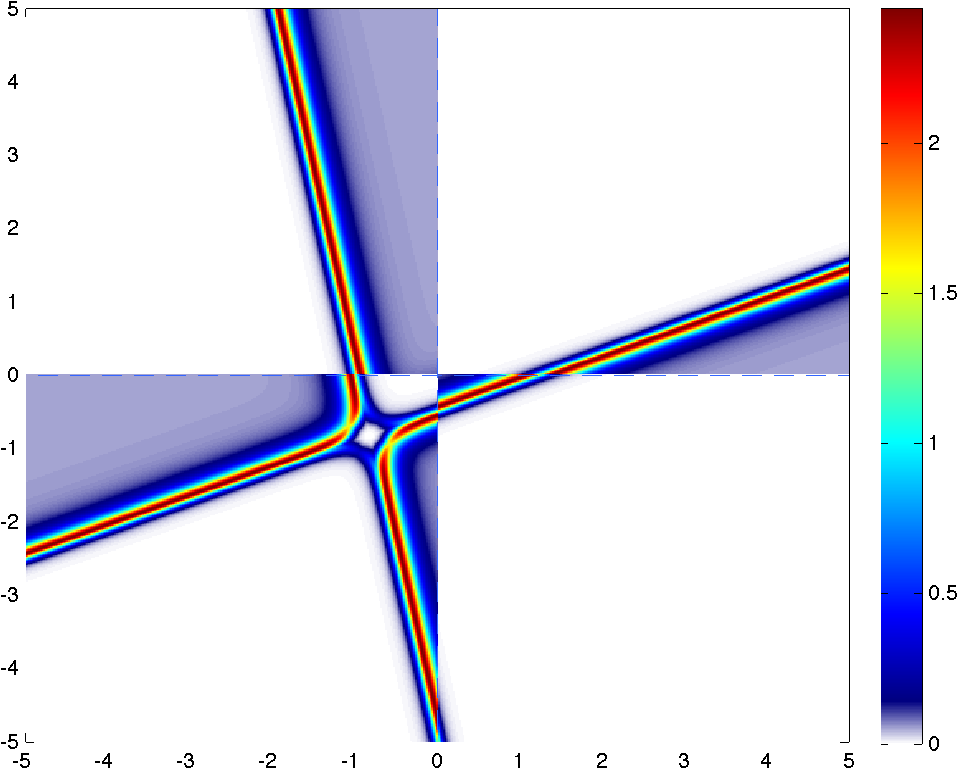}&
\includegraphics[width=0.282\textwidth]{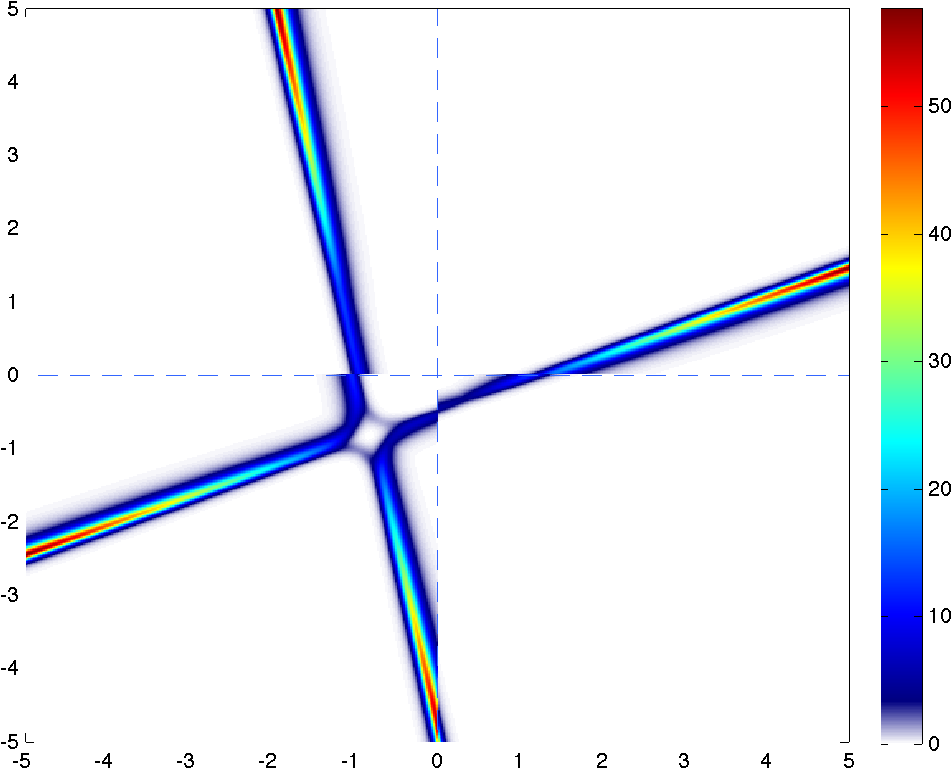}\\
\multicolumn{3}{c}{({\bf{b}}) $\gamma = 2$} \\[6pt]
\includegraphics[width=0.282\textwidth]{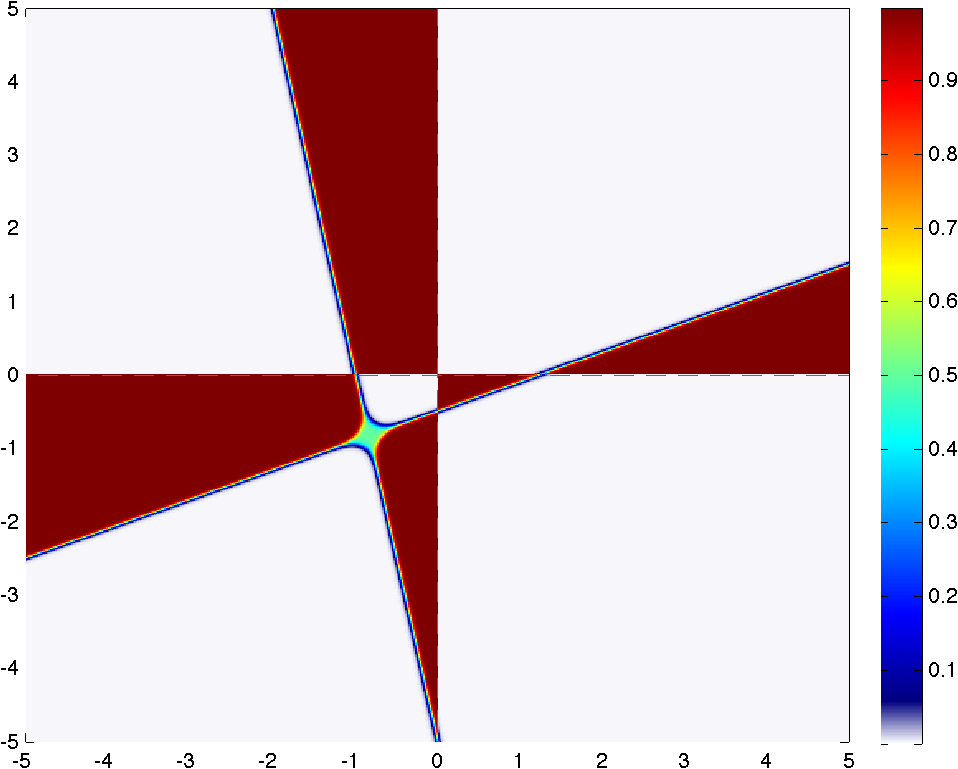}&
\includegraphics[width=0.282\textwidth]{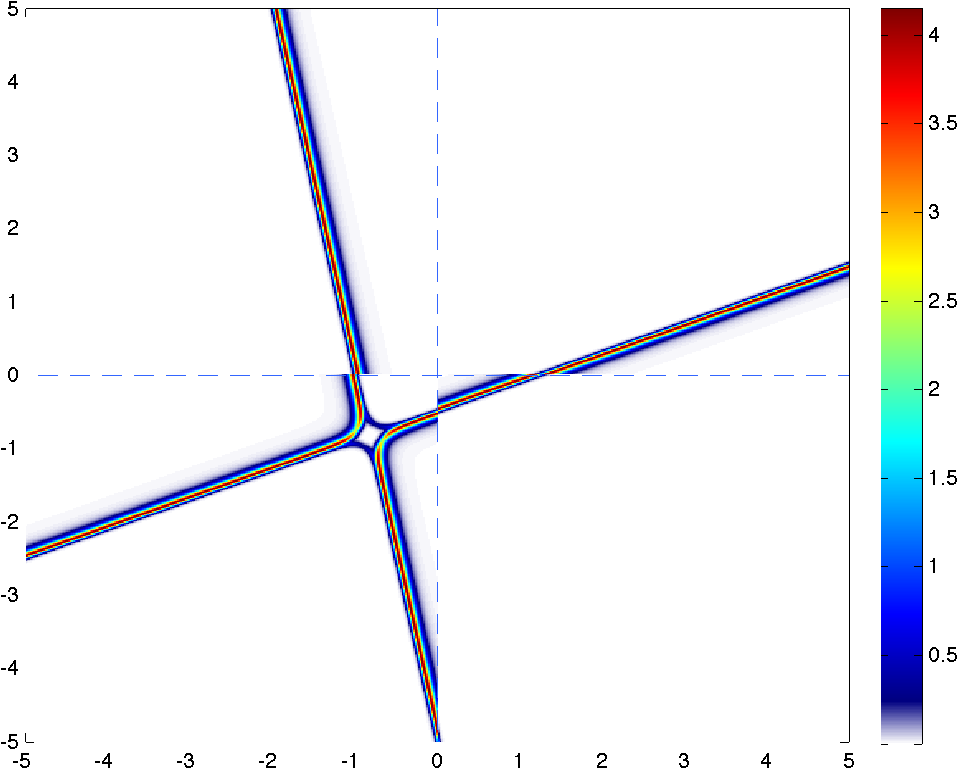}&
\includegraphics[width=0.282\textwidth]{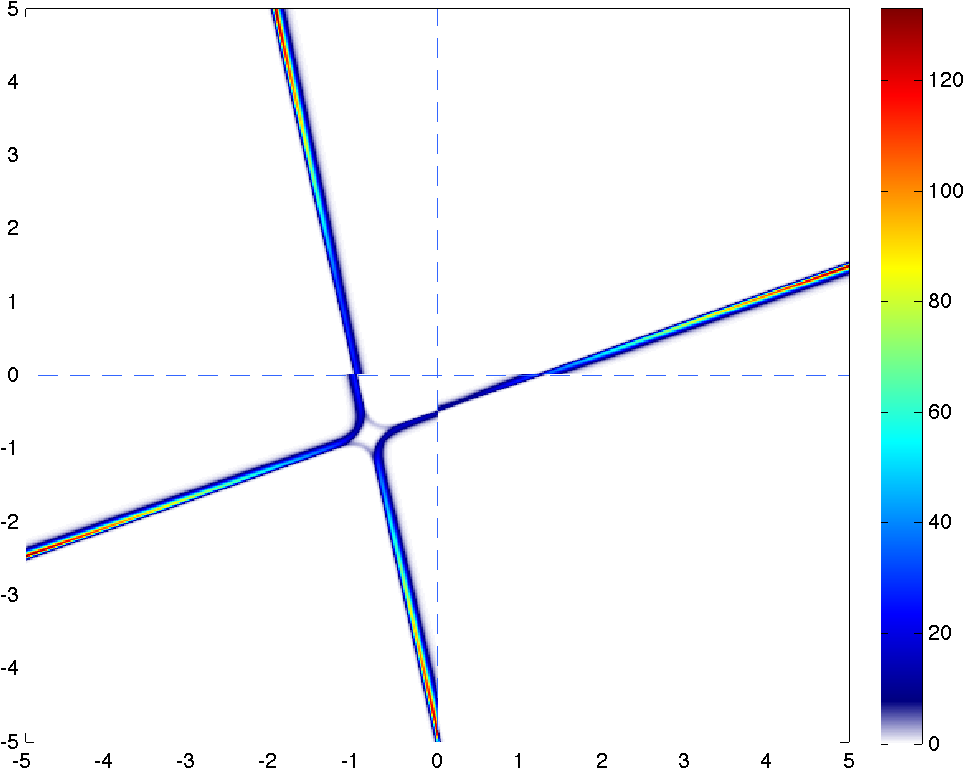}\\
\multicolumn{3}{c}{({\bf{c}}) $\gamma = 3$}
\end{tabular}
\caption{Distribution over the feature space of the maximum gradient
  components for the weight matrices in each of the three layers. The
  weight and bias terms in all layers are scaled by a factor
  $\gamma$.}\label{Fig:BackpropMaxGamma}
\end{figure}

In the next stage of the backpropagation we multiply $y_3$ by the
transpose of $A_3$ to obtain $z_3$, shown in
Figure~\ref{Fig:Backprop}(d). As an intermediate value $z_3$ itself is
not used, but it is further backpropagated though multiplication by
$\sigma_3'(v_2)$ to obtain $y_2$. As illustrated in
Figure~\ref{Fig:Sigmoid}(a), the gradient of the sigmoid is a kernel
around the origin, and when applied to $v_2$, the preimage of $x_2$
under $\sigma_3$, it emphasizes the regions of low confidence and
suppresses the regions of high confidence. This can be seen when
comparing the mask for $[v_2]_2$, shown in
Figure~\ref{Fig:Backprop}(e), with the corresponding region $[x_2]_2$
shown in Figure~\ref{Fig:NNExample2D}(d). The result obtained with
multiplication by the mask is illustrated in
Figure~\ref{Fig:Backprop}(f) and shows that backpropagation of the
error is most predominant in the boundary region as well as in some
regions where it was large to start with (most notably at the top of
the bottom-left quadrant). From here we can compute the partial
differential with respect to the entries of $A_2$ though
multiplication by $x_1$, which again damps values around the
transition region, and backpropagate further to get $z_2$, as shown in
Figures~\ref{Fig:Backprop}(g)--(i). To obtain $y_1$, we need to
multiply $z_2$ by the mask corresponding to the preimage of the
regions in $x_1$. Unlike all other layers, these values are unbounded
in the direction of the hyperplane normal and, as shown in
Figure~\ref{Fig:Backprop}(j), result in masks that vanish away from
the boundary region. Multiplication by the mask corresponding to
$[v_1]_1$ gives $[y_1]_1$ shown in Figure~\ref{Fig:Backprop}(k). We
finally obtain the partial differentials with respect to the entries
in $A_1$ by multiplying by the corresponding entries in $x_0$.  For
the first layer this stage actually amplifies the gradient entries
whenever the corresponding coordinate value exceeds one in absolute
value.  In subsequent layers the maximum values of $x$ lie in the -1
to 1 output range of the sigmoid function and can therefore only
reduce the resulting gradient components.

In Figure~\ref{Fig:BackpropMaxGamma}(a) we plot the maximum absolute
gradient components for each of the three weight matrices. It is clear
that the partial differentials with respect to $A_3$ are predominant
in misclassified regions, but also exist outside of this region in
areas where the objective function could be minimized further by
increasing the confidence levels (scaling up the weight and bias
terms). In the second layer, the backpropagated values are damped in
the regions of high confidence and concentrate around the decision
boundaries, which, in turn, are aligned with the underlying
hyperplanes. Finally, in the first layer, we see that gradient values
away from the hyperplanes have mostly vanished as a result of
multiplication with the sigmoid gradient mask, despite the
multiplication with potentially large coordinate values. Overall we
see the tendency of the gradients to become increasingly localized in
feature space towards the first layer. The boundary shifts we
discussed in Section~\ref{Sec:BoundaryShifts} can lead to additional
damping, as the sigmoid derivative masks no longer align with the
peaks in the gradient field. Scaling of the weight and bias is
detrimental to the backpropagation of the error (a phenomenon that is
also known as saturation of the sigmoids~\cite{LEC1998BOMa}) and can
lead to highly localized gradient values. This is illustrated in
Figures~\ref{Fig:BackpropMaxGamma}(b) and (c) where we scale all
weight and bias terms by a factor of $2$ and $3$, respectively.
Especially in deep networks it can be seen that a single sharp mask in
one of the layers can localize the backpropagating error and thereby
affect all preceding layers. These figures also show that the
increased scaling of the weights not only leads to localization, but
also to attenuation of the gradients. In the first layer this is
further aided by the multiplication with the coordinate
values. Summarizing, we see that the repeated multiplication by the
masks generated by the derivative of the activation function tends to
localize gradients. Multiplication with $A^T$ in the back propagation
mixes the regions with large gradients, but the location of these
regions does not otherwise change. Finally, we note that the above
principles are not restricted to the sigmoid or hyperbolic tangent
functions. However, for activation functions where the gradient masks
as not localized, for example for rectified linear units of the form
$\max(\alpha x, \beta x)$, the vanishing gradient problem is less of a
problem.

\section{Optimization}\label{Sec:Training}
In the previous section we studied how individual training samples
contribute to the overall gradient \eqref{Eq:GradientSum} of the loss
function. In this section we take a closer look at the dynamic
behavior and the changing relevance of training samples during
optimization over the network parameter vector $s$. The parameter
updates are gradient-descent steps of the form
\begin{equation}\label{Eq:GD}
s^{k+1} = s^k - \alpha \nabla\phi(s^k),
\end{equation}
with learning rate $\alpha$. The goal of this section is to clarify
the relationships between the training set and the optimization
process of the network parameters. To keep things simple we make no
effort to improve the efficiency of the optimization process and,
unless noted otherwise, we use a fixed learning rate with a moderate
value of $\alpha = 0.01$. Likewise, we compute the exact gradient
using the entire training set instead of using an approximation based
on suitably chosen subsets, as is done in practically favored
stochastic gradient descent (SGD) methods. Note, however, that the
mechanisms exposed in this section are general enough to carry over to
these and other methods without substantial changes. Similar findings
may moreover apply to other models. Throughout this section we place a
particular emphasis on the first layer of the neural network. To
illustrate certain mechanisms it often helps to keep parameters of
subsequent layers fixed. In this case it is implied that the
corresponding entries in the gradient update in \eqref{Eq:GD} are
zeroed out.

\subsection{Sampling density and transition width}

To investigate the roles of sampling density and transition width we
start with a very simple example with feature vectors $x \in
[-6,6]\times [-1,1] \subset \mathbb{R}^2$, and two classes: one to the
left of the $y$-axis (first entry is negative), and one to the
right. Example training sets with samples in each of the two classes
are plotted in Figure~\ref{Fig:SimpleDomain1}(a) and (b). Given such
training sets we want to learn the classification using a neural
network with a single hidden layer consisting of one node. To ensure
that the classes are well defined we place four training samples---two
for each class-- near the interface of the two classes and sample the
remaining points to the left and right of these points. Unless stated
otherwise we keep all network parameters fixed except for the weights
and bias terms in the first layer.

\subsubsection{Sampling density}

In the first experiment we study how the number or density of training
points affects the optimization. We initialize the network with
parameters
\[
A_1 =
{\textstyle\frac{25}{\sqrt{1.09}}}[1,0.3],\
b_1 = A_1\cdot[2,0]^T,\quad
A_2 = [3; -3],\ b_2 = [0;0],
\]
and keep the parameters in the second layer fixed. Parameter $b_1$ is
chosen such that the initial hyperplane goes through the point
$(2,0)$. Training sets consist of $n$ samples, including the four at
the interface, and are chosen such that the number of points in each
class differs by at most one. Figures~\ref{Fig:SimpleDomain1}(a) and
(b) illustrate such sets for $n=50$, and $n=800$, respectively. The
hyperplane is indicated by a thick black line, bordered with two
dashes lines which indicate the location where the output of the first
layer is equal to $\pm 0.95$. Figure~\ref{Fig:SimpleDomain1}(c) shows
the magnitude of the partial differential with respect to $[A_1]_1$ at
the initial parameter setting over the entire domain. The gradient
$\nabla\phi(s)$ is then computed as the average of the gradient values
evaluated at the individual training samples. As a measure of progress
we can look at the area of the misclassified region, i.e., the region
between the $y$-axis and the hyperplane (note this quantity does not
include information about the confidence levels of the
classification). Figure~\ref{Fig:SimpleDomain1}(e) shows this area as
a function of iteration for different sampling densities. The shape of
the loss function curves are very similar to these and we therefore
omit them here. For $n=800$ and $n=3200$ the area of the misclassified
region steadily goes down to zero, although the rate at which it does
so gradually diminishes. Although not apparent from the curves, this
phenomenon happens for all the training sets used here and we will
explain exactly why this happens in
Section~\ref{Sec:ConstraintsRegularization}.  Progress for $n=50$ and
$n=200$ appears much less uniform and exhibits pronounced stages of
fast and slow progress. The reason for this is a combination of the
sampling density and the localized gradient. From
Figure~\ref{Fig:SimpleDomain1}(c) we can see that the gradient field
is concentrated around the hyperplane, with peak values slightly to
the left of the hyperplane. When the sampling density is low it may
happen that none of the training samples is close to the
hyperplane. When this happens, the gradient will be small, and
consequently progress will be slow. When one or more points are close
to the hyperplane, the gradient will be larger and progress is
faster. Figure~\ref{Fig:SimpleDomain1}(f) shows the rate of change in
the area of the misclassified region along with the distance between
the hyperplane and its nearest training sample for $n=50$. It can be
seen that the rate increases as the hyperplane moves towards the
training sample, with the peak rate happening just before the
hyperplane reaches the point. After that the rate gradually drops
again as the hyperplane slowly moves further away from the sample.
This is precisely the state at 300,000 iterations, which is
illustrated in Figure~\ref{Fig:SimpleDomain1}(d). For $n=25$, we find
ourselves in the same situation right at the start. Initially we move
away from a single training point, but as a consequence of the low
sampling density, no other sampling points are nearby, causing a
prolonged period of very slow progress. The discrete nature of
training samples is less pronounced when the overall sampling density
is high, or when the transition widths are large.

\begin{figure}
\begin{tabular}{cc}
\includegraphics[width=0.475\textwidth]{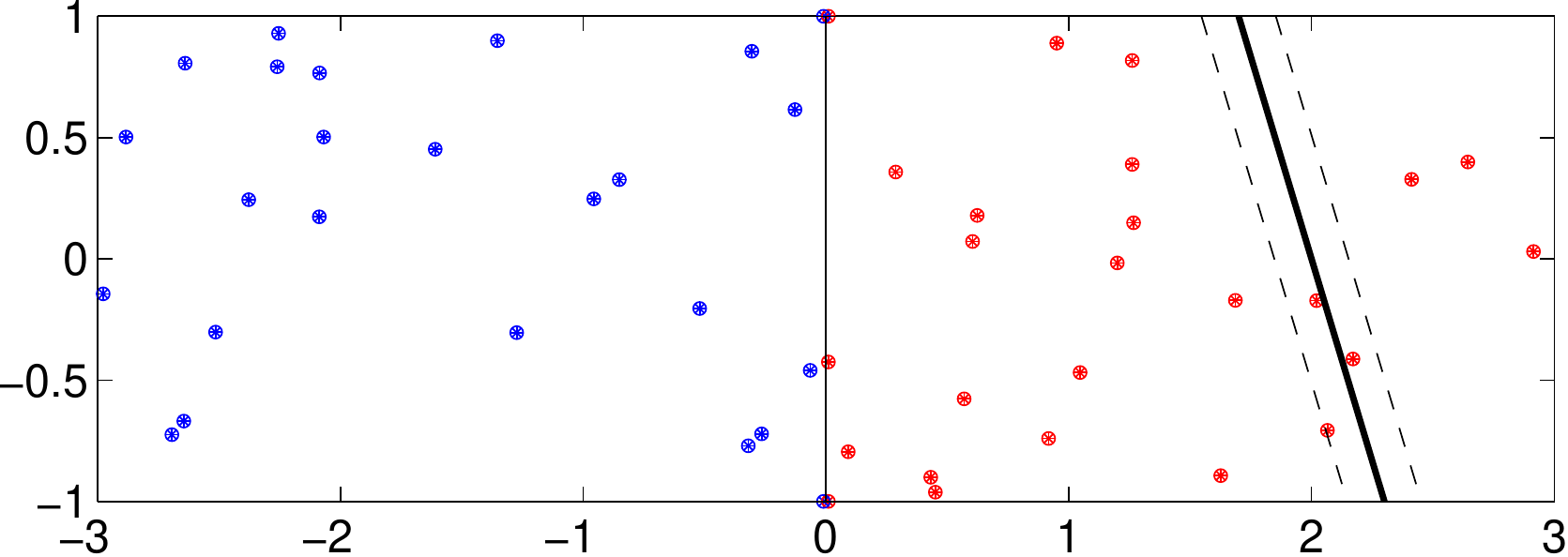}&
\includegraphics[width=0.475\textwidth]{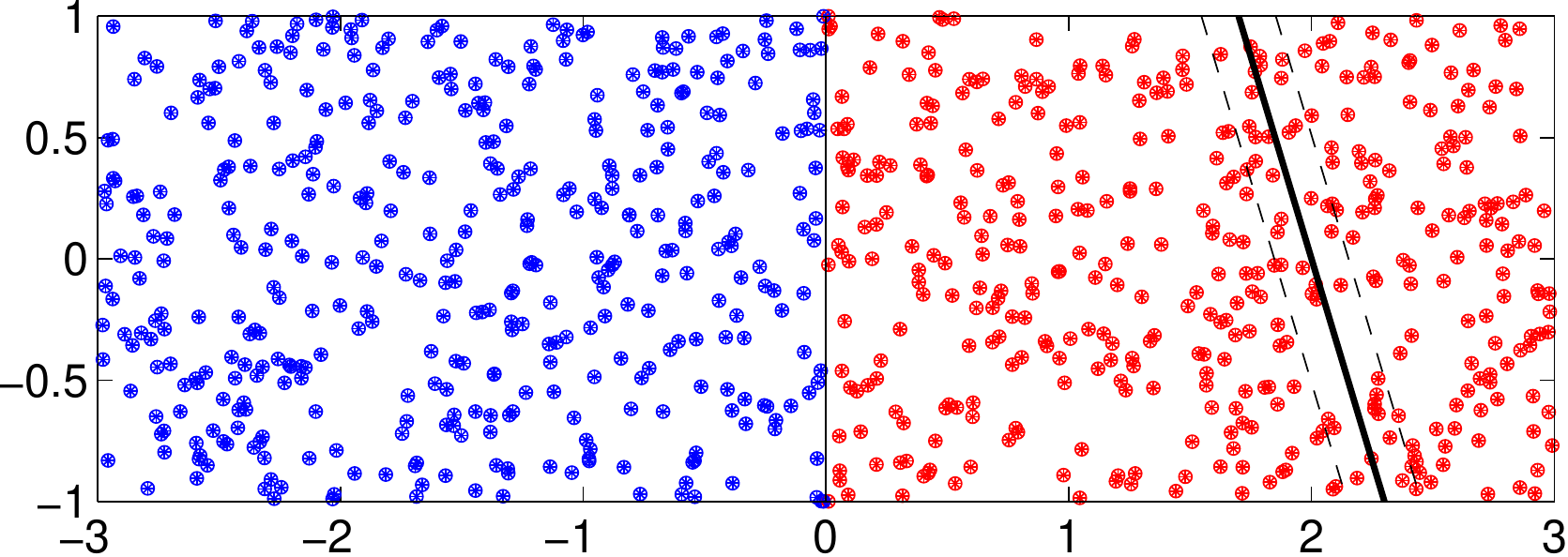}\\
({\bf{a}}) & ({\bf{b}})\\[2pt]
\includegraphics[width=0.465\textwidth]{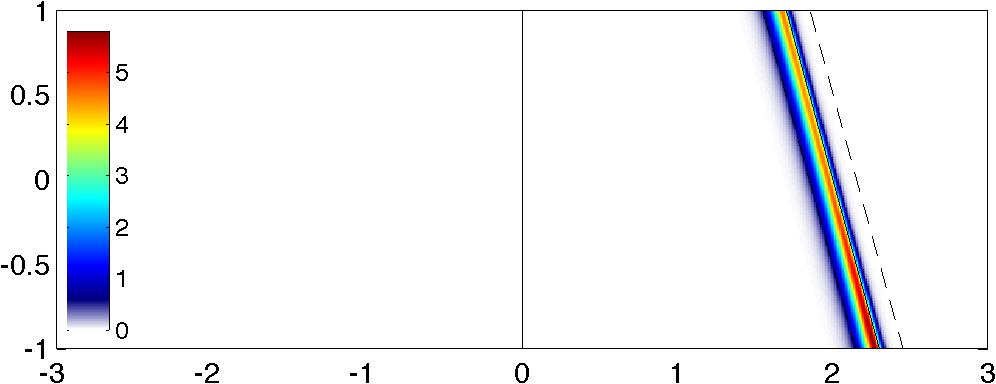}&
\includegraphics[width=0.475\textwidth]{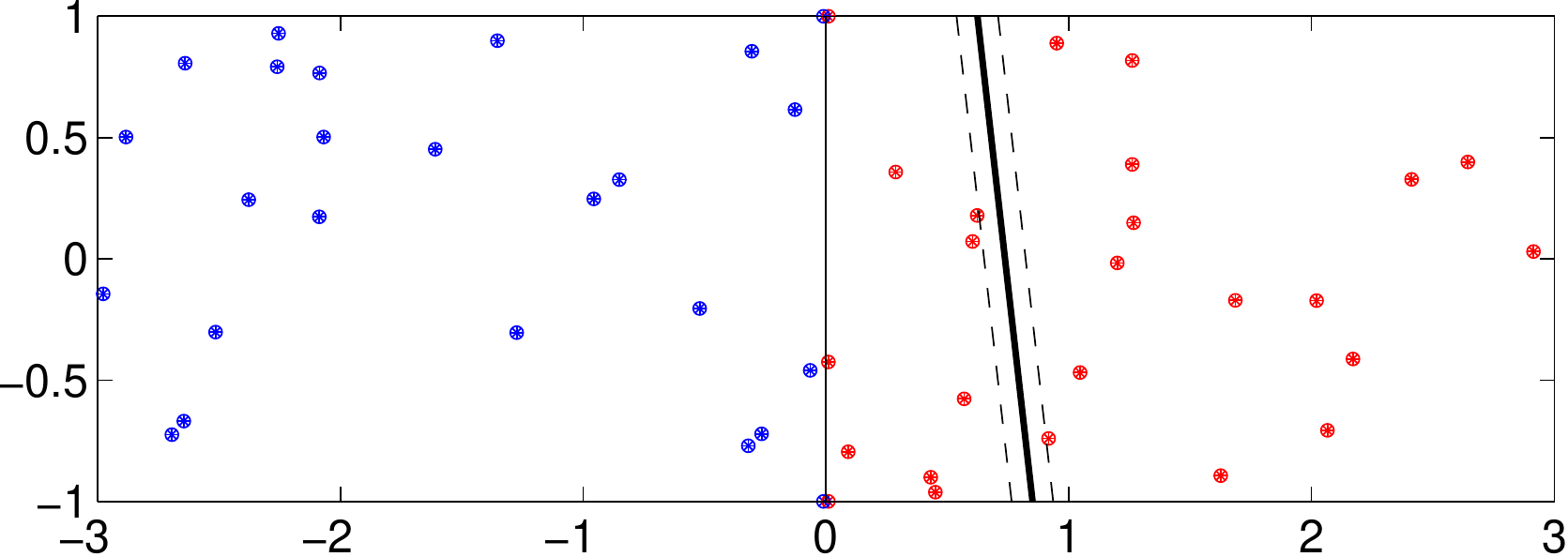}\\
({\bf{c}}) & ({\bf{d}})\\[2pt]
\includegraphics[width=0.475\textwidth]{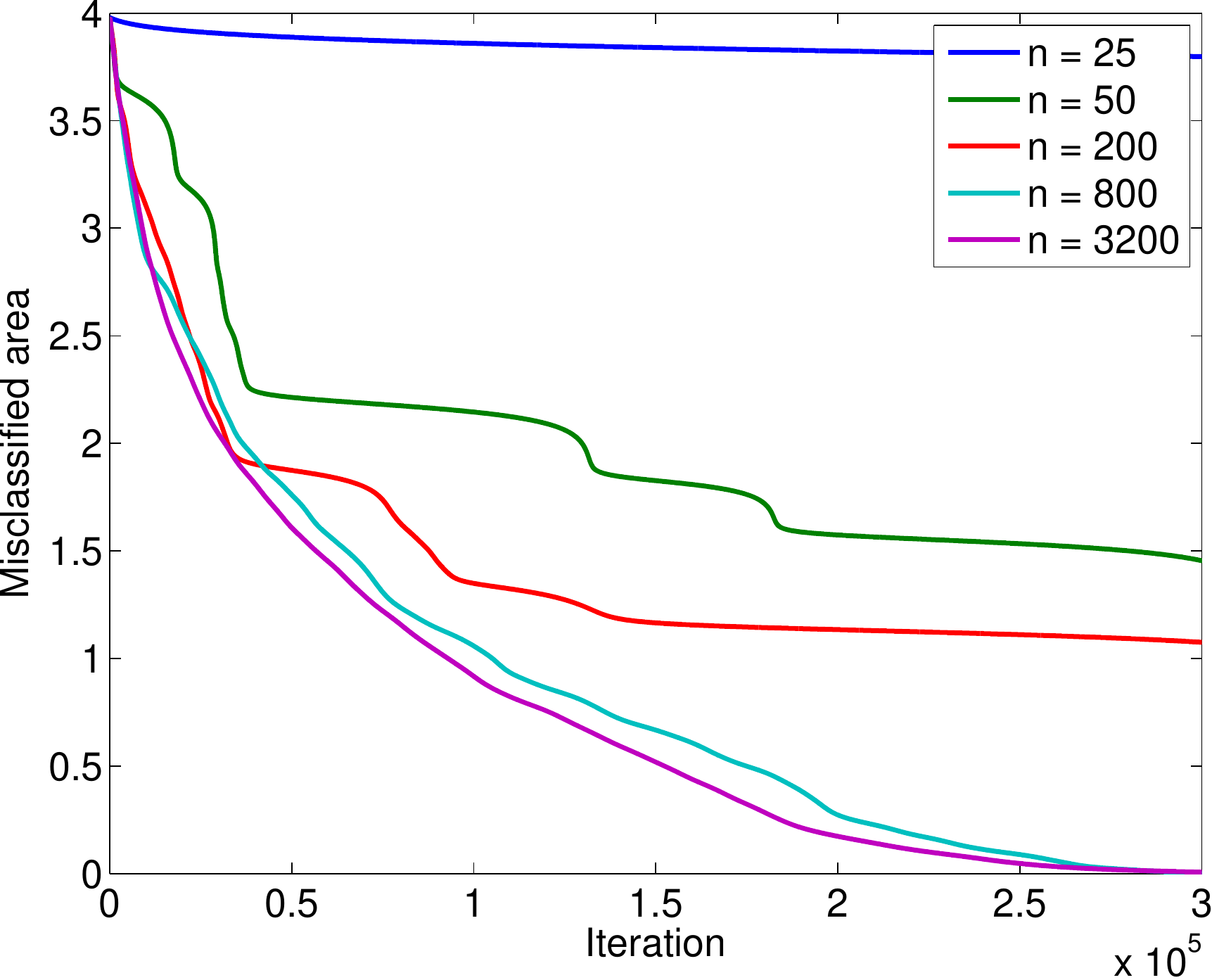}&
\includegraphics[width=0.475\textwidth]{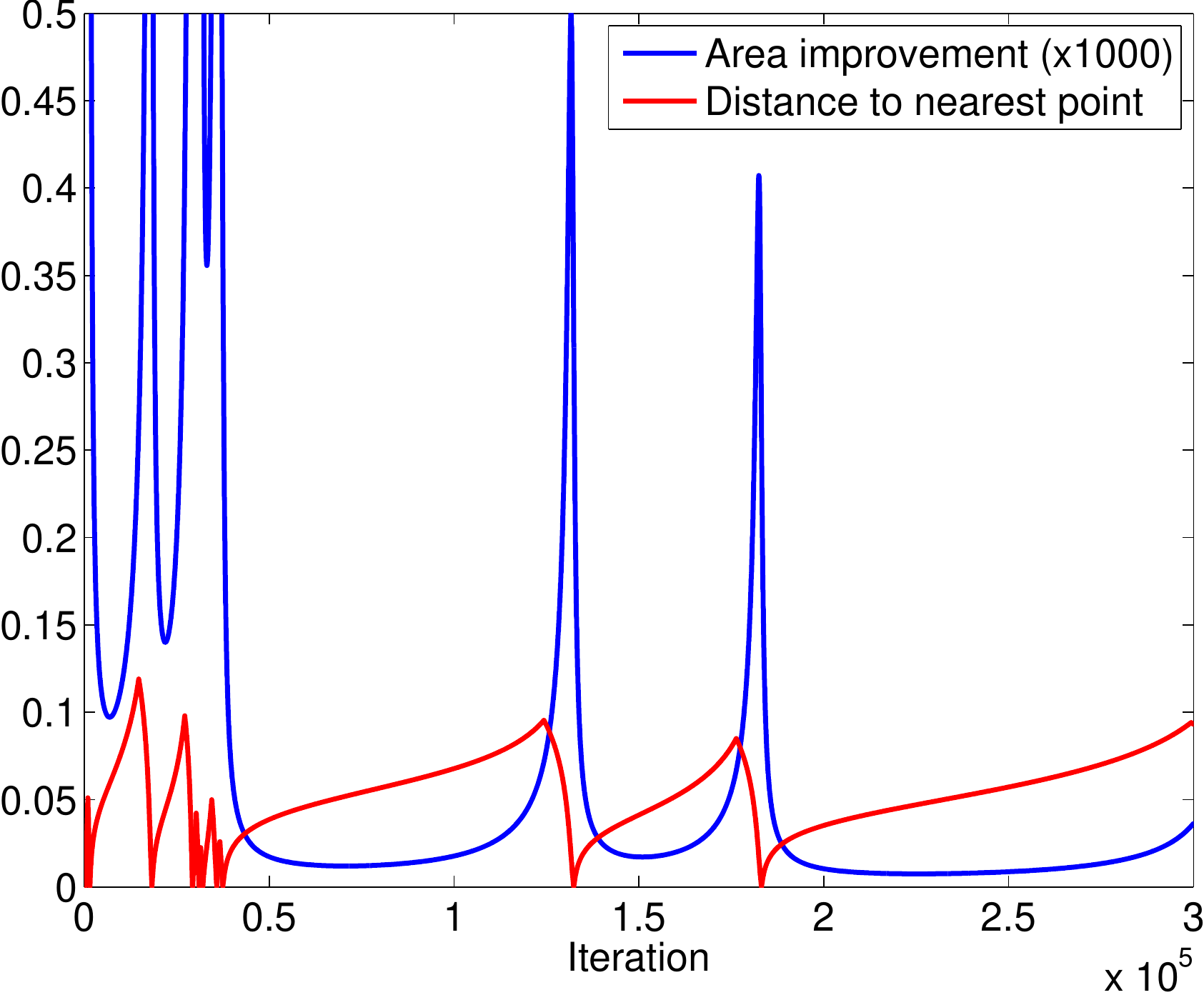}\\
({\bf{e}}) & ({\bf{f}})\\
\end{tabular}
\caption{Simple domain and initial hyperplane location with (a) 50,
  and (b) 800 training samples equally divided over two classes, along
  with (c) the initial gradient field. Plot (d) shows the location of
  the hyperplane after 300,000 iterations and (e) shows the area of
  the misclassified region as a function of iteration for different
  numbers of training samples. Plot (f) shows the reduction in
  misclassified area per iteration and the distance between the
  hyperplane and the nearest sample when 50 training samples are
  used.}\label{Fig:SimpleDomain1}
\end{figure}

\subsubsection{Transition width}

To illustrate the effect of transition widths, we used the setting
with 3,200 samples as described above, but scaled the row vector of
the initial $A_1$ to have Euclidean norm ranging from 1 to 100. In
each case we adjust $b_1$ such that the initial hyperplane goes
through the point $(2,0)$. As shown in
Figure~\ref{Fig:DifferentWeights}(a), the misclassified area reaches
zero almost immediately when $A_1$ is scaled to have unit norm. In
other words, the hyperplane is placed correctly in this case after
only 3,260 iterations. As the norm of the initial $A_1$ increases, it
takes longer to reach this point: for an initial norm of $10$ it takes
some 72,580 iterations, whereas for an initial norm of $25$ it takes
over 300,000. Accordingly, we see from
Figure~\ref{Fig:DifferentWeights}(b) that the loss also drops much
faster for small weights than it does for large weights. However, once
the hyperplane is in place, the only way to decrease the loss is by
scaling the weights to improve the confidence. This process can be
somewhat slow when the weights are small and the hyperplane placement
is finalized (as is the case when we start with small initial
weights). As a result, the setup with initial weight of 25 eventually
catches up with the earlier two, simply because it has a much sharper
transition at the boundary as the hyperplane finally closes in to the
right location.

\begin{figure}[t]
\centering
\begin{tabular}{cc}
\includegraphics[width=0.435\textwidth]{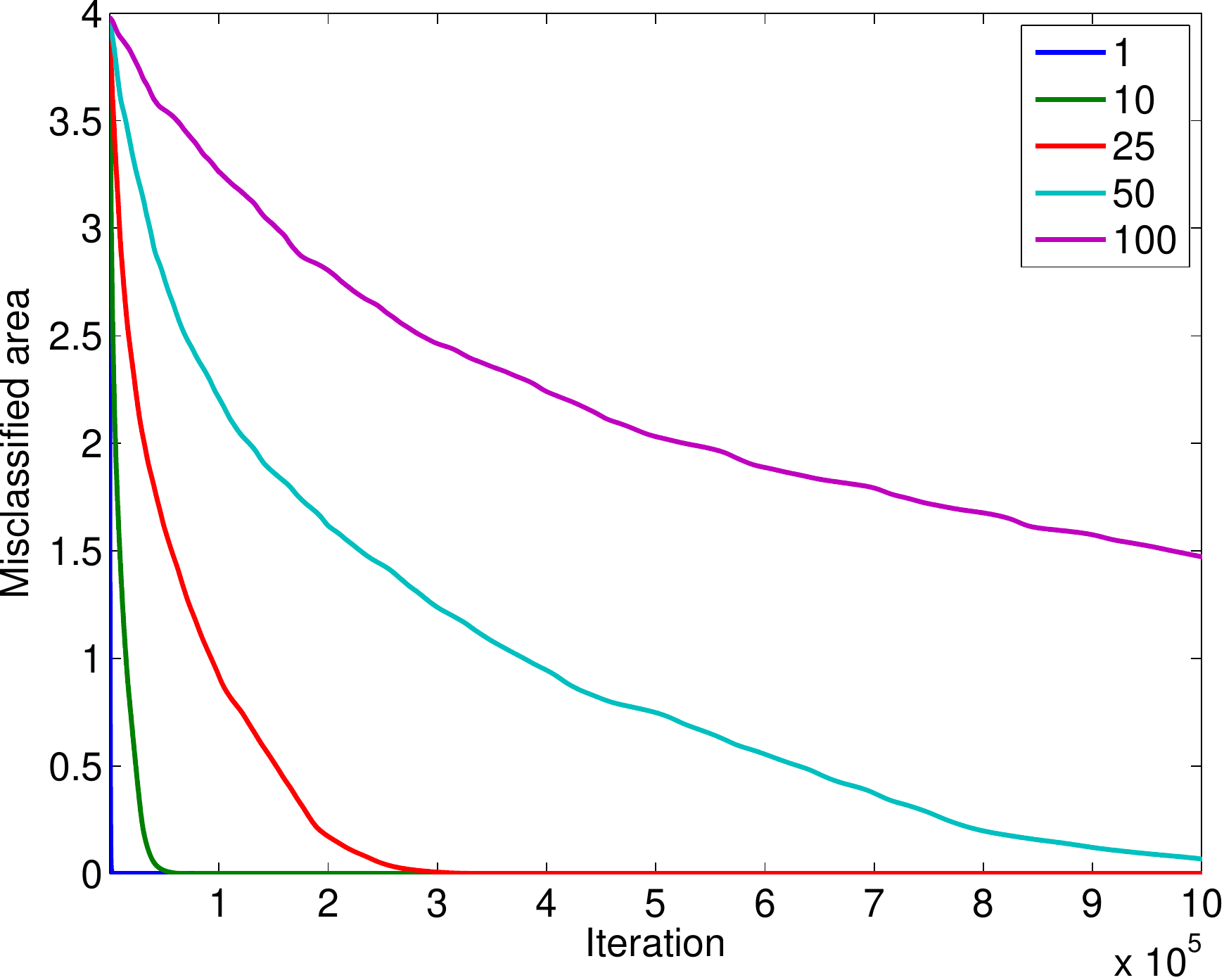}&
\includegraphics[width=0.435\textwidth]{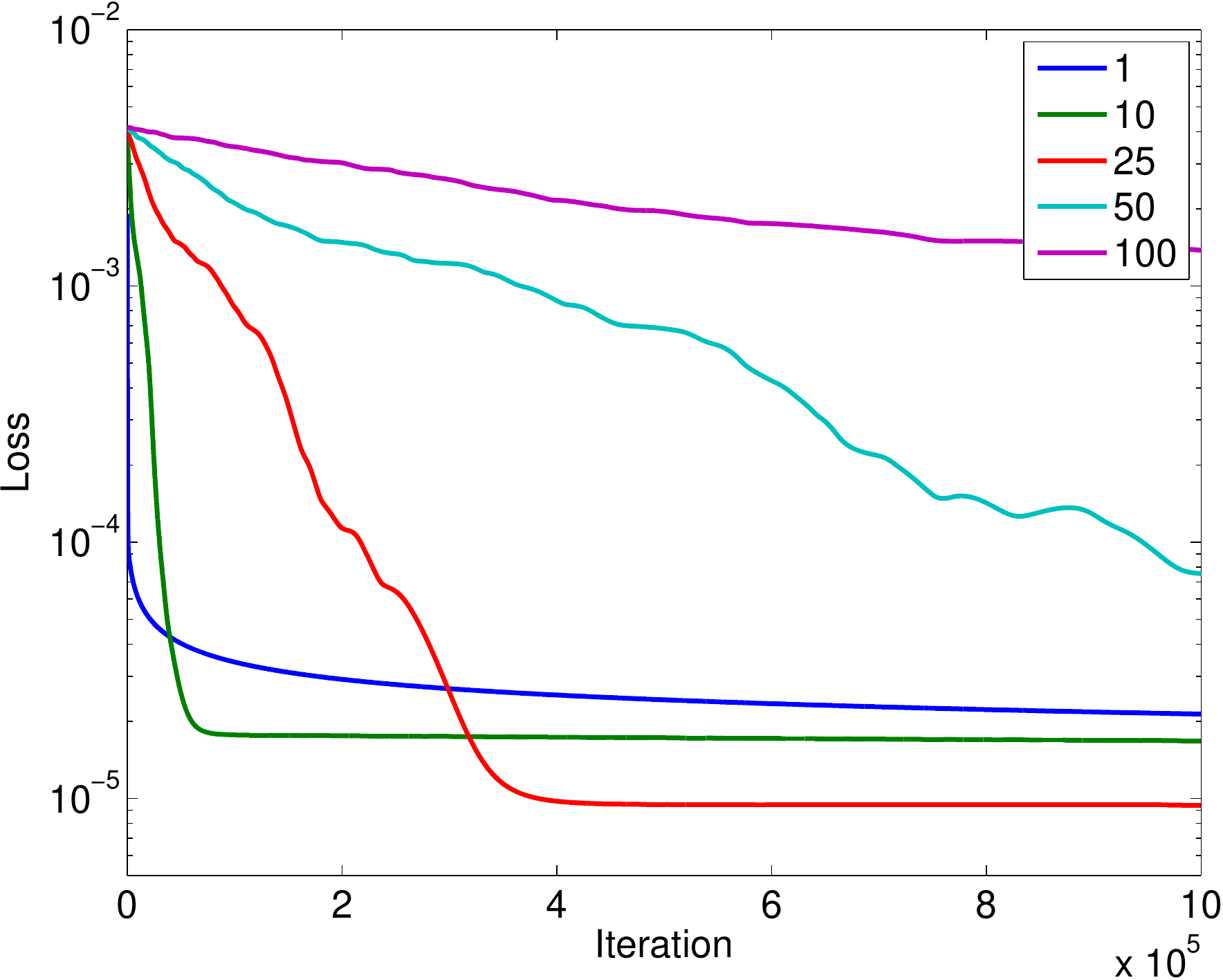}\\
({\bf{a}}) & ({\bf{b}})\\[4pt]
\includegraphics[width=0.435\textwidth]{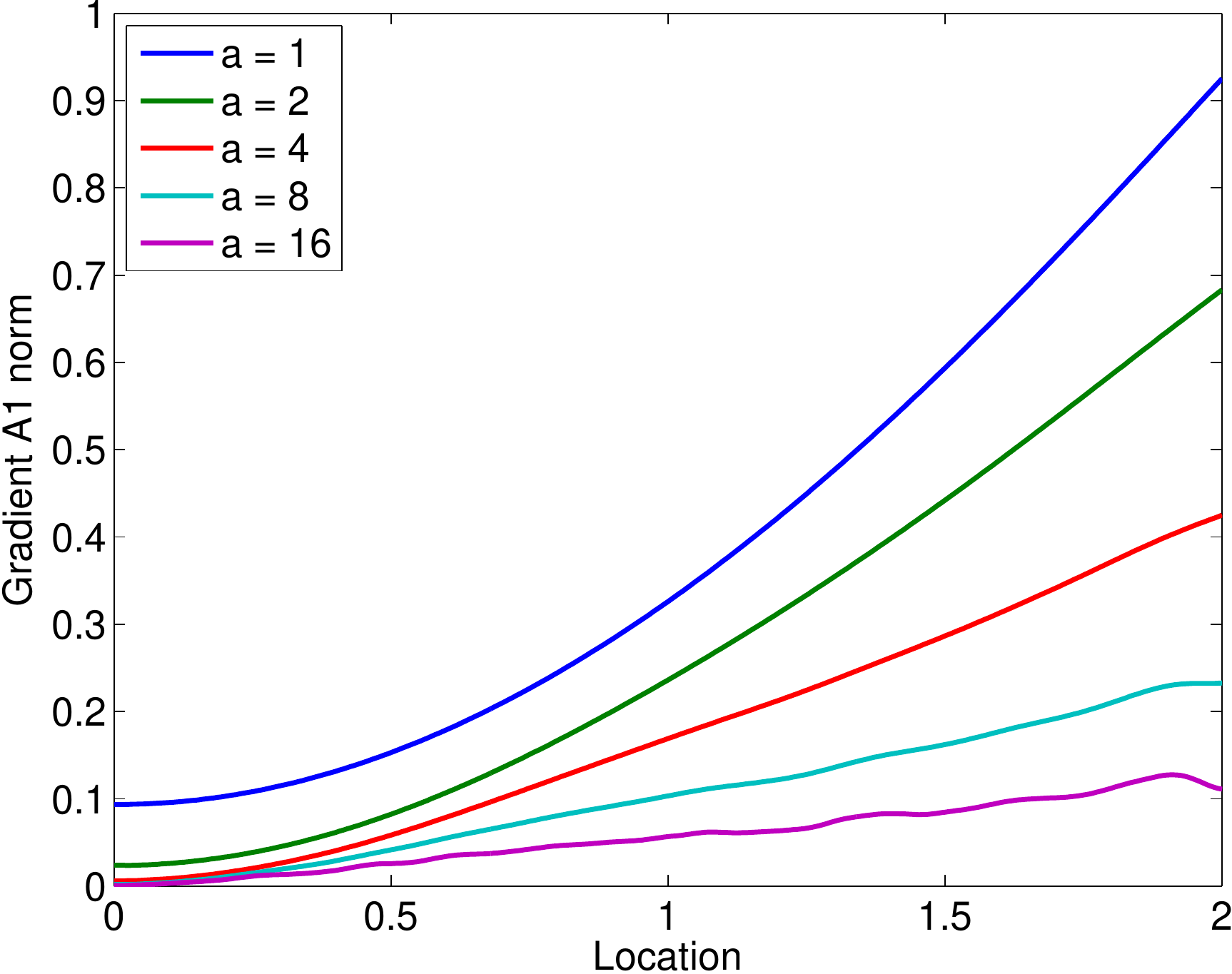}&
\includegraphics[width=0.435\textwidth]{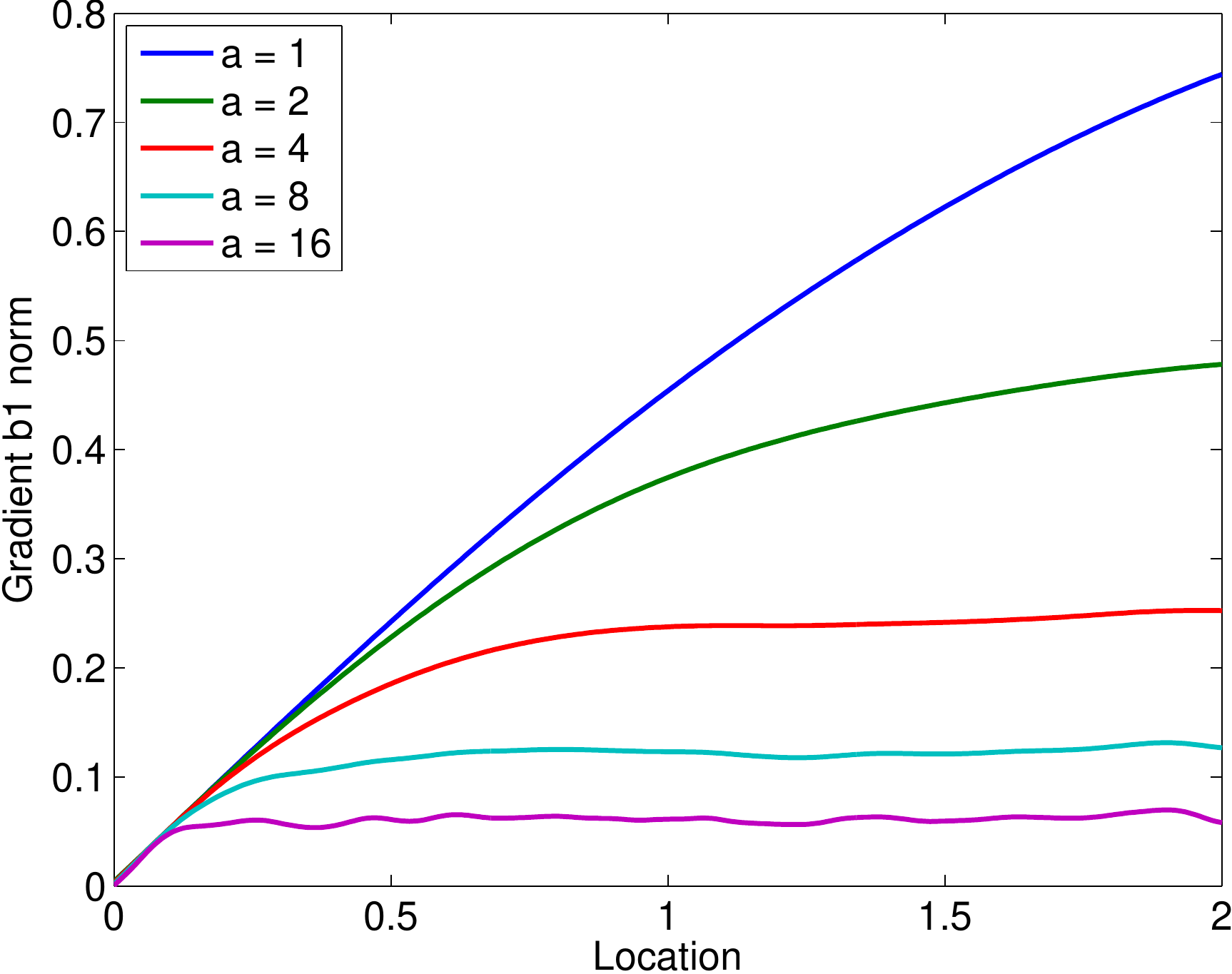}\\
({\bf{c}}) & ({\bf{d}})
\end{tabular}
\caption{Plots of (a) misclassified area and (b) the value of the loss
  function as a function of iterations for different initial
  weights. Norms of the gradients with respect to (c) $A_1$ and (d)
  $b_1$ as a function of hyperplane location with $A_1 = [a,0]$ for
  different values of $a$.}\label{Fig:DifferentWeights}
\end{figure}

The reason why the hyperplane moves faster for small initial weights
is twofold. First, the transition width and support of the gradient
field are larger. As a result, more sample points contribute to the
gradient, leading to a larger overall gradient value.  This is shown
in Figures~\ref{Fig:DifferentWeights}(c) and (d) in which we plot the
norm of the gradients with respect to $A_1$ and $b_1$ when choosing
$A_1 = [a,0]$, and $b$ such that the hyperplane goes through the given
location on the $x$-axis. The gradients with respect to either
parameters are larger for smaller $a$.  (Unlike in
Figure~\ref{Fig:BackpropMaxGamma}, the localization of the gradient
here is due only to scaling of the weights in the first layer; the
intensity of the gradient field therefore remains unaffected.) As the
value of $a$ increases, the curves in
Figures~\ref{Fig:DifferentWeights}(c) become more linear. For those
values the gradient is highly localized and, aside from the scaling by
the training point coordinates, largely independent of the hyperplane
location. The gradient with respect to $b_1$ does not include this
scaling and therefore remains nearly constant as long as the overlap
between the transition width and the class boundary is negligible. As
the hyperplane moves into the right place, the gradient vanishes due
to the cancellation of the contributions from the training points from
the classes on either side of it. The curves for $a=16$ and, to a
lesser extent for $a=8$, show minor aberrations due to a relatively
low sampling density compared to the transition width.  Second, having
larger gradient values for smaller weights means that the relative
changes in weights are amplified, thereby allowing the hyperplane to
move faster.

\subsection{Controlling the parameter scale}\label{Sec:ConstraintsRegularization}

In this section we work with a modified version of the domain shown in
Figure~\ref{Fig:SimpleDomain1}(a).  In particular, we change the
horizontal extent from $[-3,3]$ to $[-30,30]$, and randomly select 250
training samples uniformly at random for each of the two classes (thus
leaving the sampling density unaffected compare to the original
$n=50$). As a first experiment we optimize a three-layer network with
initial parameters:
\begin{equation}\label{Eq:ExtendedDomainThreeLayer}
A_1 = [1,0.3]/\sqrt{1.09},\ \ b_1 = A_1\cdot [25;0],\qquad
A_2 = 3,\ \ b_2 = 0,\qquad \mathrm{and}\qquad
A_3 = [3; -3],\ \ b_3 = 0.
\end{equation}
When we look at the row-norms of the weight matrices, plotted in
Figure~\ref{Fig:GrowingNorms}(a), we can see that all of them are
growing. This growth can help improve the final confidence levels, but
can be detrimental during the optimization process, especially when it
occurs in the layers between the first and the last. Indeed, we can
see from Figure~\ref{Fig:GrowingNorms}(b) that the hyperplane never
quite reaches the origin, despite the large number of iterations. As
illustrated in Figure~\ref{Fig:BackpropMaxGamma}, scaling of the
weight and bias terms leads to increasingly localized gradients. When
the training sample density is low compared to the size of the regions
where the gradient values are significant, it can easily happen that
no significant values from the gradient field are sampled into the
gradient. This applies in particular to the first several layers
(depending on the network depth) where the gradient fields become
increasingly localized (though not necessarily small) as a result of
the sigmoidal gradient masks that are applied during back propagation,
along with shifts in the boundary regions. This `vanishing gradient'
phenomenon can prematurely bring the training process to a halt; not
because a local minimum is reached, but simply because the sampled
gradient values are excessively small\footnote{Small gradients can
  also be due to cancellations in the various contributions. In
  practice, and especially when classes mix in a boundary zone, the
  small gradient can be expected to be due to a combination of
  the two effects.}. Scaling of the parameters in any layer except the
last can cause the gradient field to become highly localized for the
current and all preceding layers.  This can cause a cascading effect
in which suboptimal parameters in a stalled first layer lead to
further parameter scaling in later layers, eventually causing the
second layer to stall, and so on. To avoid this, we need to control
the parameter scale during optimization.

Parameter growth can be controlled by adding a regularization or
penalty term to the loss function, or by imposing explicit
constraints. Extending \eqref{Eq:NNTraining} we could use
\begin{equation}\label{Eq:NNTraining2}
\begin{array}{lcll}
\minimize{s}\quad \phi(s) + r(s),&\qquad\mbox{or}\qquad\ &
\minimize{s}&\phi(s)\\
&&\st &c_i(s) \leq 0,
\end{array}
\end{equation}
where $r(s)$ is a regularization function, and $c_i(s)$ are constraint
functions. The discussions so far suggest some natural choices of
functions for different layers. The function in the first layer should
generally be based on the (Euclidean) $\ell_2$ norm of each of the
rows in $A_1$, such as their sum, maximum, or $\ell_2$ norm. The
reason for this is that each row in $A_1$ defines the normal of a
hyperplane, and using any function other than an $\ell_2$ norm may
introduce a bias in the hyperplane directions due to a lack of
rotational invariance. For subsequent layers $k$ (except possibly the
last layer) we may want to ensure that the output cannot be too
large. In the worst case, each input from the previous layer is close
to $+1$ or $-1$, and we can limit the output value by ensuring that
the sum of absolute values, i.e., the $\ell_1$ norm, of each row in
$A_k$ is sufficiently small. Of course, the corresponding value in
$b_k$ could still be large, which may suggest adding a constraint that
$\norm{[A_k]_j}_1 \leq \abs{[b_k]_j}$ for each row $j$. However, this
constraint is non-convex and may impede sign changes in $b$.  The use
of an $\ell_1$ norm-based penalty or constraint on intermediate layers
has the additional benefit that it leads to sparse weight matrices,
which can help reduce model complexity as well as evaluation cost.

\begin{figure}[t]
\centering
\begin{tabular}{cc}
\includegraphics[width=0.425\textwidth]{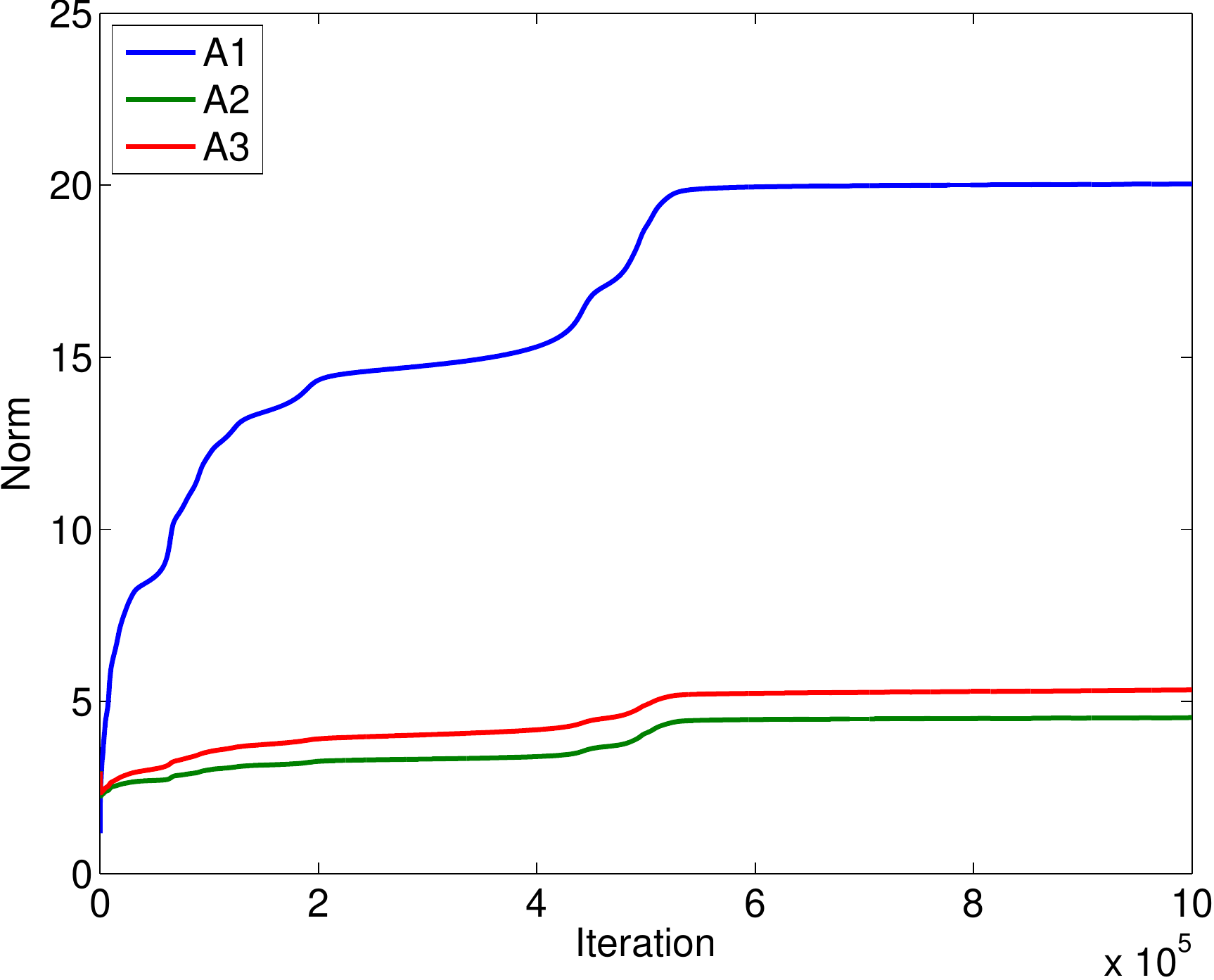}&
\includegraphics[width=0.425\textwidth]{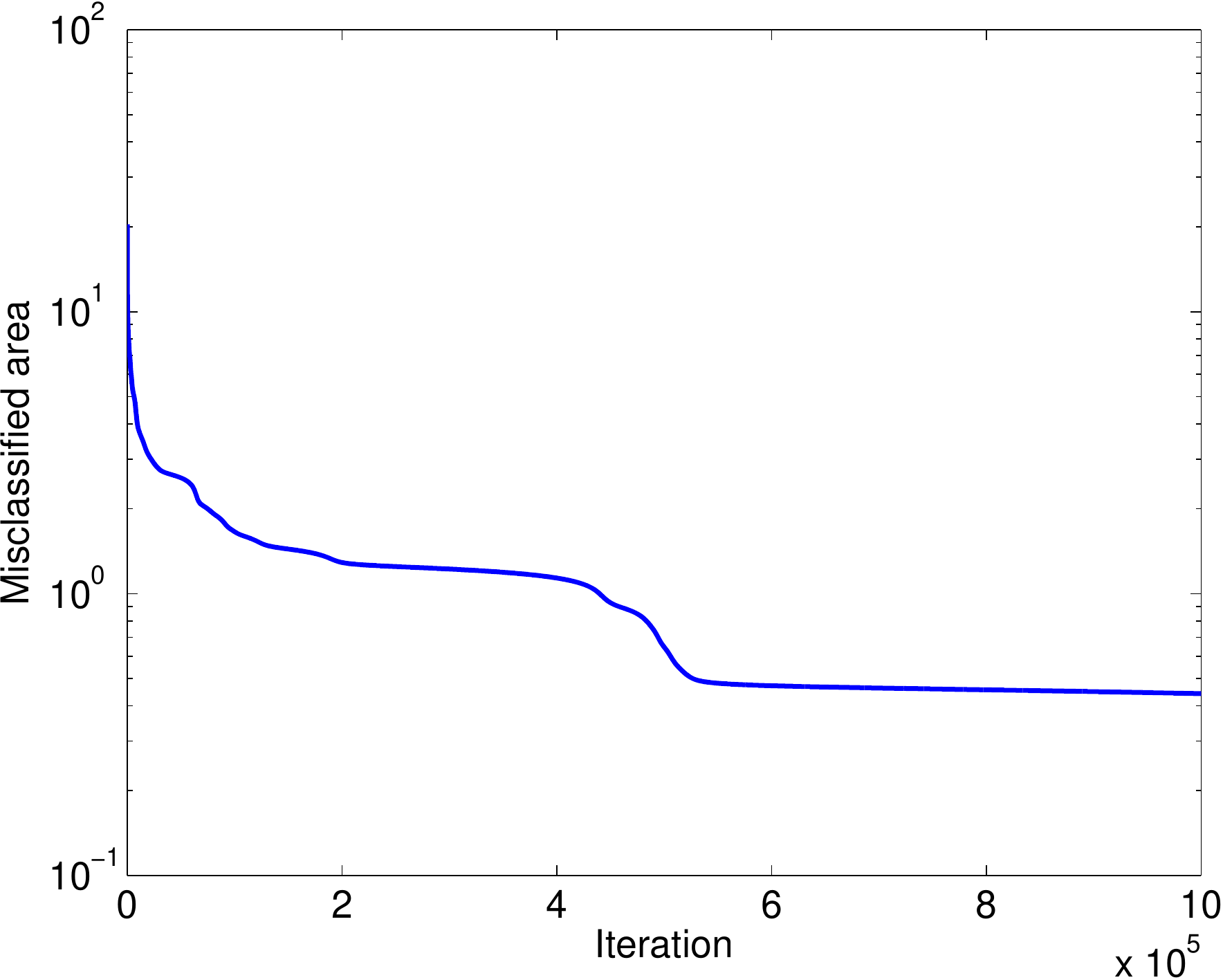}\\
({\bf{a}}) & ({\bf{b}})
\end{tabular}
\caption{Plots of (a) growth in the norms of the weight matrices and
  (b) reduction of the misclassified area as a function of
  iteration.}\label{Fig:GrowingNorms}
\end{figure}

As an illustration of the effect of $\ell_2$ regularization on the
first layer we consider the setting as given in
\eqref{Eq:ExtendedDomainThreeLayer}, but with the second layer
removed. We optimize the weight and bias terms in the first layer
using the standard formulation \eqref{Eq:NNTraining}, as well as those
in \eqref{Eq:NNTraining2} with $r(s) = \lambda/2 \norm{A_1^T}_2^2$ or
$c(s) = \norm{A_1^T}_2 \leq \kappa$. For simplicity we keep all other
network parameters fixed. Optimization in the constrained setting is
done using a basic gradient projection method with step size fixed to
0.01, as before. The results are show in
Figure~\ref{Fig:ExtendedDomain}.
When using the standard formulation we see from
Figure~\ref{Fig:ExtendedDomain}(a) that, like above and in
Figures~\ref{Fig:SimpleDomain1}(a,d), the $\ell_2$ norm of the row in
$A_1$ keeps growing. This is explained as follows: suppose the
hyperplane is vertical with $A_1$ of the form $[a, 0]$, and $b_1 =
b$. Then the area of the misclassified region is $2\abs{b} /
\abs{a}$. We can therefore reduce the misclassified area (and in this
case the loss function) by increasing $a$ and decreasing $b$, which is
exactly what happens. However, from Figure~\ref{Fig:ExtendedDomain}(b)
we can see that the rate at which the misclassified area is reduced
decreases. The reason for this is a combination of three factors.
First, the speed at which $\abs{b} / \abs{a}$ goes towards zero slows
down as $a$ gets larger. Second, the peak of the gradient field lies
along the hyperplane and shifts towards the origin with it. Because
the gradient in the first layer is formed by a multiplication of the
backpropagated error with the feature vectors (coordinates), the
gradient gets smaller too. Third, because of the growing norm of
$A_1$, the transition width shrinks and causes the gradient to become
more localized. As a result, fewer training points sample the gradient
field at significant values, leading to smaller overall gradients with
respect to both $A_1$ and $b_1$.

\begin{figure}[t]
\centering
\begin{tabular}{cc}
\includegraphics[width=0.45\textwidth]{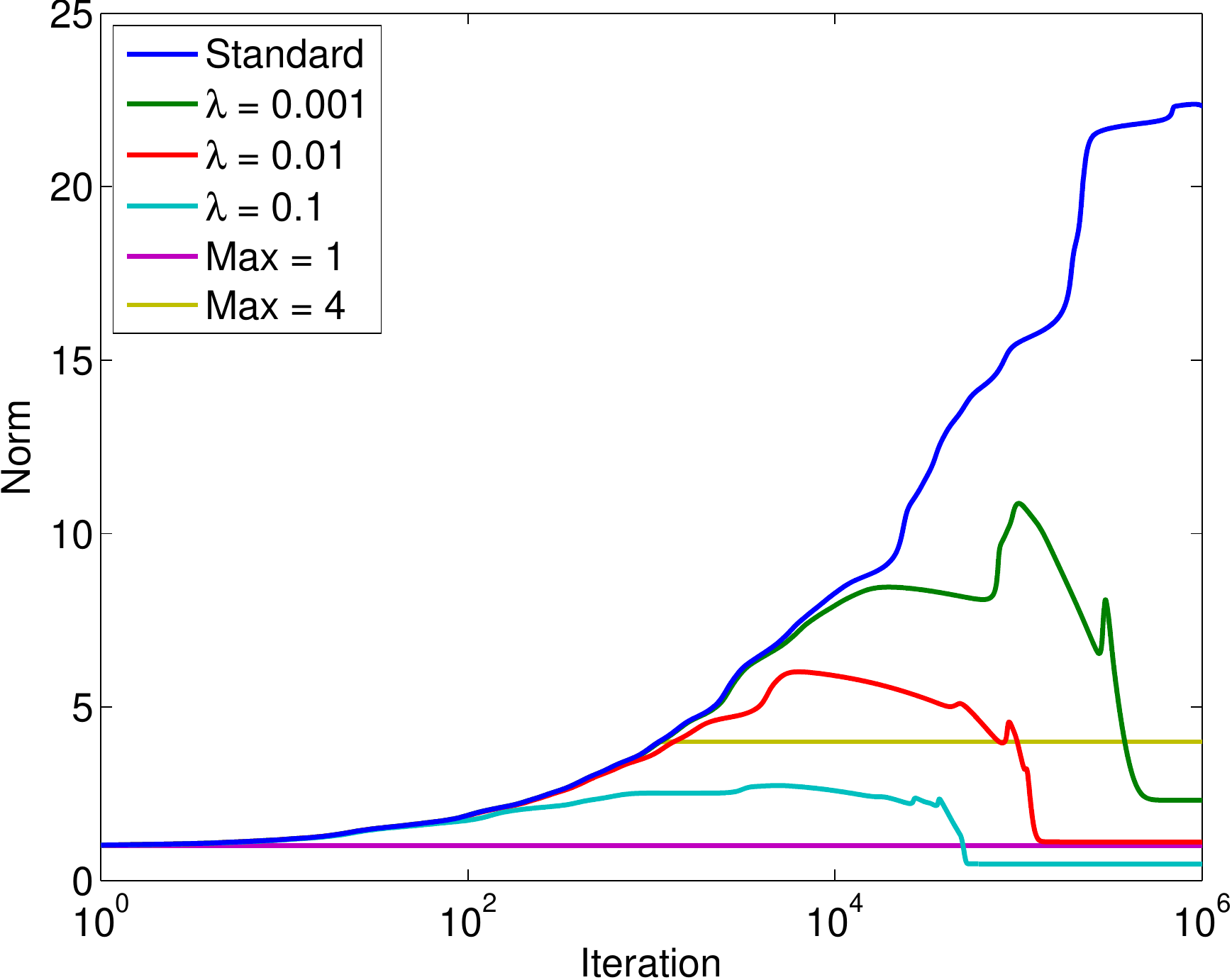}&
\includegraphics[width=0.45\textwidth]{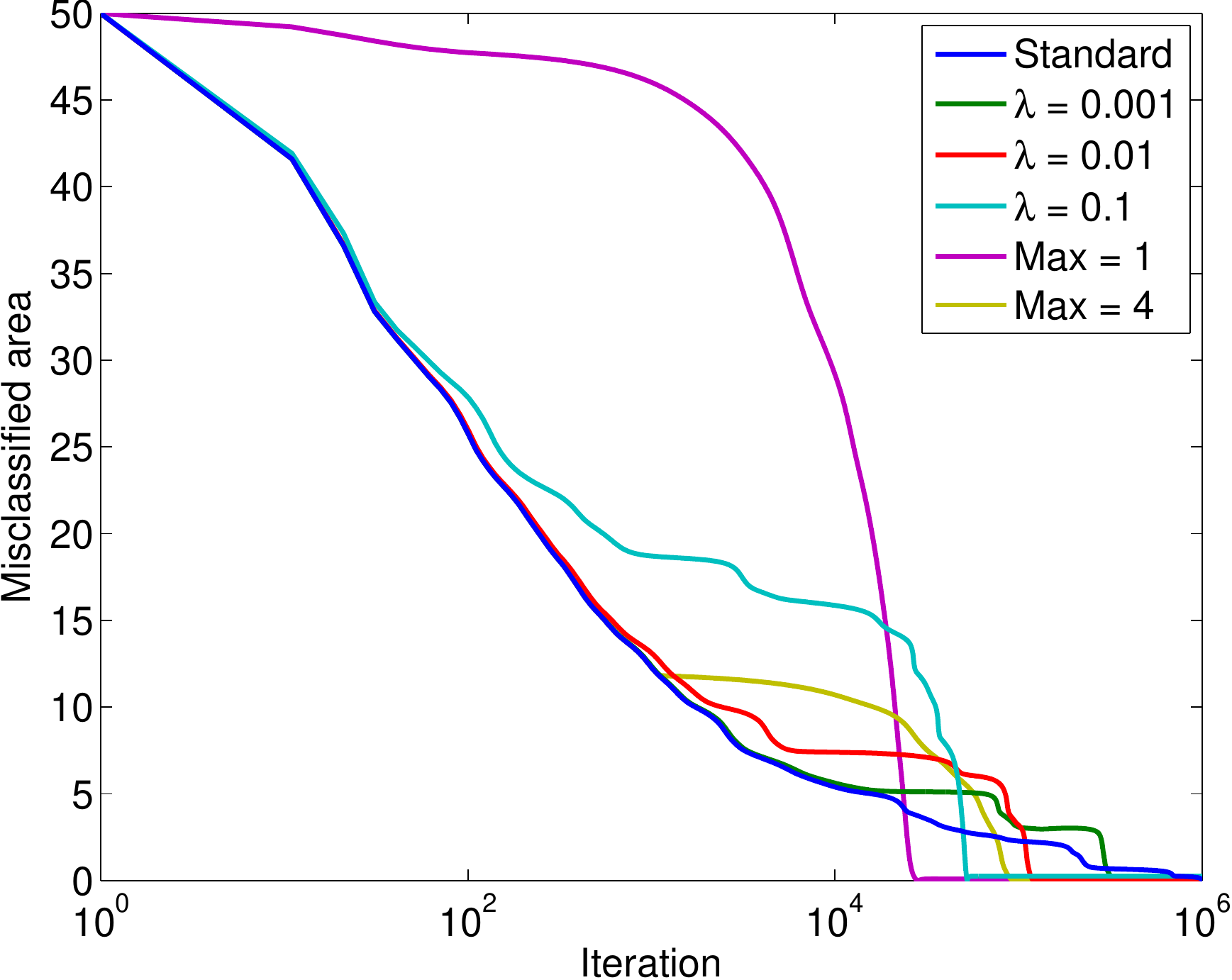}\\
({\bf{a}}) & ({\bf{b}}) \\[4pt]
\includegraphics[width=0.45\textwidth]{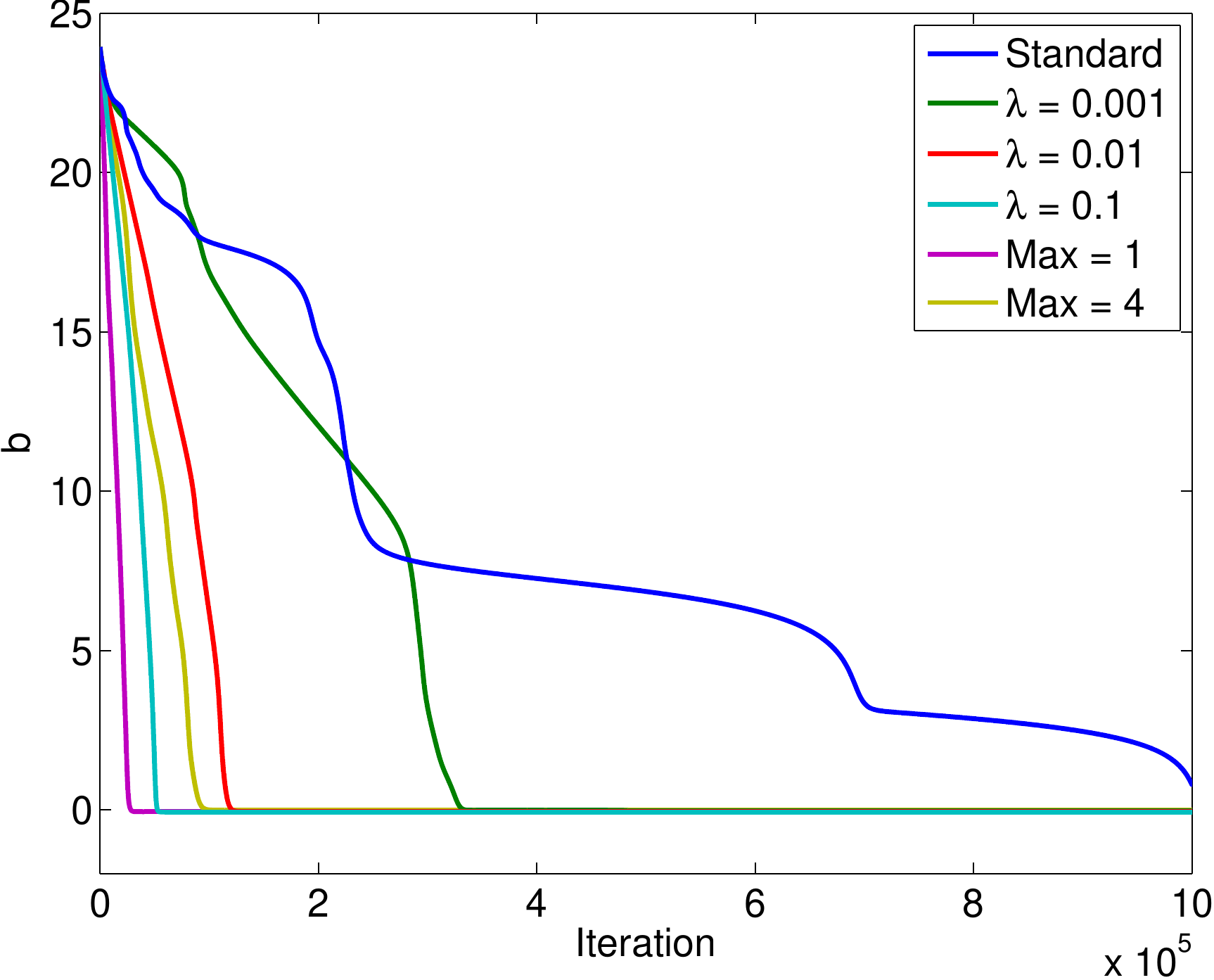}&
\includegraphics[width=0.45\textwidth]{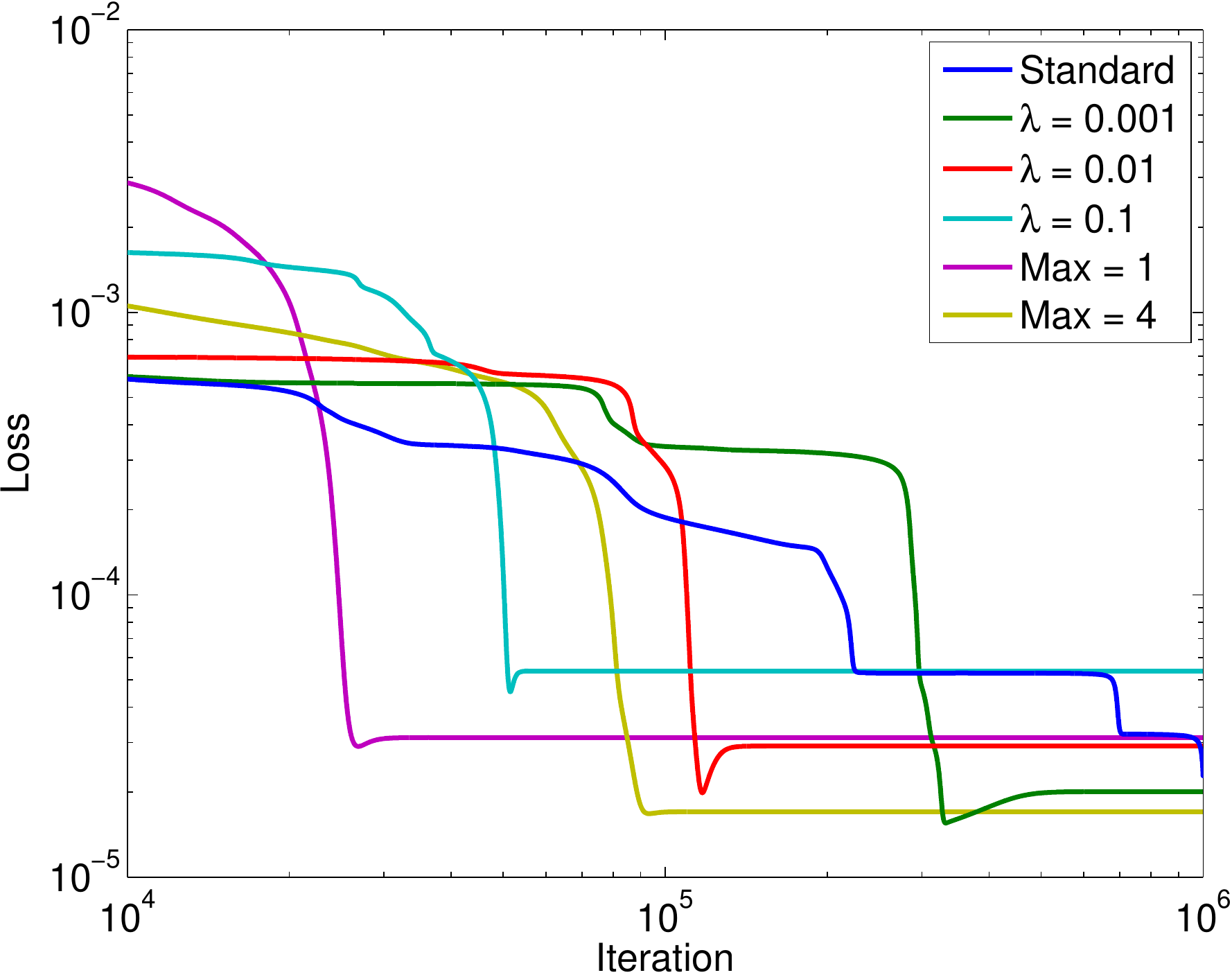}\\
({\bf{c}}) & ({\bf{d}})
\end{tabular}
\caption{Differences between (a) the norm of $A_1$; (b) the area of
  the misclassified region; (c) the magnitude of $b$; and (d) the loss
  function, as a function of iteration for standard gradient descent
  and variations of regularized and constrained
  optimization.}\label{Fig:ExtendedDomain}
\end{figure}

There is not much we can be do about the first two causes, but adding
a regularization term or imposing constraints, certainly does help
with the third, and we can see from Figure~\ref{Fig:ExtendedDomain}(a)
that the norm of $A_1$ indeed does not grow as much as in the standard
approach. At first glance, this seems to hamper the reduction of the
misclassified area, shown in Figure~\ref{Fig:ExtendedDomain}(b). This
is true initially when most of the progress is due to the scaling of
$A_1$, however, the moderate growth in $A_1$ also prevents strong
localization of the gradient and therefore results in much steadier
reduction of $b$, as shown in Figure~\ref{Fig:ExtendedDomain}(c). The
overall effect is that the constrained and regularized methods catch
up with the standard method and reduce the misclassified area to zero
first. Even so, when looking at the values of the loss function
without the penalty term, as plotted in
Figure~\ref{Fig:ExtendedDomain}(d), we see that the standard method
still reaches the lowest loss value, even though all methods have zero
misclassification. As before, this is because the two classes are
disjoint and are best separated with a very sharp transition. The
order in which the lines in Figure~\ref{Fig:ExtendedDomain}(d) appear
at the end, is therefore related to the norms in
Figure~\ref{Fig:ExtendedDomain}(a). This suggests the use of cooling
or continuation strategies in which norms are gradually allowed to
increase. The initial small weights ensure that many of the training
samples are informative and contribute to the gradients of all layers,
thereby allowing the network to find a coarse class alignment. From
there the weights can be allowed to increase slowly to fine tune the
classification and increase confidence levels. Of course, while doing
so, care needs to be taken not to allow excessive scaling of the
weights, as this can lead to overfitting.

Instead of scaling weight and bias terms we could also consider
scaling sigmoid parameters $\gamma$, or learn them
\cite{SPE1993Sa}. One interesting observation here is that even though
all networks with parameters $\alpha A_1$, $\alpha b_1$, and
$\gamma/\alpha$ are equivalent for $\alpha > 0$, their training
certainly is not. The reason is the $1/\alpha$ term that applies to
the gradients with respect to $A_1$ and $b_1$.  Choosing $\alpha > 1$
means larger parameter values and smaller gradients. This reduces both
the absolute and relative change in parameter values and is equivalent
to having a stepsize that is $\alpha^2$ smaller.  Instead of doing
joint optimization over both the layer and nonlinearity parameters, it
is also possible to learn the nonlinearity parameters as a separate
stage after optimization of the weight and bias terms.

\subsection{Subsampling and partial backpropagation}

Consider the scenario shown in Figure~\ref{Fig:SimpleDomain1}(a) and
suppose we double the number of training samples by adding additional
points to the left and right of the current domain. In the original
setting, the gradient with respect to the weights in the first layer
is obtained by sampling the gradient field shown in
Figure~\ref{Fig:SimpleDomain1}(c). In the updated setting, all newly
added points are located away from the decision boundary. As a result,
their contribution to the gradient is relatively small and the overall
gradient may be very similar to the original setting. However, because
the loss function $\phi(s)$ in \eqref{Eq:NNTraining} is defined as the
average of the individual loss-function components, we now need to
divide by $2N$ rather than $N$, thereby effectively scaling down the
gradient by a factor of approximately two. Another way to say this is
that the stepsize is almost halved by adding the new points. This
example is of course somewhat contrived, since additional training
samples can typically be expected to follow the same distribution as
existing points and therefore increase sampling density. Nevertheless,
this example makes one wonder whether the training samples on the left
and right-most side of the original domain are really needed; after
all, using only the most informative samples in the gradient
essentially amounts to larger stepsize and possibly a reduction in
computation.

For sufficiently deep networks with even moderate weights, the
hyperplane learning is already rather myopic in the sense that only
the training points close enough to the hyperplane provide information
on where to move it. This suggests a scheme in which we subsample the
training set and for one or more iterations work with only those
points that are relevant. We could for example evaluate $v_1 = A_1x_0
- b_1$ for each input sample $x_0$, and proceed with the forward and
backward pass only if the minimum absolute entry in $v_1$ is
sufficiently small (i.e., the point lies close enough to at least one
of the hyperplanes). This approach works to some extend for the first
layer when the remaining layers are kept fixed, however, it does not
generalize because the informative gradient regions can differ
substantially between layers (see e.g.,
Figure~\ref{Fig:BackpropMaxGamma}). Instead of forming a single
subsampled set of training points for all layers we can also form a
series of sets---one for each layer---such that all points in a set
contribute significantly to the gradient for the corresponding and
subsequent layers. This allows us to appropriately scale the gradients
for each layer. It also facilitates partial backpropagation in which
the error is backpropagated only up to the relevant layer, thereby
reducing the number of matrix-vector products. Given a batch of
points, we could determine the appropriate set by evaluating the
gradient contribution to each layer and finding the lowest layer for
which the contribution is above some threshold. Alternatively, we
could use the following partial backpropagation approach, which may be
beneficial in its own right, especially for deep networks.

\begin{figure}[t]
\centering
\begin{tabular}{ccc}
\includegraphics[width=0.32\textwidth]{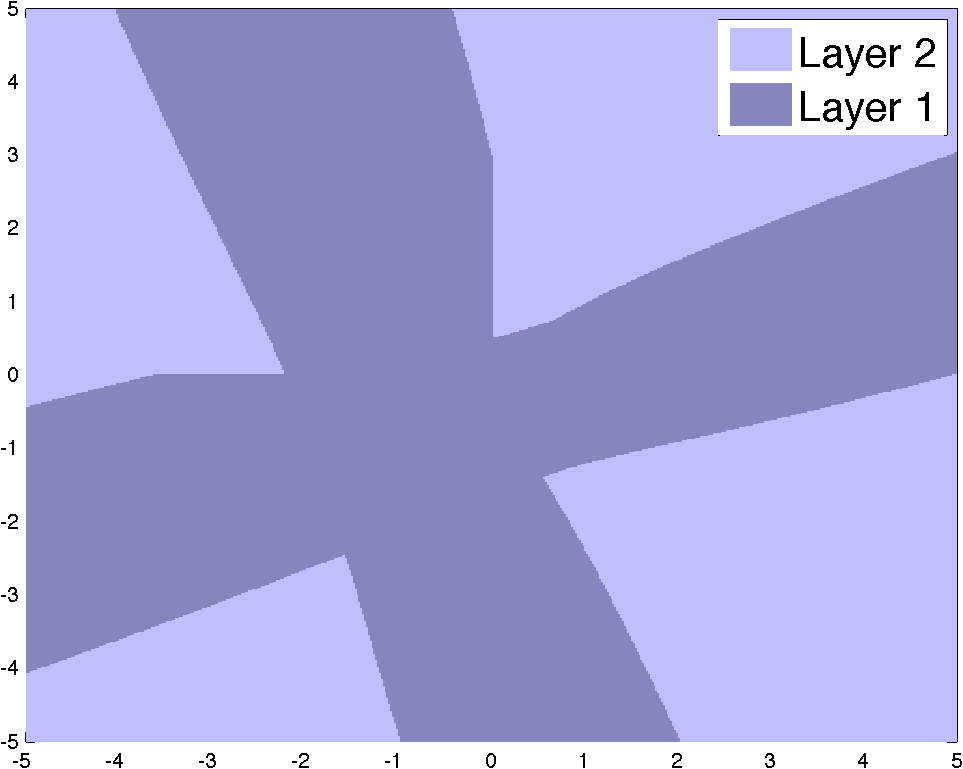}&
\includegraphics[width=0.32\textwidth]{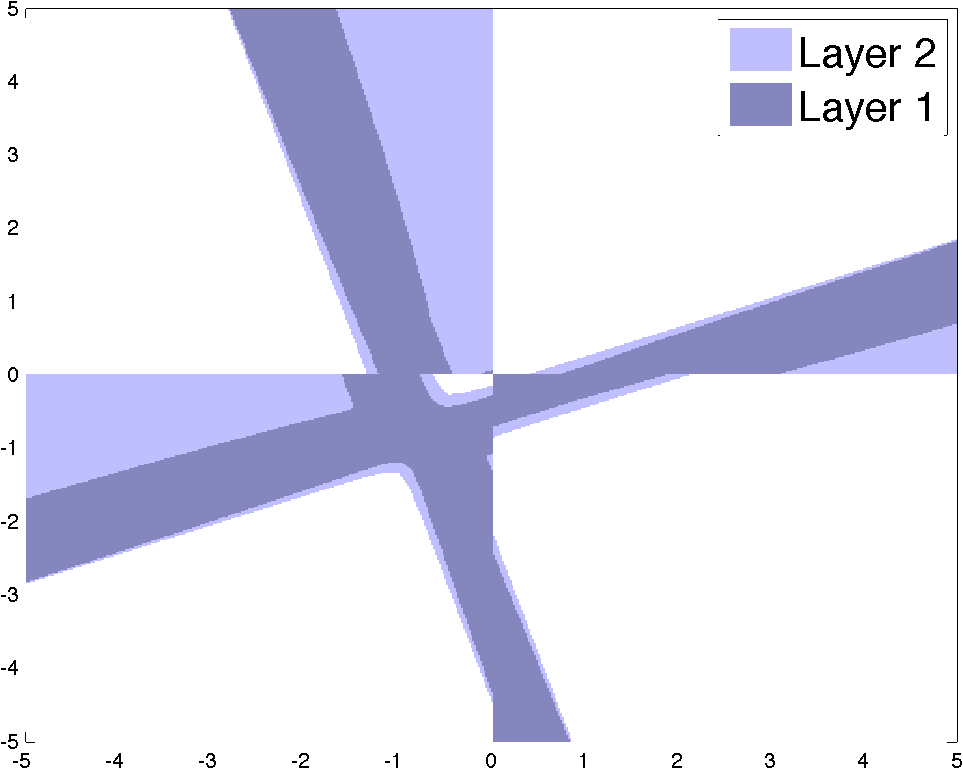}&
\includegraphics[width=0.32\textwidth]{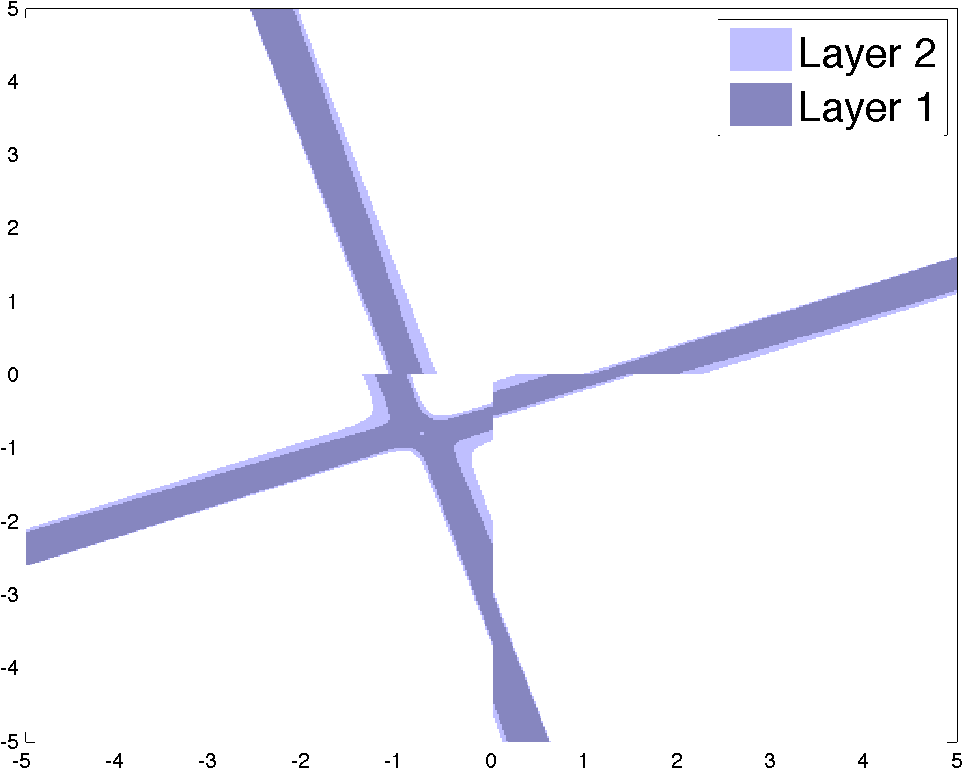}\\
({\bf{a}}) 100\%, 51\% & ({\bf{b}}) 28\%, 16\% & ({\bf{c}}) 9\%, 7\%\\[4pt]
\includegraphics[width=0.32\textwidth]{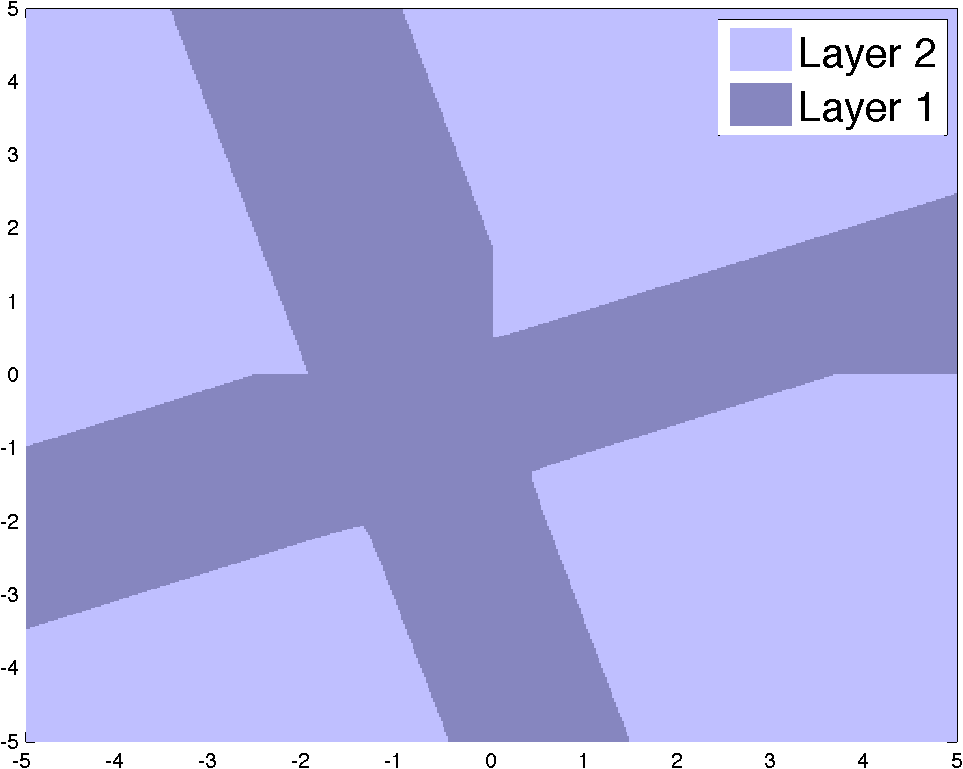}&
\includegraphics[width=0.32\textwidth]{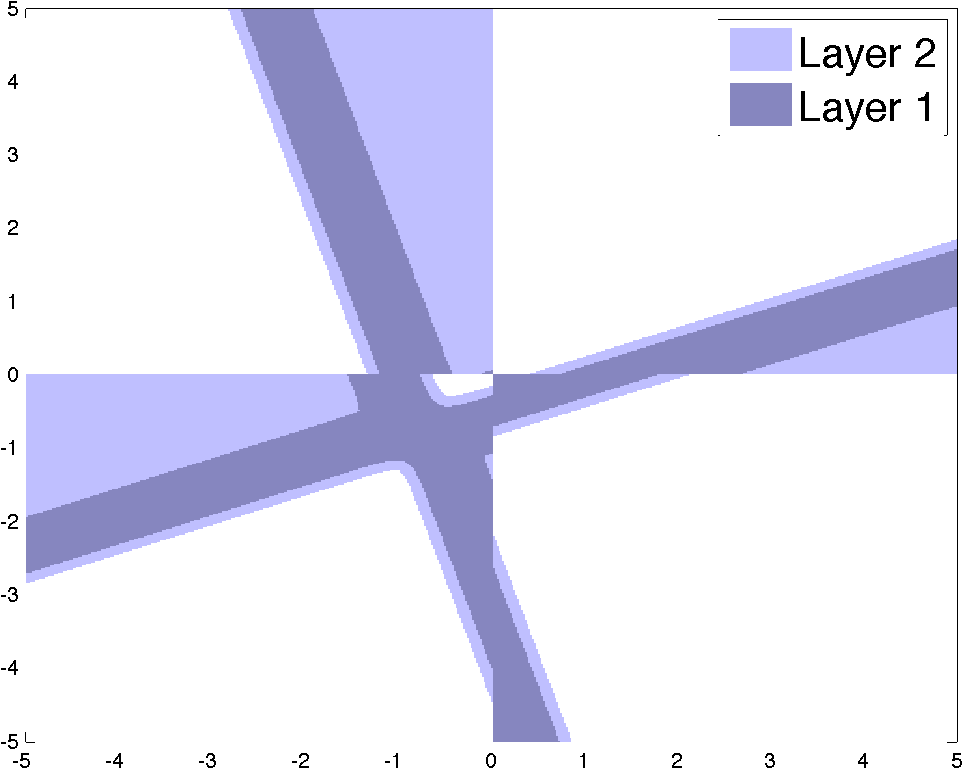}&
\includegraphics[width=0.32\textwidth]{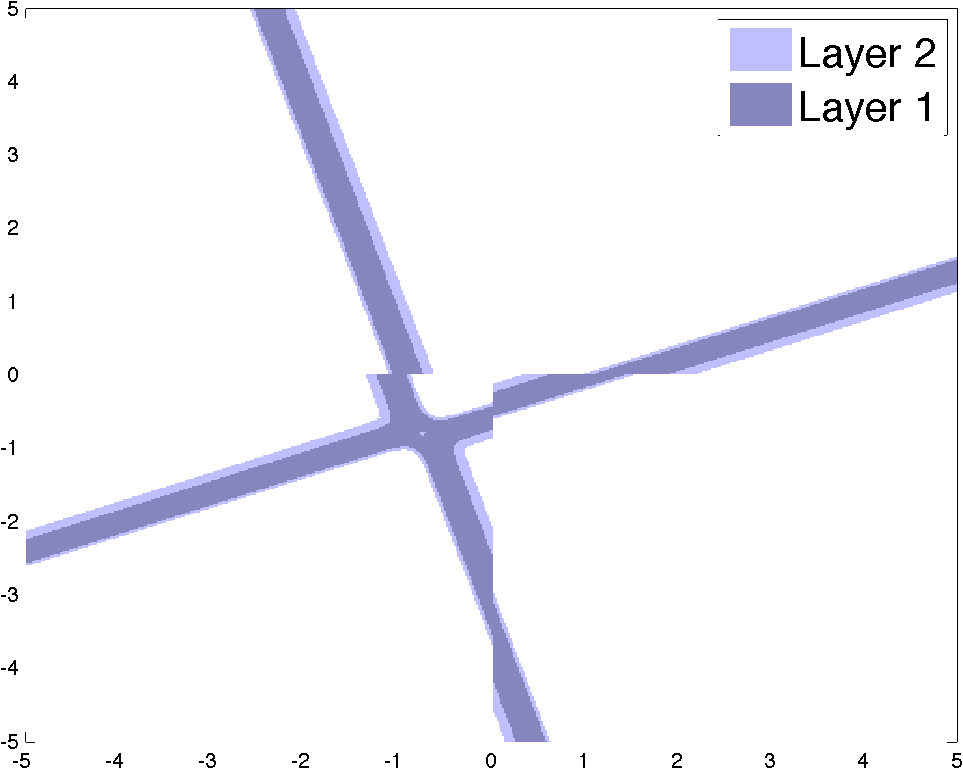}\\
({\bf{d}}) 100\%, 40\% & ({\bf{e}}) 28\%, 14\% & ({\bf{f}}) 9\%, 6\%\\[4pt]
\end{tabular}
\caption{Regions of the feature space that are backpropagated to
  layers 2 and 1. From left to right we have the settings $\gamma =
  1$, $\gamma = 2$, and $\gamma = 3$ from
  Figure~\ref{Fig:BackpropMaxGamma}, respectively. The top row shows
  the results obtained with the Frobenius norm of the gradients with
  respect to the weight matrices in each layer. The bottom row shows
  the results obtained by bounding the gradients elementwise. The
  percentages indicate the fraction of the feature space that was
  backpropagated to the second, and first layer.}\label{Fig:PartialPackprop}
\end{figure}

In order to do partial backpropagation, we need to determine at which
layer to stop. If this information is not given a priori, we need a
conservative and efficient mechanism that determines if further
backpropagation is warranted. One such method is to determine an upper
bound on the gradient components of all layers up to the current layer
and decide if this is sufficiently small. We now derive bounds on
$\norm{\partial f/\partial A_k}_F$ and $\norm{\partial f/\partial
  b_k}_F$ as well as on $\max_{i,j}\abs{[\partial f/\partial
  A_k]_{i,j}}$ and $\norm{\partial f/\partial b_k}_{\infty}$. It
easily follows from \eqref{Eq:PartialDiffsAb} that these quantities
are equal to $\norm{x_{k-1}}_2\norm{y_k}_2$ and $\norm{y_k}_2$,
respectively $\norm{x_{k-1}}_{\infty}\norm{y_k}_\infty$ and
$\norm{y_k}_{\infty}$. Since $x_{k-1}$ is known explicitly from the
forward pass, it suffices to bound the norms of $y_k$.  In fact, what
we are really after is to bound the norms of $y_k$ for all $1 \leq k <
j$ given $y_j$, since we can stop backpropagation only if all of them
are sufficiently small. For the $\ell_2$ norm we have
\begin{equation}\label{Eq:L2Bound}
\norm{y_{k-1}}_2\ \leq\ \norm{\sigma_{\gamma_{k-1}}'(v_{k-1})}_\infty \norm{z_k}_2\
\leq\ \sigma_{\gamma_{k-1}}'([v_{k-1}]_i)\cdot \sigma_{\max}(A_k)\norm{y_k}_2,
\end{equation}
where $i:= \argmin_j \abs{[v_{k-1}]_j}$, and $\sigma_{\max}(A_k)$ is
the largest singular values of $A_k$. Once we have a bound on
$\norm{y_j}_2$ we can apply \eqref{Eq:L2Bound} with $k=j$ to bound
$\norm{y_{j-1}}_2$. Although computation of $\sigma_{\max}(A_k)$ needs
to be done only once per batch but may still be prohibitively
expensive. In practice, however, it may suffice to work with an
approximate value, or use an alternative bound instead. For
$\ell_{\infty}$ we find
\begin{equation}\label{Eq:LInfBound}
\norm{y_{k-1}}_\infty\ \leq\ \max_i
\{\sigma_{\gamma_{k-1}}'([v_{k-1}]_i)\cdot
\norm{[A_{k}]_i}_2\norm{y_k}_2\}\ \leq\
\norm{\sigma_{\gamma_{k-1}}'(v_k)}_{\infty}\norm{y_k}_2 \max_{i}\{\norm{[A_k]_i}_2\},
\end{equation}
where the second, looser bound can be used if we want to avoid
evaluating $\sigma'_{\gamma_{k-1}}$ for all entries in $v_k$; the
infinity norm of this vector can be evaluated as above.

We applied the second bound in \eqref{Eq:L2Bound} and the first bound
in \eqref{Eq:LInfBound} to the setting for
Figure~\ref{Fig:BackpropMaxGamma} as follows. We first compute $y_3$
and evaluate the bound the gradients with respect to the weight and
bias terms in the first and second layer. If these bounds are smaller
than $0.05$ and $0.01$, respectively, we stop
backpropagation. Otherwise, we evaluate $y_2$ and update the bound on
the gradient with respect to the parameters of the first layer. If
this is less than $0.05$ we stop backpropagation, otherwise we
evaluate $y_1$ and complete the backpropagation process. In
Figure~\ref{Fig:PartialPackprop} we show the regions of the feature
space where backpropagation reaches the second, respectively first
layer. These regions closely match the predominant regions of the
gradient fields shown in Figure~\ref{Fig:BackpropMaxGamma}.
In practical applications the threshold values could be based on
previously computed (partial) gradient values, and may be adjusted
when the number of training samples that backpropagate to a given
layer falls below some threshold.

\section{Conclusions}\label{Sec:Discussion}
We reviewed and studied the decision region formation in feedforward
neural networks with sigmoidal nonlinearities. Although the definition
of hyperplanes and their subsequent combination is well known, very
little attention has so far been given to transitions regions at the
boundaries of classes and other regions with varying levels of
classification confidence. We clarified the relation between the
scaling of the weight matrices, the increase in confidence and
sharpening of the transition regions, and the corresponding
localization of the gradient field. The degree of localization differs
per layer and is one of the main factors that determine how much
progress can be made at each step of the training process: a high
level of localization combined with a relatively coarse sampling
density or small batch size leads to the vanishing gradient problem
where updates to one or more layers become excessively small. The
gradient field tends to become increasingly localized towards the
first layer, and the parameters in this layer are therefore most
likely to get stuck prematurely. When this happens, subsequent layers
must form classifications regions based on suboptimal hyperplane
locations. It is often possible to slightly decrease the loss function
by increasing confidence levels by scaling parameters in later
layers. This can lead to a cascading effect in which layers
successively get stuck.  The use of regularized or constrained
optimization can help control the scaling of the weights, thereby
limiting the amount of gradient localization and thus avoiding or
reducing these problems. By gradually allowing the weights to increase
it is possible to balance progress in the learning process and
attaining decision regions with sufficiently high confidence
levels. In addition, regularized and constrained optimization can help
prevent overfitting.
Analysis of the gradient field also shows that at any given iteration,
the contributions of different training points to the gradient can
vary substantially. Localization of the gradient towards the first layer
also means that some points are informative only from a certain layer
onwards. Together this suggests dynamic subset selection and partial
backpropagation, or adaptive selection of the step size for each layer
depending on the number of relevant points.

We hope that some of the results presented in this paper will
contribute to a better understanding of neural networks and eventually
lead to new or improved algorithms. There remain several topics that
are interesting but beyond the scope of the present paper. For
example, it would be interesting to see what the hyperplanes generated
during pre-training using restricted Boltzmann machines
\cite{HIN2006OTa} look like, and if there are better choices. One
possible option is to select random training samples from each class
and generate randomly oriented hyperplanes through these points by
appropriate choice of $b$. Likewise, given a hyperplane orientation
and a desired class, it is also possible to place the hyperplane at
the class boundary by choosing $b$ to coincide with the largest or
smallest inner product of the normal with points from that class.
Another interesting topic is an extension of this work to other
nonlinearities such as the currently popular rectified linear unit
given by $\nu(x) = \max(0,x)$. The advantage of these units is that
gradient masks is one for all all positive inputs and are not
localized, thereby avoiding gradient localization and thus allowing
the error to backpropagate more easily.  It would be interesting to
look at the mechanisms involved in the formation of decision regions,
which differ from those of sigmoidal units. For example, it is not
entirely clear how the logical \band{} should be implemented: summing
inverted regions and thresholding may work in some cases, but more
generally it should consist of the minimum of all input regions. In
terms of combinatorial properties, bounds on the number of regions
generated using neural networks with rectified and piecewise linear
functions were recently obtained in \cite{MON2014PCBa,PAS2013MBa}. The
main problem with rectified linear units is that it maps all negative
inputs to zero, thereby creating a zero gradient mask at those
locations.  The softplus nonlinearity \cite{GLO2011BBa}, which is a
smooth alternative in which the gradient mask never vanishes, would
also be of interest. Finally it would be good to get a better
understanding of dropout \cite{HIN2012SKSa} and second-order methods
from a feature-space perspective.

\bibliography{bibliography.bib}

\end{document}